\documentclass{article}

\usepackage{arxiv}
\usepackage{authblk}
\usepackage{hyperref}
\usepackage{graphicx}%
\usepackage{multirow}%
\usepackage{amsmath,amssymb,amsfonts}%
\usepackage{amsthm}%
\usepackage{mathrsfs}%
\usepackage{xcolor}%
\usepackage{textcomp}%
\usepackage{manyfoot}%
\usepackage{booktabs}%
\usepackage{algorithm}%
\usepackage{algorithmicx}%
\usepackage{algpseudocode}%
\usepackage{listings}%

\usepackage{apacite}

\usepackage{tabularx}
\usepackage[toc,title, page]{appendix}

\usepackage{amssymb}  
\usepackage{pifont} 
\usepackage{enumitem}

\usepackage{amssymb}  
\usepackage{pifont}

\usepackage{tikz}
\usetikzlibrary{arrows.meta, positioning}
\usetikzlibrary{matrix, positioning, fit, shapes.misc}
\usepackage[most]{tcolorbox}  

\usepackage{newunicodechar}

\usepackage[T1]{fontenc}      
\usepackage[utf8]{inputenc}   
\usepackage[english]{babel}   
\usepackage{csquotes}
\usepackage{amsmath, amsthm, amscd, amsfonts, amssymb}
\usepackage{graphicx}
\usepackage[all]{xy}          
\usepackage{color}
\usepackage{lipsum}           
\usepackage{enumitem}

\newtheorem{theorem}{Theorem}[section]

\newtheorem{proposition}[theorem]{Proposition}
\newtheorem{corollary}[theorem]{Corollary}

\theoremstyle{definition}
\newtheorem{definition}[theorem]{Definition}
\newtheorem{example}[theorem]{Example}

\theoremstyle{remark}
\newtheorem{remark}[theorem]{Remark}

\numberwithin{equation}{section}
\usepackage{xurl} 

\usepackage{etoolbox}
\appto\bibsetup{\sloppy\emergencystretch=3em}

\usepackage{hyperref}
\hypersetup{
  colorlinks=true,
  linkcolor=red,
  citecolor=cyan,
  urlcolor=red,
  bookmarksnumbered=true,
  plainpages=false
}
\begin{document}
\title{Featured Reproducing Kernel Banach Spaces for Learning and Neural Networks}


\newcommand{\shorttitle}{Featured RKBSs for Learning and Neural Networks}

\author[1]{Isabel de la Higuera}
\author[1,2]{Francisco Herrera}
\author[3,4]{M.\ Victoria Velasco}

\affil[1]{Dept. of Computer Science and Artificial Intelligence,  Andalusian Institute on Data Science and Computational Intelligence (DaSCI), University of Granada, Granada, 18140, Spain, email: \texttt{herrera@decsai.ugr.es}}   \affil[2]{ADIA Lab, Abu Dhabi, United Arab Emirates.}  \affil[3]{Dpto. de An\'{a}lisis Matem\'{a}tico, Facultad de Ciencias, Universidad de Granada 18071, Granada (Spain), \break email: \texttt{vvelasco@ugr.es}}   \affil[4]{Corresponding author}

\date{\today}

\maketitle

\begin{abstract}
Reproducing kernel Hilbert spaces provide a foundational framework for kernel-based learning, where regularization and interpolation problems admit finite-dimensional solutions through classical representer theorems. Many modern learning models, however—including fixed-architecture neural networks equipped with non-quadratic norms—naturally give rise to non-Hilbertian geometries that fall outside this setting. In Banach spaces, continuity of point-evaluation functionals alone is insufficient to guarantee feature representations or kernel-based learning formulations.

In this work, we develop a functional-analytic framework for learning in Banach spaces based on the notion of featured reproducing kernel Banach spaces. We identify the precise structural conditions under which feature maps, kernel constructions, and representer-type results can be recovered beyond the Hilbertian regime. Within this framework, supervised learning is formulated as a minimal-norm interpolation or regularization problem, and existence results together with conditional representer theorems are established.

We further extend the theory to vector-valued featured reproducing kernel Banach spaces and show that fixed-architecture neural networks naturally induce special instances of such spaces. This provides a unified function-space perspective on kernel methods and neural networks and clarifies when kernel-based learning principles extend beyond reproducing kernel Hilbert spaces.
\end{abstract}



\section{Introduction}

The mathematical foundations of machine learning are deeply rooted in functional analysis and, in particular, in the theory of reproducing kernel Hilbert spaces (RKHSs). Since the seminal work of Aronszajn (1950) \cite{arons1950}, RKHS theory has provided a rigorous and unifying framework for learning from finite datasets, allowing classical algorithms such as support vector machines, kernel ridge regression, and Gaussian processes to be formulated as functional optimization problems. Within this setting, learning naturally corresponds to minimal-norm interpolation or regularization in a Hilbert space of functions endowed with a reproducing kernel. Moreover, continuity of point-evaluation functionals is equivalent to the existence of a reproducing kernel, and hence an associated feature map, which in turn underpins classical representer theorems for a broad class of norm-monotone regularized learning problems \cite{ScholkopfSmola2002}.

However, many function classes relevant to modern machine learning—most notably fixed-architecture neural networks equipped with non-quadratic parameter norms—do not naturally admit a Hilbert space formulation \cite{ZhangXuZhang2009RKBS}. In Banach spaces, continuity of point evaluation alone is insufficient to guarantee the existence of a feature representation or a representer theorem, and additional structural assumptions are required, as we will show in this work. Motivated by this gap, reproducing kernel Banach spaces (RKBSs) have been proposed as a generalization of RKHSs, yet the precise conditions under which kernel constructions and representer-type results can be recovered remain only partially understood.

In this work, we identify the structural requirements that enable feature-map–induced RKBSs and conditional representation theorems in Banach spaces, and we show that fixed neural network architectures naturally induce special vector-valued featured reproducing kernel Banach spaces. This provides a unified functional-analytic perspective on kernel methods and neural networks. 
Recent work has renewed interest in functional-analytic and norm-based perspectives on learning dynamics and generalization in neural networks, including kernel and infinite-width limits, implicit regularization, and non-Hilbertian geometries \cite{Neyshabur2018,Jacot2018NTK,sun2023unified} from complementary perspectives. 

Table~\ref{tab:rkhs-rkbs-comparison} summarizes the structural assumptions that underpin the theoretical developments in this paper. 
In particular, it highlights why feature-map–induced RKBSs are required to formulate learning problems and derive representer-type results in Banach spaces, and how these assumptions enable the neural-network interpretations developed later.

\begin{table}[H]
\centering
\renewcommand{\arraystretch}{1.6}
\begin{tabularx}{\textwidth}{p{2.7cm} X X X}
\hline
\textbf{Aspect} 
& \textbf{RKHS} 
& \textbf{General RKBS} 
& \textbf{Featured RKBS (this paper)} \\
\hline
Ambient function space 
& Hilbert space with continuous point evaluation
& Banach space with continuous point evaluations 
& Banach space induced by a feature map \\

Feature-map representation 
& Exists canonically via the reproducing kernel 
& Not implied by the continuity of point evaluation  
& Implicit in the definition \\

Kernel object 
& Reproducing kernel uniquely determined 
& Exists only  in the subclass of the featured RKBS
& Kernel associated with the feature map\\

Learning formulation 
& Regularization / minimal norm interpolation via RKHS norm 
& Regularization possible, but structure may be weak 
& Learning posed as minimal-norm interpolation or regularization in the induced RKBS \\

Representer theorem 
& Holds for norm-monotone regularization 
& Depends on the specific RKBS under consideration
& Admit a precise structural characterization\\

Form of solutions 
& Finite kernel expansion over training points 
& Generally non-constructive; weak-\(*\) closure issues 
& Finite kernel expansion when representer conditions are satisfied \\

Role of dual / predual 
& Trivial (Hilbert self-duality) 
& Predual may not exist
& Predual space is derived from the feature map. In special featured RKBS the predual is an RKBS \\

Neural network interpretation 
& Indirect, typically via infinite-width or kernel limits
& No systematic interpretation 
& Natural for fixed architectures via induced feature maps or the associated kernel\\
\hline
\end{tabularx}
\caption{Comparison of reproducing kernel Hilbert spaces (RKHSs), general reproducing kernel Banach spaces (RKBSs), and the featured RKBSs studied in this work. The table highlights the structural assumptions required to obtain kernel-based learning formulations and representer-type solutions in Banach settings, and their connection with neural-network interpretations developed later in the paper.}
\label{tab:rkhs-rkbs-comparison}
\end{table}

\newpage
\textbf{Contributions.}
This paper makes the following contributions:
\begin{enumerate}
    \item We introduce the notion of \emph{featured reproducing kernel Banach spaces}, distinguishing them from general RKBSs and characterizing the precise structural conditions under which a Banach space of functions admits a feature-map representation.
    \item We show that, unlike the Hilbert setting, continuity of point evaluations in Banach spaces does not guarantee representer theorems, and we provide necessary and sufficient conditions under which representer-type solutions exist in RKBSs spaces. As a corollary, explicit sufficient conditions are also provided.
    \item We formulate learning in featured RKBSs as a minimal-norm interpolation and regularization problem, and establish the existence of solutions using functional-analytic arguments that highlight the role of the Hahn--Banach theorem.

    \item We extend the framework to vector-valued featured RKBSs, preserving the full expressive power of the scalar-valued theory. Therefore, we develop the associated kernel constructions. Moreover, the relevance of special featured RKBSs is made evident. 

   \item We demonstrate that fixed-architecture neural networks naturally induce special vector-valued featured RKBSs, thereby providing a unified functional-analytic interpretation of kernel methods and neural networks. Finally, we discuss the training process and propose representative training schemes.

\end{enumerate}

The present work focuses on the function-space formulation of learning in Banach settings and on the existence and structure of minimal-norm solutions. We do not address optimization dynamics, statistical generalization bounds, or algorithmic efficiency, which depend on additional assumptions beyond the scope of this paper. Our results should therefore be interpreted as providing a foundational framework upon which such analyses may be built.

To our knowledge, this is the first work that provides a complete feature-map–based characterization of representer theorems in Banach spaces and connects them systematically to fixed-architecture neural networks.

The paper is organized as follows. 
Section~\ref{sec:2} reviews the necessary background on reproducing kernel Hilbert spaces and situates reproducing kernel Banach spaces within the broader landscape of kernel-based learning. 
Section~\ref{sec:3} introduces the functional-analytic foundations of reproducing kernel Banach spaces, emphasizing the distinction between general RKBSs and feature-map--induced (featured) RKBSs and characterizing the structural conditions under which feature representations exist. 
Section~\ref{sec:rkbs} formulates supervised learning in featured RKBSs as minimal-norm interpolation and regularization problems and establishes existence results together with conditional representer-type formulations. 
Section~\ref{sec:rep_theorems} develops representer theorems for special featured RKBSs, identifying the precise conditions under which finite kernel representations of solutions are guaranteed. 
Section~\ref{sec:vector_valued} extends the framework to vector-valued featured RKBSs and develops the associated kernel constructions for multi-output learning problems. 
In Section~\ref{neuronas}, we show that fixed-architecture neural networks naturally induce special vector-valued featured RKBSs and provide a functional-analytic interpretation of network training as a minimal-norm learning problem. 
Section~\ref{sec:example} presents an illustrative example that makes the abstract constructions explicit and demonstrates how the theory applies to concrete neural network architectures. 
Section~\ref{sec:LT-insights} distills the main learning-theoretic insights and implications of the framework. 
Finally, Section~\ref{sec:conc} concludes with a discussion of limitations, implications, and directions for future research.

\section{Related Work}
\label{sec:2}

Kernel methods and regularization in reproducing kernel Hilbert spaces (RKHSs) provide a standard framework in which learning can be cast as optimization over function spaces, and solutions to broad classes of regularized empirical risk minimization problems admit finite expansions over the training set. Early representer-type results can be traced to spline theory \cite{kimeldorf1971some}, while modern formulations and generalizations in the kernel learning literature establish representer theorems under norm-monotone regularization and general loss terms \cite{scholkopf2001generalized,ScholkopfSmola2002}.

{To move beyond Hilbert geometries—motivated by non-quadratic regularization, sparsity-promoting norms, and other non-Hilbert inductive biases—RKBSs were introduced as Banach space analogs of RKHSs \cite{ZhangXuZhang2009RKBS,xu2023sparse}. A recurring challenge in RKBS-based learning is that Hilbert tools (e.g., Riesz representation) are not available in general, and representer theorems require additional structure. Recent work has continued to formalize RKBS constructions and clarify which Banach geometries support kernel-based learning principles and tractable representations \cite{wang2021representer}.

Within RKBSs, finite-dimensional solution structure results can be substantially subtler than in RKHSs, and can depend on duality structure and subdifferential geometry. Wang and Xu (2021) \cite{wang2021representer} provide a systematic study of representer theorems for minimum-norm interpolation and regularized learning in Banach spaces, including explicit formulations and conditions under which optimization can be reduced to a truly finite-dimensional problem.
More recently, Wang, Xu and Yan (2024) \cite{wang2024sparse} studied sparse representation theorems for RKBS learning, giving conditions that yield kernel representations with fewer terms than the number of observations and analyzing how sparsity depends on the RKBS norm and regularization.
These lines of work motivate a finer-grained taxonomy of RKBSs based on whether feature-map constructions and representer-type solutions are available—precisely the perspective adopted in this paper.

In parallel, a growing literature studies neural networks via function-space norms and variational characterizations, connecting architecture and parameter norms to induced function-space geometries and implicit bias
\cite{Jacot2018NTK,Neyshabur2018,Wojtowytsch2020}. 
 For example, neural tangent kernel (NTK) work provides an influential infinite-width kernel limit \cite{Jacot2018NTK}, while research on implicit regularization highlights how optimization and geometry select solutions beyond explicit penalties \cite{Neyshabur2018}.
Recent work strengthens the function-space bridge between kernels and networks in directions closely aligned with Banach geometries. Sun et al. (2023) \cite{sun2023unified} study how mirror descent controls implicit regularization, explicitly emphasizing the role of geometry in the learned solutions.
Shenouda et al. (2024) \cite{shenouda2024variation} introduce vector-valued variation spaces, a class of RKBSs tailored to multi-output neural networks, and establish representer theorems and structural insights for multi-task learning and compression.

Follain and Bach (2025) \cite{follain2024enhanced} propose a framework that explicitly integrates feature learning and kernel methods via regularized empirical risk minimization over function classes that can be interpreted both as kernel constructions and as infinite-width neural networks with learned projections/nonlinearities.

These works underscore active interest in unifying kernels, neural networks, and norm-based learning principles—context in which the featured-RKBS framework developed here provides a complementary architecture-aware Banach-space perspective.

In contrast to these works, we focus on identifying the minimal structural assumptions under which kernel-based learning and representer theorems can be recovered in Banach spaces, and on interpreting fixed-architecture neural networks within this framework.

\section{From RKHSs to Reproducing Kernel Banach Spaces}
\label{sec:3}

This section introduces the functional-analytic framework underlying reproducing kernel Banach spaces (RKBSs), with an emphasis on concepts that are directly relevant to learning theory. While reproducing kernel Hilbert spaces provide a complete and well-understood setting for kernel-based learning, many learning problems of practical interest involve non-Hilbertian geometries arising from non-quadratic regularization, sparsity-inducing norms, or parameter norms associated with neural networks. RKBSs generalize RKHSs to Banach spaces and thus offer a natural mathematical language for studying such settings, provided additional structural assumptions are satisfied.

We begin by recalling the definition of RKBSs and highlighting the role of point-evaluation continuity, before introducing the notion of feature-map--induced RKBSs. This distinction is crucial from a learning perspective: unlike in the Hilbert case, not every RKBS admits a feature representation, and consequently not every RKBS supports kernel-based learning algorithms or representer theorems. The results in this section establish the structural foundations needed to formulate learning problems in Banach spaces in a way that parallels, but does not trivially extend, classical RKHS theory.

Throughout this paper, $X$ will be a nonempty set and $\mathbb{K}$ will stands for either the field $\mathbb{R}$ of real numbers or the field $\mathbb{C}$ of complex numbers.  Furthermore, $\mathcal{F}(X,\mathbb{K})$ will denote the vector space over $\mathbb{K}$  consisting of all functions from $X$ into $\mathbb{K}$, that is,
\[
\mathcal{F}(X,\mathbb{K}):=\{f\colon X\to\mathbb{K}: f \hspace{0.1cm }\text{ is a function}\},
\]
endowed with the usual pointwise operations, namely
\[
(f+g)(x):= f(x)+g(x), \hspace{0.2cm} \text{and} \hspace{0.2cm} (\alpha f)(x):=\alpha f(x),
\]
for all $f,g\in\mathcal{F}(X,\mathbb{K}),\alpha\in \mathbb{K} $ and $x\in X$.

This section is devoted to the study of reproducing kernel Hilbert spaces (RKHS) as a preliminary step toward their extension to reproducing kernel Banach spaces (RKBS). Both notions are defined on Hilbert
or Banach spaces that are vector subspaces of $\mathcal{F}(X,\mathbb{K}),$ denoted
 by $H$ and $B,$ respectively, so that their elements are functions with
no \textit{a priori} properties. The base field of these spaces is $\mathbb{%
K}$. Consequently, $H \subseteq\mathcal{F}(X,\mathbb{K})$ means that the Hilbert space $H$ is a vector subspace of $\mathcal{F}(X,\mathbb{K})$, and similarly for a Banach space $B\subseteq\mathcal{F}(X,\mathbb{K})$. Throughout this paper, this convention will be adopted without further notice (see, for example, Definition~\ref{def2.1}.)
In the next theorem, we characterize a reproducing kernel Hilbert space (RKHS) in terms of its feature map
and kernel function. In this context if $\mathbb{K=R}$, the conjugation is
understood as the identity.

\begin{theorem}
\label{Hilbert}Let $H\subseteq \mathcal{F}(X,\mathbb{K})$ be a Hilbert
space. Then, the following assertions are equivalent:

$\mathrm{(i)}$\ The linear map $\delta _{x}:H\to \mathbb{K}$ given by $%
\delta _{x}(f)=f(x),$ for every $f\in H,$ is continuous for every $x\in X$.
This means that there exists $C_{x}\geq 0$ such that 
\begin{equation*}
\left\vert \delta _{x}(f)\right\vert =\left\vert f(x)\right\vert \leq
C_{x}\left\Vert f\right\Vert ,\text{ \ for \ every }f\in H.
\end{equation*}

$\mathrm{(ii)}$\ There exists a map $\Phi :X\to H$ such
that $f(x)=\left\langle f,\Phi (x)\right\rangle$,  for every $x\in X.$

$\mathrm{(iii)}$\ There exists a function $k:X\times X\to 
\mathbb{K}$ such  that $k(\cdot ,x)\in H$, for every $x\in X$, satisfying $f(x)=\left\langle f,k(\cdot
,x)\right\rangle ,$ for every  $f\in H$ and $x\in X.$

If these assertions hold, then $\Phi $ and $k$ are unique. Moreover,  
\begin{equation*}
H=\overline{\mathrm{Lin}\{k(\cdot ,x):x\in X\}}^{\|\cdot\|}=\overline{\mathrm{Lin}\{\Phi
(x):x\in X\}}^{\|\cdot\|}
\end{equation*}%
and 
\begin{equation*}
k(x,x^{\prime })=\overline{\left\langle k(\cdot ,x),k(\cdot ,x^{\prime
})\right\rangle }=\overline{\left\langle \Phi (x),\Phi (x^{\prime
})\right\rangle }.
\end{equation*}
\end{theorem}

\begin{proof}
    $\mathrm{(i)\Longrightarrow (ii)}$ Since $\delta _{x}$ is linear, its continuity implies that $\delta _{x}\in H^{\ast }$ and, consequently, by the Riesz representation theorem
  \cite[Theorem 3.4]{conway1994course}, there exists a unique $k_{x}\in H$ such that $\delta _{x}(\cdot
)=\left\langle \cdot ,k_{x}\right\rangle .$ 
Let $\Phi :X\to H$ be given by $\Phi (x)=k_{x},$
for every $x\in X.$ Then, 
\begin{equation*}
\left\langle f,\Phi (x)\right\rangle =\left\langle f,k_{x}\right\rangle
=\delta _{x}(f)=f(x),  \text{ where }  f\in H.
\end{equation*}

$\mathrm{(ii)\Longrightarrow (iii)}$ Define $k(\cdot ,x):=\Phi (x)$ and $%
k(x,x')=\delta _{x}(k(\cdot ,x'))=\delta _{x}(\Phi(x'))$, for every $x,x' \in X$ (note that  $%
k(\cdot ,x')=\Phi (x')\in H\subseteq \mathcal{F}(X,\mathbb{K})$$).$
By $\mathrm{(ii)}$, we have  $k(\cdot ,x)\in H$ and $%
f(x)=\left\langle f,k(\cdot ,x)\right\rangle$, for every $x\in X.$

$\mathrm{(iii)\Longrightarrow (i)}$ Since $f(x)=\left\langle f,k(\cdot
,x)\right\rangle,$ for every $x\in X$ and $f\in H,$ from the Cauchy-Schwarz
inequality, letting $C_{x}=\left\Vert k(\cdot ,x)\right\Vert ,$ it follows that 
\begin{equation*}
\left\vert \delta _{x}(f)\right\vert =\left\vert f(x)\right\vert \leq
C_{x}\left\Vert f\right\Vert ,
\end{equation*}
for every $x\in X, $ which proves  $\mathrm{(i).}$

Finally, the uniqueness of $\Phi (x)$ follows from the fact that $%
\left\langle f,\Phi (x)\right\rangle =f(x)$. Indeed, if another map $%
 \widetilde{\Phi }:X\to H$ also satisfies this property, then 
\begin{equation*}
\left\langle f,\Phi (x)-\widetilde{\Phi }(x)\right\rangle =f(x)-f(x)=0,
\end{equation*}%
for every $f\in H$ and  $x\in X.$ Taking $f=\Phi (x)-\widetilde{\Phi }(x) \in H$ we
conclude that $\Phi =\widetilde{\Phi }$. The proof of the uniqueness  of $k$ is similar by
replacing $\Phi (x)$ and $\widetilde{\Phi }(x)$ with $k(\cdot ,x)$ and $%
\widetilde{k}(\cdot ,x),$ respectively. Finally, since $$f(x)=\left\langle f,\Phi (x)\right\rangle
=\left\langle f,k(\cdot ,x)\right\rangle ,$$ for every $x\in X$ and $f\in H,$ it follows that if $\left\langle f,\Phi (x)\right\rangle =0$ for every $x\in X$ then $%
(\mathrm{Lin}\{\Phi (x):x\in X\})^{\bot }=\{0\}$ and therefore, by~\cite[Corollary~2.9]{conway1994course}, it follows that
\begin{equation*}
H=\{0\}^{\bot }= (\mathrm{Lin}\{\Phi (x):x\in X\})^{\bot \perp }=\overline{%
\mathrm{Lin}\{\Phi (x):x\in X\}}^{\|\cdot\|}.
\end{equation*}%
By the same reasoning, $\{k(\cdot ,x):x\in X\}^{\bot }=\{0\}$ and
therefore,%
\begin{equation*}
H=\overline{\mathrm{Lin}\{k(\cdot ,x):x\in X\}}^{\|\cdot\|}.
\end{equation*}%
Finally,%
\begin{equation*}
k(x,x^{\prime })=\delta _{x}(k(\cdot ,x^{\prime }))=\left\langle k(\cdot
,x^{\prime }),k(\cdot ,x)\right\rangle =\overline{\left\langle k(\cdot
,x),k(\cdot ,x^{\prime })\right\rangle },
\end{equation*}%
and similarly $k(x,x^{\prime })=\delta _{x}(\Phi (x^{\prime }))=\left\langle
\Phi (x^{\prime }),\Phi (x)\right\rangle =\overline{\left\langle \Phi
(x),\Phi (x^{\prime })\right\rangle }.$
\end{proof}
The next definition can be found in \cite[Definition 2.1]{Paulsen_Raghupathi_2016}.

\begin{definition}
\label{def2.1}
A \textbf{reproducing kernel Hilbert space} (\textbf{RKHS}) is a Hilbert
space $H\subseteq \mathcal{F}(X,\mathbb{K})$ such that the point evaluation mapping $\delta
_{x}:H\to \mathbb{K}$, defined by $$\delta_x(f)=f(x),\hspace{0.2cm}\text{for every} \hspace{0.2cm}f\in H,$$ is continuous for each $x\in X.$ 

If $H$ is an
RKHS then, the mapping $\Phi $ given in assertion $\mathrm{(ii)}$ of Theorem %
\ref{Hilbert} is called the \textbf{feature map} associated with $H,$ and, similarly,
the mapping $k$ given in assertion $\mathrm{(iii)}$ is the\textbf{\
reproducing kernel function}
associated with $H.$ Moreover,
the \textbf{reproducing} \textbf{%
property} of the kernel $k$ is the equality $f(x)=\left\langle f,k(\cdot
,x)\right\rangle ,$ for every $x\in X$ and $f\in H.$
\end{definition}

Obviously, in Definition~\ref{def2.1}, the continuity of $\delta _{x}$ (that is
assertion $\mathrm{(i)}$ in Theorem \ref{Hilbert}) can be replaced by either
assertion $\mathrm{(ii)}$ or $\mathrm{(iii)}$ in Theorem \ref{Hilbert},
since these three properties are equivalent, as we have shown$.$

\begin{remark}
The Moore-Aronszajn Theorem \cite[Property 2.4]{arons1950} states that to every Hermitian
positive definite function $k:X\times X\to \mathbb{K}$, there 
corresponds a unique reproducing kernel Hilbert space (RKHS) \break$H\subseteq 
\mathcal{F}(X,\mathbb{K}),$ and conversely. This correspondence arises from
the fact that $k$ is Hermitian and positive definite if and only if
the space $$H_{0}:=\mathrm{Lin}\{k(x,\cdot ):x\in X\}$$ is a pre-Hilbert space.
Indeed, if $H$ is an RKHS, then the associated reproducing kernel is necessarily Hermitian and positive definite, as can be verified directly.
Conversely, if $k:X\times X\to \mathbb{K}$ is Hermitian and positive
definite, then the space $H$ given by completion of 
$H_{0}:=\mathrm{Lin}\{k(x,\cdot ):x\in X\}$
is a Hilbert space, and it follows that $k$ satisfies condition $\mathrm{
(iii)}$ in Theorem~\ref{Hilbert}.
\end{remark}

The above definition motivates the next one, established in \cite[Definition 2.1]{banach1}.

\vspace{0.1cm}
\begin{definition}
 A Banach space $B\subseteq \mathcal{F}(X,\mathbb{K})$ is a \textbf{%
reproducing kernel Banach space} (\textbf{RKBS}) if, for each $x\in X,$ the evaluation map $\delta
_{x}:B\to \mathbb{K}$ defined by $$
\delta _{x}(f)=f(x) \ \ (f \in B ),
$$
is continuous. That is, there exists $C_{x}\geq 0$ such that 
\begin{equation*}
\left\vert \delta _{x}(f)\right\vert =\left\vert f(x)\right\vert \leq
C_{x}\left\Vert f\right\Vert \text{, for every }f\in B.
\end{equation*}
\end{definition}

Obviously, every reproducing kernel Hilbert space is an RKBS. \ Other
essential examples can be obtained very easily, as we show next.

\vspace{0.1cm}
\begin{example}
\label{pro} Let $X$ be a non-empty set, and let $%
W$ be a Banach space. Then,
every map $\Phi :X\rightarrow W$ gives rise to an RKBS. Indeed, if 
\begin{equation*}
W_{0}:=\overline{\mathrm{Lin}\{\Phi (x):x\in X\}}^{\left\Vert \cdot
\right\Vert },
\end{equation*}%
then the space $B_{\Phi }$ given by 
\begin{equation*}
B_{\Phi }:=\{f_{\mu }(\cdot ):=\mu (\Phi (\cdot )):\mu \in W_{0}^{\ast
}\}\subseteq \mathcal{F}(X,\mathbb{K}),
\end{equation*}%
where $f_{\mu }(x):=\mu (\Phi (x)),$ provided with the norm $\left\Vert
f_{\mu }\right\Vert :=\left\Vert \mu \right\Vert ,$ is an RKBS.

That $B_{\Phi }$ is a vector space is obvious since $f_{\mu }+f_{\widetilde{%
\mu }}=f_{\mu +\widetilde{\mu }}$ and $\alpha f_{\mu }=f_{\alpha \mu }$ for
every $\mu ,\widetilde{\mu }$ $\in W_{0}^{\ast }$ and $\alpha \in \mathbb{K%
}\,.$ Since $B_{\Phi }$ is algebraically isomorphic to $W_{0}^{\ast }$, it
is clear that $\left\Vert f_{\mu }\right\Vert :=\left\Vert \mu \right\Vert $
defines a Banach norm of $B_{\Phi }$. On the other hand, the evaluation map $%
\delta _{x}:B_{\Phi }\rightarrow \mathbb{K}$ is such that 
\begin{equation*}
\left\vert \delta _{x}(f_{\mu })\right\vert =\left\vert \mu (\Phi
(x))\right\vert \leq \left\Vert \Phi (x)\right\Vert \left\Vert \mu
\right\Vert .
\end{equation*}%
Thus, $\left\vert \delta _{x}(f_{\mu })\right\vert \leq C_{x}$ $\left\Vert
f_{\mu }\right\Vert ,$ for every $f_{\mu }\in B_{\Phi },$  where $%
C_{x}=\left\Vert \Phi (x)\right\Vert $. Consequently, $%
B_{\Phi }$ is an RKBS.
\end{example}

We are only interested in those $B_\Phi$ associated with a map $\Phi:X\rightarrow W$, where $W$ is a Banach space such that $W\subseteq \mathcal{F}(X^{\prime },\mathbb{K})$, for some non-empty set $X'$. This motivates the following definition.

\vspace{0.1cm}
\begin{definition}
\label{feature}
A \textbf{feature map} is a map $\Phi :X\rightarrow W\subseteq \mathcal{F}%
(X^{\prime },\mathbb{K}),$ where $X$ and $X^{\prime }$ are non-empty sets
and $W$ is a Banach space that is a vector subspace of $\mathcal{F}%
(X^{\prime },\mathbb{K})$. If $\Phi $ is a feature map, then the space $%
B_{\Phi }$ described in  Example \ref{pro} is said to be the \textbf{RKBS
associated with}  $\Phi.$
\end{definition}

Linked to the concept of a feature map is that of a kernel function, as we
are going to show.

\begin{definition}
Let $X$ and $X^{\prime }$ be non-empty sets. We say that a map $k:X\times
X^{\prime }\rightarrow \mathbb{K}$ is a \textbf{kernel function} if there
exists a Banach space $W\subseteq \mathcal{F}(X^{\prime },\mathbb{K})$ $\ $%
such that $k(x,\cdot )\in W,$ for every $x\in X.$ The space $W$ is then called the \textbf{Banach space associated with} $k$.
\end{definition}

\begin{proposition}
\label{equiker}Let $X$ and $X^{\prime }$ be non-empty sets and $W\subseteq 
\mathcal{F}(X^{\prime },\mathbb{K})$ be a Banach space. Then, every kernel function $%
k:X\times X^{\prime }\rightarrow \mathbb{K}$ such that  $k(x,\cdot )\in W,$ for every $x\in X,$ has a unique associated feature map \break$\Phi :X\rightarrow W\subseteq \mathcal{F}(X^{\prime },\mathbb{K})$, and
\textit{vice versa}.
\end{proposition}

\begin{proof}
Given a kernel function $k:X\times X^{\prime }\rightarrow \mathbb{K}$, if $%
k(x,\cdot )\in W\subseteq \mathcal{F}(X^{\prime },\mathbb{K}),$ for every $%
x\in X,$ then its associated feature map is defined by $\Phi :X\rightarrow W$
where $\Phi (x):=k(x,\cdot )$, for every $x\in X.$

Conversely, given a feature map $\Phi :X\rightarrow W\subseteq \mathcal{F}%
(X^{\prime },\mathbb{K}),$  define its associated kernel function as $%
k:X\times X^{\prime }\rightarrow \mathbb{K}$ given by $k(x,x^{\prime
})=\delta _{x}^{\prime }(\Phi (x))$, where $k(x,\cdot ):=\Phi (x)$, for
every $x\in X$ and $x^{\prime }\in X^{\prime }$.
\end{proof}

\begin{remark}
\label{kernelaso}According to the above proposition, if $B_{\Phi }$ is the RKBS associated with the feature map $\Phi
:X\rightarrow W$ then, $B_{\Phi }$ can
be defined also from the perspective of the kernel. That is%
\begin{equation*}
B_{\Phi }^{k}:=\{\widehat{f}_{\mu }(\cdot ):\mu \in W_{0}^{\ast }\},
\end{equation*}%
where $\widehat{f}_{\mu }(x):=\mu (k(x,\cdot ))$ and $\left\Vert \widehat{f}%
_{\mu }\right\Vert :=\left\Vert \mu \right\Vert $. Therefore, \ $B_{\Phi
}^{k}=B_{\Phi }$ where $\Phi (x):=k(x,\cdot ),$ for every $x\in X$, and we say that $B_{\Phi }^{k}$ is the\textbf{\ RKBS associated with the
kernel function }$k.$
\end{remark}

Our next goal is to show that, in contrast to the case of the RKHS, not all
RKBS are related to a feature map.

From now on, given a Banach space $W,$ we will denote its topological dual
by $W^{\ast }$, and by $J:W\rightarrow W^{\ast \ast }$ the canonical
embedding defined by $J(w)=J_{w},$ where $J_{w}(f)=f(w)$ for every $f\in
W^{\ast }.$ As it is well known, $J$ is an isometry; that is, $\left\Vert
J_{w}\right\Vert =\left\Vert w\right\Vert $ for every $w\in W$ where the
norm on $W^{\ast \ast }$ is the standard operator norm (see, for example \cite[Proposition 1.11.3]{megginson1998introduction}).

As stated in \cite[Definition 1.3.24]{dales2025banach}:

\vspace{0.1cm}
\begin{definition}
Let $B$ be a Banach space. A \textbf{predual} of $B$ is a pair $(B_\ast,\tau )$ consisting
of a Banach space $B_\ast$ and a linear homeomorphism $\tau :B\to (B_\ast)^{\ast
}.$ The predual is said to be isometric if $\tau$ is an isometry. The space $B$ is called a \textbf{dual}
space if it admits a predual $(B_\ast,\tau )$.
\end{definition}

In this work we will consider only isometric preduals $(B_\ast,\tau ),$ which means that the map  $\tau :B\to (B_\ast)^{\ast }$ is an isometric isomorphism. Motivated by this, we introduce the following definition.

\begin{definition}
\label{def1}
Let $W_1$ and $W_2$ be Banach spaces. We say that $W_1$ is a \textbf{copy} of $W_2$ if there exists an isometric isomorphism $\varphi:W_1\to W_2$, in which case we write $W_1 \overset{\varphi}{\equiv}W_2$. Thus, $\varphi$ is a linear and bijective map such that $\|\varphi(w_1)\|= \|w_1\|$, for every $w_1\in W_1$. Moreover, if $w_1\in W_1$ and $w_2\in W_2$ satisfy
$\varphi(w_1)=w_2$, then we  write $w_1\overset{\varphi}{\equiv}w_2$.
\end{definition}

Many times, when the isometric isomorphism is clear (or irrelevant), we will simply write $W_1\equiv W_2$ and $w_1\equiv w_2$.

Recall that if $W_1$ and $W_2$ are Banach spaces and $\varphi:W_1\to W_2$ is an isomorphism, then the adjoint map of $\varphi$, which is the  map $\varphi^\ast:W_2^\ast\to W_1^\ast$ defined by 
\begin{equation*}
    [\varphi^\ast(w_2^\ast)](w_1):=w_2^\ast(\varphi(w_1)),
\end{equation*}
for every $w_1\in W_1$, it is also an isomorphism (see \cite[Proposition 4.27]{van_Neerven_2022} for details). Due to this, we have the following result, whose proof is straightforward.

\vspace{0.1cm}
\begin{proposition}
\label{prop1}
    Let $W_1$ and $W_2$ be Banach spaces. If $W_1$ is a copy of $W_2$, then $W_2^\ast$ is a copy of $W_1^\ast$. Indeed, if $W_1 \overset{\varphi}{\equiv}W_2$, then $W_2^\ast\overset{\varphi^\ast}{\equiv}W_1^\ast$. Moreover, if $w_i^\ast\in W_i^\ast$ for $i=1,2$, then $w_2^\ast\equiv w_1^\ast$ if and only if $w_1^\ast=w_2^\ast\varphi$.
\end{proposition}

Note that if  $B\subseteq \mathcal{F}(X,\mathbb{K})$ is an RKBS having a
isometric predual $B_{\ast }$, then $B\overset{\varphi }{\equiv }(B_{\ast
})^{\ast }$ and, by Proposition \ref{prop1}, we also have that $(B_{\ast
})^{\ast \ast }\overset{\varphi ^{\ast }}{\equiv }B^{\ast }$.

In the next result we characterize the existence of a feature map in an RKBS.

\begin{proposition}
\label{propo}
    Let $B\subseteq\mathcal{F}(X,\mathbb{K}) $ be an RKBS and let $B_\ast$ be a Banach space such that $B\overset{\varphi}{\equiv}(B_\ast)^\ast$. Then the following properties are equivalent:
    
    $\mathrm{(i)}$ For every $x\in X$, we have that $J_{w_x}\overset{\varphi^\ast}{\equiv}\delta_x$, for some $w_x\in B_\ast.$
    
    $\mathrm{(ii)}$ There exists a map $\Phi:X\to B_\ast$, such that $J_{\Phi(x)}\overset{\varphi^\ast}{\equiv}\delta_x$, for every $x\in X$.

    Moreover, if this properties are fulfilled, then $B_\ast=\overline{\mathrm{Lin}\{\Phi (x):x\in X\}}^{||\cdot||}$ and $$B=B_\Phi:=\{f_\mu(\cdot):=\mu(\Phi(\cdot)): \mu \in (B_\ast)^\ast\}.$$
\end{proposition}

\begin{proof}
    $\mathrm{(ii)\Longrightarrow (i)}$ This statement is immediate (set $w_x:=\Phi(x),$ for every $x\in X).$

    $\mathrm{(i)\Longrightarrow (ii)}$ First, we observe that if $J_{w_x}\overset{\varphi^\ast}{\equiv}\delta_x$, for some $w_x\in B_\ast,$ such an element $w_x $ must be unique. Indeed, if there exist $w_x,\widetilde{w}_x\in B_\ast$ such that $J_{w_x}\overset{\varphi^\ast}{\equiv}\delta_x$ and $J_{\widetilde{w}_x}\overset{\varphi^\ast}{\equiv}\delta_x$, then $J_{{w}_x}=J_{\widetilde{w}_x}=(\varphi^\ast)^{-1}(\delta_x)$, so that $J_{w_x-\widetilde{w}_x}=0$, which means that $w_x=\widetilde{w}_x$. Consequently, we can define $\Phi
:X\rightarrow B_{\ast }$ as $\Phi (x)=w_{x},$ for every $x\in X.$
    
    Finally, that  $B=B_{\Phi }$ follows from the fact that $%
    J_{\Phi (x)}\overset{\varphi ^{\ast }}{\equiv }\delta _{x}$. By Proposition \ref{prop1} this means that 
    $J_{\Phi (x)}\varphi =\delta _{x}$. We claim that the space ${\mathrm{Lin}%
\{\Phi (x):x\in X\}\subseteq }$ $B_{\ast }$ is such that $$({\mathrm{Lin}\{\Phi (x):x\in X\}})^\bot=\{0\}.$$ In fact, since $(B_\ast)^\ast=\varphi(B)$, we deduce that, if $f\in B$ is such that $\varphi(f)(\Phi(x))=0$ for every $x\in X$, then $J_{\Phi (x)}(\varphi(f))=0$ for every $x\in X$, so that $\delta_x(f)=0$, for every $x\in X$ and hence $f=0$, so that $\varphi(f)=0$ which proves the claim. By~\cite[Corollary~2.9]{conway1994course}, we have
$$B_\ast=\overline{\mathrm{Lin}\{\Phi (x):x\in X\}}^{\left\Vert \cdot
\right\Vert}.$$
    Thus, for every $f\in B$,   
    \begin{equation*}
    f_{\varphi (f)}(x):=\varphi (f)(\Phi (x))=J_{\Phi (x)}(\varphi (f))=\delta
    _{x}(f)=f(x)\text{ \quad } \hspace{0.1cm}(x\in X),
    \end{equation*}%
    so that $f=f_{\varphi (f)}$ and hence 
    \begin{equation*}
    B:=\{f:f\in B\}=\{f_{\varphi (f)}:f\in B\}=\{f_{\varphi (f)}:\varphi (f)\in
    (B_\ast)^{\ast }\}=\{f_{\mu }:\mu \in (B_\ast)^{\ast }\}=B_\Phi,
    \end{equation*} as desired.
\end{proof}

We have shown that if $B\subseteq \mathcal{F}(X,\mathbb{K})$ is an RKBS then $%
\delta _{x}\in B^{\ast }$ for every $x\in X$ and the possibilities are two:
either $B$ admits a predual $B_{\ast }$ such that  $B \, { \equiv} \, (B_{\ast })^{\ast } $ and $ \delta _{x}$ can be identified with a unique element in  $J(B_{\ast }) \subseteq (B_{\ast })^{\ast \ast } { \equiv}  B^{\ast},$  or no such predual exists.

The first option means that there exist  $\Phi
:X\rightarrow B_{\ast }$ with $B_{\ast }:=\overline{\mathrm{Lin}\{\Phi
(x):x\in X\}}^{\left\Vert \cdot \right\Vert }$ such that  $B=B_{\Phi }.$ This 
justifies the following definition.

\vspace{0.1cm}
\begin{definition}
\label{fea}
Let $B\subseteq \mathcal{F}(X,\mathbb{K})$ be an RKBS. We say that $B$ is a 
\textbf{featured RKBS }if $B=B_{\Phi }$ for some feature map  $\Phi
:X\rightarrow W\subseteq \mathcal{F}(X^{\prime },\mathbb{K}).$ This means
that 
\begin{equation*}
B=B_{\Phi }:=\{f_{\mu }(\cdot ):=\mu (\Phi (\cdot )):\mu \in W_{0}^{\ast }\},
\end{equation*}%
where $W_{0}:=\overline{\mathrm{Lin}\{\Phi (x):x\in X\}}^{\left\Vert \cdot
\right\Vert },$ and $\left\Vert f_{\mu }\right\Vert :=\left\Vert \mu
\right\Vert .$  

Note that $B=B_{\Phi } \overset{\varphi }{\equiv } W_{0}^*$ where $\varphi (f_\mu)=\mu$ (see Example \ref{pro}).

\end{definition}

The following result is derived from  Proposition~\ref{propo}.

\vspace{0.1cm}
\begin{corollary}
\label{thm216}
Let $B\subseteq \mathcal{F}(X,\mathbb{K})$ be an RKBS and, for every $x \in X$, let $\delta _{x}: B\rightarrow\mathbb{K}$ be the evaluation map given by $\delta_x (f):=f(x)$, for all $f\in B$.Then, $B$ is a featured RKBS if, and only if, $B$ has a predual $B_{\ast }$ such that  $B\overset{\varphi}{\equiv}(B_\ast)^\ast$ and, for every $x \in X,$ there exists $w_{x}\in B_{\ast
}$ such that $J_{w_{x}}\overset{\varphi^\ast}{\equiv} \delta_{x}$ (that is, $\delta_{x}=J_{w_{x}}\varphi$), where $J: B_{\ast }\rightarrow (B_{\ast })^{**} $ is the is the canonical embedding.
\end{corollary}

\begin{proof}
If $ B=B_{\Phi }\,\ $ for some feature map  $\Phi
:X\rightarrow W\subseteq \mathcal{F}(X^{\prime },\mathbb{K}),$ then   $B=B_{\Phi }\overset{\varphi}{\equiv}
W_{0}^{\ast },$ where $$W_{0}:=\overline{\mathrm{Lin}\{\Phi
(x):x\in X\}}^{\left\Vert \cdot \right\Vert },$$ and $\varphi (f_\mu) = \mu$ for every $f_\mu \in B_{\Phi }$. Furthermore, for every $ \mu \in W_{0}^{\ast }, \hspace{0.1cm} f_\mu \in B_{\Phi },$ and $x \in X$,  $$\delta _{x} (f_\mu) =f_\mu (x)= \mu(\Phi (x)) = J_{\Phi (x)}(\mu) = J_{\Phi (x)} (\varphi (f_\mu)),$$
so that $\delta _{x} = J_{\Phi (x)} \varphi$ which means by Proposition  \ref{prop1} that  $J_{\Phi (x)}\overset{\varphi^\ast}{\equiv}\delta _{x}$, for every $x \in X$. Finally, set $B_{\ast }:=W_{0}$ and $w_{x}:=\Phi (x)$ to conclude the proof of the necessary condition. 
The sufficient condition follows directly from Proposition~\ref{propo}.
\end{proof}

The next result summarizes all the equivalent facts we have discussed.

\begin{theorem}
\label{th216} 
Let $B\subseteq \mathcal{F}(X,\mathbb{K})$ be a Banach space.
Consider the following assertions:

$\mathrm{(i)}$ $B$ is an RKBS (that is, for every $x\in X,\,\ $the point
evaluation mapping $\delta _{x}:B\rightarrow \mathbb{K}$, defined by $\delta
_{x}(f)=f(x),$ for $f\in B,$ is continuous; and hence $\delta _{x}\in
B^{\ast }$).

$\mathrm{(ii)}$ $B$ is a featured RKBS (that is, there exists a feature map $%
\Phi :X\rightarrow W\subseteq \mathcal{F}(X^{\prime },\mathbb{K}),$ such
that $B=B_{\Phi }$)$.$

$\mathrm{(iii)}$ $B$ has a predual $B_{\ast }$ such that $B\overset{\varphi }%
{\equiv }(B_{\ast })^{\ast }$ and, for each $x\in X,$ there exists $%
w_{x}\in B_{\ast }$ such that $J_{w_{x}}\overset{\varphi ^{\ast }}{\equiv }%
\delta _{x}$ $.$

$\mathrm{(iv)}$\ $B$ has a predual $B_{\ast }$ such that $B\overset{\varphi}{\equiv }(B_{\ast })^{\ast }$ and there exists a feature map \mbox{$\Phi
:X\rightarrow B_{\ast }$} such that \break $B_{\ast }=\overline{\mathrm{Lin}\{\Phi
(x):x\in X\}}^{\left\Vert \cdot \right\Vert }$ and $J_{\Phi (x)}$ $\overset{%
\varphi ^{\ast }}{\equiv }\delta _{x},$ for every $x\in X.$

$\mathrm{(v)}$ $B$ has a predual $B_{\ast }$ such that $B\overset{\varphi }{%
\equiv }(B_{\ast })^{\ast }$ and there exists a kernel function $k:X\times
X^{\prime }\rightarrow \mathbb{K}$, such that \break$B_{\ast }:=\overline{\mathrm{Lin}%
\{k(x,\cdot ):x\in X\}}^{^{\left\Vert \cdot \right\Vert }}$ and $%
J_{k(x,\cdot )}$ $\overset{\varphi ^{\ast }}{\equiv }\delta _{x},$ for every 
$x\in X$.

$\mathrm{(vi)}$ There exists a kernel function $k:X\times X^{\prime
}\rightarrow \mathbb{K}$ such that $B=B_{\Phi }^{k}.$

Then $\mathrm{(ii)\Longleftrightarrow (iii)\Longleftrightarrow
(iv)\Longleftrightarrow (v)\Longleftrightarrow (vi)\Longrightarrow (i),}$
and if $B_{\ast }$ is reflexive then all these assertions are equivalent.
Moreover, if the above assertions are fulfilled then $\Phi $ and $k$ are
unique.
\end{theorem}

\begin{proof} From Proposition \ref{propo} and Corollary  \ref{thm216} we have
$\mathrm{(ii)\iff (iii)\iff (iv)}$. Furthermore, the equivalence $\mathrm{(iv)\iff (v)}$ follows from Proposition \ref{equiker}, whereas
$\mathrm{(ii)\iff (vi)}$ is derived from Remark \ref{kernelaso}. Thus, $$\mathrm{(ii)\Longleftrightarrow (iii)\iff
(iv)\iff (v)\iff (vi).}$$ Trivially, $\mathrm{(ii)\Rightarrow (i).}$
Moreover, if $B_{\ast }$ is reflexive then $B_{\ast }\overset{J}{\equiv }%
(B_{\ast })^{\ast \ast }$ so that assertion $\mathrm{(i)\Rightarrow (iii)\ }$
is obvious since $\delta _{x}\in B^{\ast }\hspace{0.1cm}{\equiv }\hspace{0.1cm}(B_{\ast })^{\ast
\ast }=J(B_\ast).$

Finally, if $B\overset{\varphi }{\equiv }(B_{\ast })^{\ast }$ and $\Phi ,%
\widetilde{\Phi }:X\rightarrow B_{\ast }$ are such that such 
\begin{equation*}
B_{\ast }=\overline{\mathrm{Lin}\{\Phi (x):x\in X\}}^{^{\left\Vert \cdot
\right\Vert }}=\overline{\mathrm{Lin}\{\widetilde{\Phi }(x):x\in X\}}%
^{^{\left\Vert \cdot \right\Vert }}
\end{equation*}%
and $J_{\Phi (x)}$ $\overset{\varphi ^{\ast }}{\equiv }\delta _{x}\overset{%
\varphi ^{\ast }}{\equiv }J_{\widetilde{\Phi }(x)},$ for every $x\in X,$ then 
$J_{\Phi (x)}-J_{\widetilde{\Phi }(x)}=0,$ and hence $\Phi =\widetilde{\Phi }%
.$ Similarly, $k$ is also unique.
\end{proof}

We have established in Definition \ref{fea} that a Banach space $B\subseteq \mathcal{F}(X,\mathbb{K})$ is
a featured RKBS if and only if $B=B_{\Phi }$ for some feature map $\Phi
:X\rightarrow W \subseteq \mathcal{F}(X,\mathbb{K}').$ In this case we say that $\Phi $ is the\textbf{\ feature
map associated with }  $B=B_{\Phi }$ and that the corresponding kernel  $%
k:X\times X^{\prime }\rightarrow \mathbb{K},$ which satisfies that $B=B_{\Phi }^{k}=$ $%
B_{\Phi }$ is the \textbf{kernel function associated }to $B=B_{\Phi }$.
Moreover, the property $J_{\Phi (x)}$ $\overset{\varphi ^{\ast }}{\equiv }%
\delta _{x}$ (equivalently $J_{k(x,\cdot )}$ $\overset{\varphi ^{\ast }}{%
\equiv }\delta _{x}$) for every $x\in X,$ is called the \textbf{reproducing
property.}

\vspace{0.1cm}
\begin{remark}
As stated previously, $B\subseteq \mathcal{F}(X,\mathbb{K})$ is a
featured RKBS if and only if $B=B_{\Phi }$ for some feature map $\Phi
:X\rightarrow W,$ in which case 
\begin{equation*}
W_{0}=\overline{\mathrm{Lin}\{\Phi (x):x\in X\}}^{{\left\Vert \cdot
\right\Vert }}=B_{\ast }.
\end{equation*}%
In many applications, it may happen that the Banach space $W$ is well known, while the
predual $W_{0}=B_{\ast }$ is not. Motivated by this observation, we also define $B_{\Phi }$ via the space $W^{\ast }.$
This is done as follows:%
\begin{equation*}
B_{\Phi }:=\{f_{\mu }(\cdot ):=\mu (\Phi (\cdot )):\mu \in W^{\ast }\}
\end{equation*}%
where $\left\Vert f_{\mu }\right\Vert :=\inf \{\left\Vert \widetilde{\mu }%
\right\Vert :\widetilde{\mu }\in W^{\ast }$ with $f_{\mu }=f_{\widetilde{\mu 
}}\}.$ In fact, 
\begin{equation*}
\left\Vert f_{\mu }\right\Vert :=\inf \{\left\Vert \widetilde{\mu }%
\right\Vert :\widetilde{\mu }\in W^{\ast }\text{ with }f_{\mu }=f_{%
\widetilde{\mu }}\}=\left\Vert \mu +W_{0}^{\bot }\right\Vert =\left\Vert \mu
_{|W_{0}}\right\Vert .
\end{equation*}%
Note that $W^{\ast }/W_{0}^{\bot }\equiv W_{0}^{\ast }$ under the classical 
identification $\mu +W_{0}^{\bot }\equiv \mu _{|W_{0}}$ (see, for instance, \cite[Theorem 1.10.16]{megginson1998introduction}).
\end{remark}

As we know, a Banach space $B\subseteq \mathcal{F}(X,\mathbb{K})$ is a
RKBS if, and only if, $\delta _{x}\in B^{\ast },$ for every $x\in X$ and we
have shown that if $B$ has not a predual then $B$ is not a featured RKBS.
Consequently, there are reproducing kernel Banach\ spaces that are not
featured ones, in contrast to the case of RKHS (see the equivalence between $(\mathrm{i})$ and $(\mathrm{ii})$ in Theorem~\ref{Hilbert}).

To illustrate this point, we present the following example.
\begin{example}
\label{noequi}Let $B=C_{0}(\Omega )$, where $%
\Omega $ is a locally compact Hausdorff space that is not compact. Then, $B$ has not a predual and
nevertheless the evaluation mapping $\delta _{x}$ is continuous, for every $%
x\in X.$ Therefore, $\delta _{x}\in B^{\ast },$ for every $x\in X,$ and there does
not exist a Banach space $W$ $\ $with a feature map $\Phi :X\to
W $ \ such that $B=B_{\Phi }.$ This is because, if this happens, then $%
B=C_{0}(\Omega )$ is isometrically isomorphic to 
\begin{equation*}
(B_{\ast })^{\ast }=\left( \overline{\mathrm{Lin}\{\Phi (x):x\in X\}}\right)
^{\ast }.
\end{equation*}
Thus, since an isometric isomorphism between Banach spaces preserves the
extreme points of the closed unit ball, it follows that the closed unit ball
of $C_{0}(\Omega )\ $has extreme points, which is a contradiction. In fact, by the
Banach-Alaoglu theorem \cite[Theorem 3.15]{rudin1991functional}, the closed unit ball $%
\overline{B}(0,1)$ of $(B_{\ast })^{\ast }$ is w$^{\ast }-$compact and, by
the Krein-Milman theorem \cite[Theorem $3.23$]{rudin1991functional}, the set of extreme points of $\overline{B}(0,1)$
is not empty and $\overline{B}(0,1)$ is the weak$^{\ast }-$closed convex hull of them.

Let show now that the closed unit ball of $C_{0}(\Omega )$ has no extreme
points$.$ Consider $f\in C_{0}(\Omega )$ with $\left\Vert f\right\Vert
_{\infty }=1$ (if $\left\Vert f\right\Vert _{\infty }<1$, then $f$ is clearly not an
extreme point). Let $F_{1}:=\left\vert f\right\vert ^{-1}([0,\frac{1%
}{4}])$, and $F_{2}:=\left\vert f\right\vert ^{-1}([\frac{3}{4},1]).$ Then, $%
F_{1}$ and $F_{2}$ are closed subsets of $\Omega $ as $\left\vert
f\right\vert $ is continuous. \ Moreover, $F_{1}\neq \varnothing $ since
there exists a compact set $K\subseteq \Omega $ such that $\left\vert
f(x)\right\vert <\frac{1}{4},$ for every $x\in \Omega \backslash K.$ Also $%
K\,\neq \varnothing $ as $\left\Vert f\right\Vert _{\infty }=1.$ Similarly, $%
F_{2}\neq \varnothing $ as $\left\Vert f\right\Vert _{\infty }=1.$
Obviolusly, $F_{1}\cap F_{2}=\varnothing $ and the Urysohn's Lemma \cite[p.
28]{naimark1972normed} allows us to find $g:\Omega\to \lbrack 0,\frac{1}{4}]$ such
that $g(x)=\frac{1}{4}$ for every $x\in F_{1}$ and $\ g(x)=0$ for every $%
x\in F_{2}.$ It follows that $\left\vert (f\pm g)(x)\right\vert \leq 1$, for every $x\in \Omega.$ Consequently, $%
\left\Vert f\pm g\right\Vert _{\infty }=1$, and%
\begin{equation*}
f=\frac{1}{2}(f+g)+\frac{1}{2}(f-g),
\end{equation*}%
which proves that $f$ is not an extreme point of the unit
ball of $C_{0}(\Omega ).$ We conclude that $C_{0}(\Omega )$ is not
isometrically isomorphic to $(B_{\ast })^{\ast }$, for any Banach space $%
B_{\ast }.$
\end{example}


\begin{remark}
The predual of a Banach space do not need to be unique. For instance, $c$
and $c_{0}$ are non-isomorphic preduals of $l_{1}$ (see  \cite%
{brown1976uniqueness, rao}). In this sense, $B\subseteq \mathcal{F}(X,\mathbb{K})$
can be a featured Banach space for different feature maps related to
different preduals. However, if $B_{\ast }$ and $\widetilde{B}_{\ast }$ \
are preduals of $B$ (which do not need to be isomorphic) with $B \overset{ \varphi}{\equiv} (B_{\ast })^*$ and $B \overset{\widetilde { \varphi}}{\equiv} (\widetilde{B}_{\ast })^*,$ and if $\Phi
\,:X\to B_{\ast }$ and $\widetilde{\Phi }\,:X\to \widetilde{%
B_{\ast }}$ are such that 
\begin{equation*}
B_{\ast }:=\overline{\mathrm{Lin}\{\Phi (x):x\in X\}}^{||\cdot||}\text{ and }\widetilde{%
B_{\ast }}:=\overline{\mathrm{Lin}\{\widetilde{\Phi }(x):x\in X\}}^{||\cdot||}
\end{equation*}%
then, $B_{\Phi }$ $\equiv (B_{\ast
})^{\ast }$ and $B_{\widetilde{\Phi }}$ $\equiv (\widetilde{B_{\ast }}%
)^{\ast },$ so that $B_{\Phi }$ $\equiv B\equiv B_{\widetilde{\Phi }}$. 

Moreover, if $\delta_x \overset{\varphi ^{\ast}}{\equiv} J_{\Phi(x)}$ and  $\delta_x \overset{\widetilde{\varphi} ^{\ast}}{\equiv}J_{\widetilde{\Phi}(x)}$ (which means that $\delta_x =J_{\Phi(x)}\varphi= J_{\widetilde{\Phi}(x)} \widetilde{\varphi}$),  for every $x\in X$, then $B=B_{\Phi}= B_{\widetilde{\Phi}}$.
\end{remark}

We know that if a Banach space $B\subseteq \mathcal{F}(X,\mathbb{K})$ is a
featured RKBS then $B$ has a predual $B_{\ast }$ and \thinspace $B=B_{\Phi }$
for some feature map $\Phi :X\rightarrow B_{\ast }$ \ such that $$ B_{\ast }:=\overline{\mathrm{Lin}\{\Phi (x):x\in X\}}^{||\cdot||}$$ (see Theorem \ref%
{th216}). If, additionally, $B_{\ast }$ is an RKBS then we obtain the
following result.

\begin{theorem}
\label{special}Let $B_{\Phi }$ $\subseteq \mathcal{F}(X,\mathbb{K})$ be a
featured RKBS with predual $B_{\ast }\subseteq $ $\mathcal{F}(X^{\prime },%
\mathbb{K)}$. Let $\Phi :X\rightarrow B_{\ast }$ the corresponding feature
map and $k:X\times X^{\prime }\rightarrow \mathbb{K}$ the associated kernel.
If $B_{\ast }$ is an RKBS then, for every $x^{\prime }\in X^{\prime }$ we
have that $k(\cdot ,x^{\prime })\in B_{\Phi }$ and $\delta _{x^{\prime }}\in
(B_{\ast })^{\ast },$ being  $f_{ \delta
_{x^{\prime }}} = k(\cdot ,x^{\prime })$. Moreover,  
\begin{equation*}
 B_{\Phi }=\overline{\mathrm{Lin}%
\{k(\cdot ,x^{\prime }):x^{\prime }\in X^{\prime }\}}^{w\ast } \equiv \overline{\mathrm{Lin}\{\delta _{x^{\prime }}:x^{\prime }\in X^{\prime }\}}%
^{w\ast }=(B_{\ast })^{\ast },
\end{equation*}
with $k(\cdot ,x^{\prime })\equiv \delta
_{x^{\prime }}$.  
\end{theorem}

\begin{proof}
That $B_{\ast }$ is an RKBS means that $\delta _{x^{\prime }}\in (B_{\ast
})^{\ast }$ for every $x^{\prime }\in X^{\prime },$ where $$B_{\ast }:=W_{0}:=\overline{\mathrm{Lin}\{\Phi (x):x\in
X\}}^{{\left\Vert \cdot
\right\Vert }}.$$ By Definition~\ref{fea}, we have $B_{\Phi }=\{f_{\mu _{0}}(\cdot ):=\mu _{0}(\Phi (\cdot )):\mu _{0}\in
(B_{\ast })^{\ast }=(W_{0})^{\ast }\}$ with $\left\Vert f_{\mu _{0}}\right\Vert =\left\Vert \mu _{0}\right\Vert .$
Under the identification $(B_{\ast })^{\ast }=W_{0}^{\ast }$ $%
\equiv B_{\Phi }$ given by $$\mu _{0}\in W_{0}^{\ast }\equiv f_{\mu
_{0}}(\cdot )\in B_{\Phi },$$ we have that 
$\delta _{x^{\prime }}\in W_{0}^{\ast }\equiv f_{\delta _{x^{\prime }}}(\cdot
)=k(\cdot ,x^{\prime })\in B_{\Phi },$ and consequently 
\begin{equation*}
\mathrm{Lin}\{\delta _{x^{\prime }}:x^{\prime }\in X^{\prime }\}\subseteq
(B_{\ast })^{\ast }=W_{0}^{\ast }\equiv \mathrm{Lin}\{k(\cdot ,x^{\prime
}):x^{\prime }\in X^{\prime }\}\subseteq B_{\Phi },
\end{equation*}%
and we claim that%
\begin{equation}
\overline{\mathrm{Lin}\{k(\cdot ,x^{\prime }):x^{\prime }\in
X^{\prime }\}}^{w\ast }=B_{\Phi } \equiv \overline{\mathrm{Lin}\{\delta _{x^{\prime }}:x^{\prime }\in X^{\prime }\}}%
^{w\ast }=(B_{\ast })^{\ast } .  \label{wea}
\end{equation}

To prove the claim, note that the identification $\delta _{x^{\prime
}}\equiv f_{\delta _{x^{\prime }}}(\cdot )$ is obvious. Moreover, $f_{\delta
_{x^{\prime }}}(x)=$ $\delta _{x^{\prime }}(\Phi (x))$ and, since $\Phi
(x)=k(x,\cdot )$ by Proposition \ref{equiker}, we have that 
\begin{equation*}
f_{\delta _{x^{\prime }}}(x)=\delta _{x^{\prime }}(\Phi (x))=\delta
_{x^{\prime }}(k(x,\cdot ))=k(x,x^{\prime })=\delta _{x}(k(\cdot ,x^{\prime
}))=[k(\cdot ,x^{\prime })](x),
\end{equation*}%
for every $x\in X.$ Therefore, $f_{\delta _{x^{\prime }}}=k(\cdot ,x^{\prime
})\in B_{\Phi }.$ On the other hand, let%
\begin{equation*}
V_{0}:=\{\delta _{x^{\prime }}:x^{\prime }\in X^{\prime }\}\subseteq
W_{0}^{\ast }=(B_{\ast })^{\ast }.
\end{equation*}%
Then, 
\begin{equation*}
^{\perp }V_{0}:=\{g\in B_{\ast }:\delta _{x^{\prime }}(g)=g(x^{\prime })=0,%
\text{ for every }x^{\prime }\in X^{\prime }\}=\{g\in B_{\ast }:g=0\}=\{0\}.
\end{equation*}%
From the bipolar theorem \cite[Theorem 1.8.]{conway1994course}, we obtain that 
\begin{equation*}
B_{\Phi }=\{0\}^{\perp }=(^{\perp }V_{0})^{\perp }=\overline{\mathrm{Lin}%
V_{0}}^{w\ast }.
\end{equation*}%
Therefore (\ref{wea}) follows, concluding the proof. 
\end{proof}

\begin{definition}
Let $B_{\Phi }$ $\subseteq \mathcal{F}(X,\mathbb{K})$ be the RKBS\ associated
with a feature map $\Phi :X\rightarrow B_{\ast }$ where  $$B_{\ast }:=\overline{%
\mathrm{Lin}\{\Phi (x):x\in X\}}^{\left\Vert \cdot \right\Vert }.$$ The space 
$B_{\Phi }$ is said to be a \textbf{special featured RKBS} if its predual $%
B_{\ast }\subseteq $ $\mathcal{F}(X^{\prime },\mathbb{K)}$ is an RKBS.
\end{definition}

As showed in Theorem \ref{special}, if $B_{\Phi }$ is a special featured RKBS%
$\ $then, $$B_{\Phi }=\overline{\mathrm{Lin}\{k(\cdot ,x^{\prime }):x^{\prime
}\in X^{\prime }\}}^{w\ast }.$$

\begin{remark}
A Banach space $B$ (with a predual  $B_{\ast }$) may have a subspace $M$ that is $w^{\ast
}-$dense but not norm-dense  in $B$. For instance, considering $%
l_{\infty }\equiv (l_{1})^{\ast }$ we have that $\overline{c_{0}}^{w^{\ast
}}=l_{\infty }$ meanwhile $\overline{c_{0}}^{\left\Vert \cdot \right\Vert
}=c_{0}.$ Similarly, taking $L_{\infty }([0,1])\equiv (L_{1}([0,1]))^{\ast
} $ we have that $\overline{C([0,1]))}^{\left\Vert \cdot \right\Vert
}=C([0,1]))$ and, nevertheless, $\overline{C([0,1]))}^{w^{\ast }}=L_{\infty
}([0,1])$. Note that by Goldstine's Theorem \cite[Theorem~1 of Subsection 17.2.2]{kadets2018course}, every 
Banach space is $w^\ast-$dense in its second dual.
\end{remark}

The previous theorem asserts  that, \ if $B_{\Phi }$ is a especial featured space, 
then $\ B_{\Phi }$, viewed as a dual Banach space, is the smallest $w^*$-closed  Banach space containing $k(\cdot
,x^{\prime })$, for every $x^{\prime }\in X^{\prime }$, and in which the point
evaluation maps $\delta _{x}$ are continuous, for all $x \in X$.

\section{Norm Interpolation and Regularization in Featured RKBS}
\label{sec:rkbs}

In this section, we formulate supervised learning problems in the setting of featured reproducing kernel Banach spaces. From a machine learning perspective, learning from a finite dataset is typically posed as either a regularized empirical risk minimization problem or, in the interpolation regime, as the search for a minimal-norm function consistent with the observed data. In reproducing kernel Hilbert spaces, such formulations are well understood and lead to finite-dimensional solutions through classical representer theorems; in Banach spaces, however, the lack of Hilbertian geometry fundamentally changes this picture.

We show that featured RKBSs provide a natural function-space setting in which learning can still be interpreted as minimal-norm interpolation or regularization, while making explicit the additional structural assumptions required in the Banach case. We establish existence results using functional-analytic arguments and clarify how the geometry of the underlying Banach norm influences the solution set. In contrast to the Hilbert setting, representer-type solutions are no longer automatic. Nevertheless, we provide necessary and sufficient conditions for their existence, which yield generic representation theorems in the context of RKBs. These theorems are then used in subsequent sections to address training processes.

We emphasize that featured RKBSs provide a structural framework for learning formulations, while the existence of finite representer solutions requires additional assumptions developed later.

Let $W\subseteq \mathcal{F}(X^{\prime },\mathbb{K})$ be a Banach space, and $%
\Phi :X\rightarrow W$ be a feature map. Assume that $$W=\overline{%
\operatorname{Lin}\{\Phi (x):x\in X\}}^{^{\left\Vert \cdot \right\Vert }},$$ is 
not restrictive. Let $B_{\Phi }\subseteq \mathcal{F}(X,\mathbb{K})$ 
be the  RKBS 
associated with $\Phi $, that is, 
\begin{equation*}
B_{\Phi }:=\{f_{\mu }(\cdot )=\mu (\Phi (\cdot )):\mu \in W^{\ast }\},
\end{equation*}%
with $\left\Vert f_{\mu }\right\Vert :=\left\Vert \mu \right\Vert ,$ for
every $f_{\mu }\in B_{\Phi }.$ The kernel map associated with the feature map~$\Phi$ is the function
\mbox{$k : X \times X' \to \mathbb{K}$} given
by $k(x,x^{\prime }):=\delta _{x^{\prime }}(\Phi (x)),$ for every $x\in X$
and $x^{\prime }\in X^{\prime }$. In this way, the space $B_{\Phi }$ can
also be expressed in terms of $k$, as follows 
\begin{equation*}
B_{\Phi }:=\{f_{\mu }(\cdot )=\mu (k(x,\cdot )):\mu \in W^{\ast }\}.
\end{equation*}

Given the dataset $\mathcal{D}_{m}=\{(x_{i},y_{i})\}_{i=1}^{m} \subseteq (X,\mathbb{K})$, let $%
\mathcal{L}:B_{\Phi }\rightarrow \mathbb{K}^{m}$ be the mapping given for every $f_{\mu } \in B_{\Phi }, $ by  
\begin{equation*}
\mathcal{L}(f_{\mu })=\big(f_{\mu }(x_{1}),\dots ,f_{\mu }(x_{m})\big).
\end{equation*}%
Set $\boldsymbol{y}=(y_{1},\dots ,y_{m})\in \mathbb{K}^{m}$ and define 
\begin{equation*}
\mathcal{M}_{{\boldsymbol{y}}}:=\{f_{\mu }\in B_{\Phi }:\mathcal{L}(f_{\mu })=%
\boldsymbol{y}=(y_{1},\dots ,y_{m})\in \mathbb{K}^{m}\}.
\end{equation*}%
Then, the MNI problem in $B_{\Phi }$ associated with the dataset $\mathcal{D}%
_{m}$ is the following 
\begin{equation*}
 \operatorname*{arg\,min}\{\Vert f_{\mu }\Vert :f_{\mu }\in \mathcal{M}_{\boldsymbol{y}}\}.
\end{equation*}

This is what is meant, within this framework, by \emph{learning} a function
in $B_{\Phi }$ from the dataset $\mathcal{D}_{m}$. To solve this problem
define the Banach space 
\begin{equation*}
W^{\mathrm{data}}:=\mathrm{Lin}\{k(x_{1},\cdot ),\dots ,k(x_{m},\cdot
)\}\subseteq W \subseteq \mathcal{F}(X^{\prime },\mathbb{K}).
\end{equation*}%
Without loss of generality, we may assume that $\{k(x_{1},\cdot ),\dots
,k(x_{m},\cdot )\}$ are linearly independent (otherwise, we can work with a maximal linearly independent subset). The data map 
\begin{equation*}
\mu^{\mathrm{data}}:W^{\mathrm{data}}\rightarrow \mathbb{K}
\end{equation*}%
is defined as the unique linear map such that 
\begin{equation*}
\mu^{\mathrm{data}}(k(x_{k},\cdot ))=y_{k},\quad \text{for }k=1,\dots ,m.
\end{equation*}%
Note that  $f_{\mu }\in \mathcal{M}_{\boldsymbol{y}}$ if, and only if, $\mu$  is an extension of $%
\mu^{\mathrm{data}}$ to $W$ (as $\mu |_{W^{\mathrm{data}}}=\mu^{%
\mathrm{data}}$ means that $f_{\mu }(x_k)=y_k,$ for every $k=1,\dots ,m$).  

The Hahn-Banach Theorem  \cite[Corollary~6.5]%
{conway1994course} guarantees the existence of a \textbf{Hahn-Banach extension
of }$\mu^{\mathrm{data}}$ to $W$. This is a map $$\hat{\mu}^{\mathrm{data}}: W\rightarrow \mathbb{K}$$ such that $$\hat{\mu}^{\mathrm{data}} |_{W^{\mathrm{data}}}=\mu^{%
\mathrm{data}} \quad \text{and}\quad\Vert \hat{\mu}^{\mathrm{data}}\Vert =\Vert \mu^{\mathrm{data}}\Vert. $$ Thus,  $f_{\hat{\mu}^{\mathrm{data}}}\in \mathcal{M}_{\boldsymbol{y}}$ with 
$\Vert f_{\hat{\mu}^{\mathrm{data}}}\Vert =\Vert \mu^{\mathrm{data}}\Vert $. Since this extension has minimal norm, we obtain the following result. 

\vspace{0.1cm}
\begin{theorem}[MNI-Problem in $B_{\Phi }$]
\label{mnip} Let $W\subseteq \mathcal{F}(X^{\prime },\mathbb{K})$ be a
Banach space, and  $\Phi :X\rightarrow W$ a mapping such that $W=\overline{%
\mathrm{Lin}\{\Phi (x):x\in X\}}^{^{\left\Vert \cdot \right\Vert }}.$  Let $k:X\times X^{\prime }\rightarrow \mathbb{K}$ be the kernel associated with $\Phi$, and $%
B_{\Phi }\subseteq \mathcal{F}(X,\mathbb{K})$ the RKBS associated with $%
\Phi $. Given a dataset 
\begin{equation*}
\mathcal{D}_{m}=\{(x_{i},y_{i})\}_{i=1}^{m}
\end{equation*}%
consider the corresponding MNI problem 
\begin{equation}
\operatorname*{arg\,min}\{\Vert f_{\mu }\Vert :f_{\mu }\in \mathcal{M}_{\boldsymbol{y}}\}.
\label{MNI}
\end{equation}%
Then, $f_{\mu }\in B_{\Phi }$ is a solution of~\eqref{MNI} if, and only if,
$\mu $ is a Hahn-Banach extension of the associated data map $\mu^{%
\mathrm{data}}:W^{\mathrm{data}}:=\mathrm{Lin}\{k(x_{1},\cdot ),\dots ,k(x_{m},\cdot
)\}\rightarrow \mathbb{K}$, defined by 
\begin{equation*}
\mu^{\mathrm{data}}(k(x_{k},\cdot ))=y_{k},\quad \text{for }k=1,\dots ,m.
\end{equation*}%
Moreover, the set $\mathcal{S}_{\boldsymbol{y}}$ of solutions of~\eqref{MNI}
 is non-empty and convex.
\end{theorem}

According to Hadamard's classical criteria, a problem is well-posed if it
satisfies three conditions: existence of a solution, uniqueness of the
solution, and stability (i.e., the solution depends continuously on the data
so that small perturbations produce small changes in the solution), see \cite%
{Hadamard1923,uesbook}. In learning problems, stability often fails, 
as even small perturbations in the training data may lead to drastic changes in the solution. Consequently, this ill-posedness necessitates the introduction of regularization.

To regularize the MNI problem~\eqref{MNI}, one usually adds a
regularization term to the data-fidelity term as follows: 
\begin{equation}
\operatorname*{arg\,min} \left\{ Q_{\boldsymbol{y}}(\mathcal{L}(f_{\mu }))+\lambda _{0}\,\varphi
\!\left( \left\Vert f_{\mu }\right\Vert \right) :f_{\mu }\in B_{\Phi
}\right\} \quad \text{(Regularization Problem)},  \label{regu}
\end{equation}%
where $Q_{\boldsymbol{y}}:\mathbb{K}^{m}\rightarrow \mathbb{R}^{+}_{0}$ is a loss
function (with $Q_{\boldsymbol{y}}(\mathbf{y})=0$), $\varphi :\mathbb{R}^{+}\rightarrow \mathbb{R}^{+}_0$ is a
regularization function, and $\lambda _{0}>0$ is a regularization parameter. 
In general, both $Q_{\mathbf{y}}$ and $\varphi $ \ convex and coercive. Also 
$Q_{\mathbf{y}}$ is lower semicontinuous meanwhile $\varphi $ is continuous
and nondecreasing.

Note that the Regularization Problem is more general and flexible than the MNI Problem, since it no longer requires the strict constraint $\mathcal{L}(f_{\mu})=\boldsymbol{y}$. This allows for better adaptation to contexts where the data
are noisy or cannot be exactly interpolated. Recall that\textbf{\
overfitting }occurs when a model learns the training data too well,
including noise or irregularities that do not belong to the underlying
pattern. Such models often perform poorly on unseen data, as they capture
artifacts rather than general structure.

A solution to the Regularization Problem also exists. Indeed, if $f_{\mu
}^{0}\in B_{\Phi }$ is a solution of the associated MNI-problem, $ \operatorname*{arg\,min}\{\Vert f_{\mu }\Vert :f_{\mu }\in \mathcal{M}_{\boldsymbol{y}}\},$ then $%
\left\Vert f_{\mu }^{0}\right\Vert =\left\Vert \mu^{\mathrm{data}}\right\Vert $
$\ $and it follows that a solution of (\ref{regu}) can not have a norm
greater than $\left\Vert \mu^{\mathrm{data}}\right\Vert $. 
Consequently, if $\overline{B_{(W^{\mathrm{data}})^\ast}}(0, \left\Vert \mu^{\mathrm{data}}\right\Vert)$ denotes the closed ball of $(W^{\mathrm{data}})^\ast$ centered at zero with radius \ $\left\Vert \mu^{\mathrm{data}}\right\Vert$ then, the function $$\mathcal{R}_{\boldsymbol{y}}:\overline{B_{(W^{\mathrm{data}})^\ast}} \left( 0, \left\Vert \mu^{\mathrm{data}}\right\Vert \right)\rightarrow \mathbb{R}^{+}_{0}$$
given by $$\mathcal{R}_{\boldsymbol{y}}(\widetilde{\mu} )=Q_{\boldsymbol{y}}((\widetilde{\mu}(k(x_1,\cdot)),\cdots, \widetilde{\mu} (k(x_m,\cdot)))+\lambda _{0}\,\varphi \!\left( \left\Vert \widetilde{\mu}  \right\Vert
\right) ,$$ for every $\widetilde{\mu} \in \overline{B_{(W^{\mathrm{data}})^\ast}}(0, \left\Vert \mu^{\mathrm{data}}\right\Vert) $ attains its minimum \cite[Lemma~2, p.~62]{pareto}. If $\widetilde{\mu} _0: W^\text{data} \rightarrow\mathbb{K}$ is a functional at which $\mathcal{R}_\mathbf{y}$ attains its minimum, then any Hahn-Banach extension of $\mu_0$ is a solution of the problem (\ref{regu}). Therefore, the next results follows:

\begin{theorem}
\label{mni-re} 
Let $W\subseteq \mathcal{F}(X^{\prime },\mathbb{K})$ be a
Banach space, and  $\Phi :X\rightarrow W$ where $W=\overline{%
\mathrm{Lin}\{\Phi (x):x\in X\}}^{^{\left\Vert \cdot \right\Vert }}.$  Let $k:X\times X^{\prime }\rightarrow \mathbb{K}$ be the kernel associated with $\Phi$, and $%
B_{\Phi }\subseteq \mathcal{F}(X,\mathbb{K})$  the RKBS associated with $%
\Phi $. Given a dataset $\mathcal{D}_{m}=\{(x_{i},y_{i})\}_{i=1}^{m}$, consider 
 the regularization problem given by 
\begin{equation}
\operatorname*{arg\,min} \left\{ Q_{\boldsymbol{y}}(\mathcal{L}(f_{\mu }))+\lambda _{0}\,\varphi
\!\left( \left\Vert f_{\mu }\right\Vert \right) :f_{\mu }\in B_{\Phi
}\right\} .  \label{regula}
\end{equation}%
Define the data functional  $\mu^{\mathrm{data}}:W^{\mathrm{data}}:=\mathrm{Lin}\{k(x_{1},\cdot ),\dots ,k(x_{m},\cdot
)\}\rightarrow \mathbb{K}$, by 
\begin{equation*}
\mu^{\mathrm{data}}(k(x_{k},\cdot ))=y_{k},\quad \text{for }k=1,\dots ,m.
\label{datam}
\end{equation*}%
Then, the function $\overline{B_{(W^{\mathrm{data}})^\ast}}(0, \left\Vert \mu^{\mathrm{data}}\right\Vert)\rightarrow \mathbb{R}_0^+$
given by 
$$\mathcal{R}_{\boldsymbol{y}}(\widetilde{\mu} )=Q_{\boldsymbol{y}}((\widetilde{\mu}  (k(x_1,\cdot)),\cdots, \widetilde{\mu}  (k(x_m,\cdot)))+\lambda _{0}\,\varphi \!\left( \left\Vert \widetilde{\mu}  \right\Vert
\right) ,$$
 attains its minimum. Furthermore, the set of solutions of the minimization problem (\ref{regula}) consists of all functions  $f_{\widetilde{\mu}} \in B_{\Phi}$ for which  $\widetilde{\mu}  \in W^*$ is a Hahn-Banach extension of a minimizer   $\widetilde{\mu}_{0}\in \overline{B_{(W^{\mathrm{data%
}})^\ast}}(0,\left\Vert \mu^{\mathrm{data}}\right\Vert)$ of  $\mathcal{R}_{\boldsymbol{y}}$. 
\end{theorem}
\noindent
The Hahn–Banach theorem guarantees the existence of solutions to the minimal-norm interpolation problems under consideration. However, in general, the explicit construction of the desired extensions is unknown, except in certain specific cases (for instance in Hilbert spaces, via the Riesz--Fr\'{e}chet
theorem),  as the Hahn-Banach theorem is merely an existence theorem for the desired  functionals. In the next section, we investigate when a solution
can be built from the elements of $W^{\mathrm{data}}.$

\section{Representer theorems in special featured RKBS}
\label{sec:rep_theorems}

Representer theorems play a central role in kernel-based learning by reducing infinite-dimensional optimization problems to finite-dimensional ones. In reproducing kernel Hilbert spaces, such results follow directly from the Hilbertian structure and hold for broad classes of regularized learning problems. In Banach spaces, however, representer theorems are no longer guaranteed and their validity depends on additional geometric and duality properties of the underlying function space.

In this section, we investigate representer theorems in the setting of special featured RKBSs. We show that when a featured RKBS admits a suitable predual structure and satisfies additional assumptions—such as dataset-dependent conditions ensuring compatibility between evaluation functionals and the Banach geometry—solutions to minimal-norm interpolation and regularization problems admit finite kernel representations. These results precisely delineate when classical kernel-based learning principles extend to Banach spaces and when they fundamentally fail.

The classical Representer Theorem for RKHS  \cite{kimeldorf1971some} states that, 
given a dataset $$\mathcal{D}_{m}:=\{(x_{i},y_{i})\}_{i=1}^{i=%
m},$$ the solution of the associated MNI-norm problem in an RKHS always
exists, is unique and can be written as a finite linear combination of
functions that are kernel evaluations at the training points;\ that is $f^{
\mathrm{sol}}(\cdot ):=\sum_{i=1}^{m}\alpha _{i}k(\cdot
,x_{i}).$

If $B_{\Phi
}\subseteq \mathcal{F}(X,\mathbb{K})$ is a special featured RKBS,  that is a featured RKBS whose predual $B_{\ast }\subseteq 
\mathcal{F}(X^{\prime },\mathbb{K})$ is an RKBS then, by Theorem \ref{special}, we
know that $k(\cdot ,x^{\prime })\in B_{\Phi }$ and  also 
$\delta _{x^{\prime }}\in (B_{\ast })^{\ast },$ with $B_{\Phi } \equiv  (B_{\ast })^{\ast }$ and $k(\cdot ,x^{\prime
})\equiv \delta _{x^{\prime }}.$ Indeed,%
\begin{equation*}
B_{\Phi }=\overline{\mathrm{Lin}\{k(\cdot ,x^{\prime }):x^{\prime }\in
X^{\prime }\}}^{w\ast }\equiv \overline{\mathrm{Lin}\{\delta _{x^{\prime
}}:x^{\prime }\in X^{\prime }\}}^{w\ast }=(B_{\ast })^{\ast }.
\end{equation*}
It is therefore natural to ask under what conditions a solution $f_{\mu }\in $ $%
B_{\Phi }$ of the MNI-problem (\ref{MNI}) is of the form 
$$f^{\mathrm{sol}}(\cdot
)= \sum_{i=1}^{k}\beta _{i}k(\cdot ,x_{i}^{\prime }).$$
That is, to determine the conditions under which a solution of (\ref{MNI}), which is known to exist in  the space \break$B_{\Phi }=\overline{\mathrm{Lin}\{k(\cdot ,x^{\prime
}):x^{\prime }\in X^{\prime }\}}^{w\ast }$, actually belongs to $%
\mathrm{Lin}\{k(\cdot ,x^{\prime }):x^{\prime }\in X^{\prime }\}.$ In the following result, we characterize when this occurs.

For simplicity, throughout this section, we assume that the base field is $\mathbb{K}=\mathbb{R}$. (The extension of the following results to the complex case is straightforward).  

\begin{theorem}[Representer theorem in featured RKBSs]

\label{finM}Let $B_{\Phi }\subseteq \mathcal{F}(X,\mathbb{R})$ be a
special featured RKBS and \break $k:X\times X^{\prime }\rightarrow \mathbb{R}$ the corresponding kernel function. Given a dataset $\mathcal{D}_{%
m}:=\{(x_{i},y_{i})\}_{i=1}^{m}\,\ $consider the
associated MNI\ problem  
\begin{equation}
\label{MNI-final}
\arg \min \{\left\Vert f_{\mu }\right\Vert :f_{\mu }\in \mathcal{M}_{y}\}.
\end{equation}%
This problem has a solution of the type $$f^{\mathrm{sol}}(\cdot) := \sum_{i=1}^{k} \beta_i\, k(\cdot, x_i'),$$
for some  $ x_1', \dots, x_k' \in X'$, if and only if there exist
$\alpha_1, \dots, \alpha_{m} \in \mathbb{R}$ with

\begin{equation}
    \label{eq5}
    \left\|
        \sum_{i=1}^{m} \alpha_i\, k(x_i, \cdot)
    \right\| \leq 1
\end{equation}

and $\beta_1, \dots, \beta_k \in  \mathbb{R}$ satisfying that
\begin{equation}
\label{eq6}
    (\alpha_1 \cdots \alpha_{m})
\begin{pmatrix}
k(x_1, x_1') & \cdots & k(x_1, x_k') \\
\vdots & \ddots & \vdots \\
k(x_{m}, x_1') & \cdots & k(x_{m}, x_k')
\end{pmatrix}
\begin{pmatrix}
\beta_1 \\ \vdots \\ \beta_k
\end{pmatrix}
= \| f^{\mathrm{sol}}\|,
\end{equation}
and 
\begin{equation}
\label{eq7}
    \begin{pmatrix}
    k(x_1, x_1') & \cdots & k(x_1, x_k') \\
    \vdots & \ddots & \vdots \\
    k(x_{m}, x_1') & \cdots & k(x_{m}, x_k')
    \end{pmatrix}
    \begin{pmatrix}
    \beta_1 \\ \vdots \\ \beta_k
    \end{pmatrix}
    =
    \begin{pmatrix}
    y_1 \\ \vdots \\ y_m
    \end{pmatrix}.
    \end{equation}

Moreover, $f^{\mathrm{sol}}(\cdot) := \sum_{i=1}^{k} \beta_i\, k(\cdot, x_i')$ attains its norm at $\sum_{i=1}^{m} \alpha_i k(x_i, \cdot)$.

\end{theorem}
\begin{proof}
By Theorem \ref{special}, we have $k(\cdot, x_i')=f_{\delta_{x_i'}}$, as $B_{\Phi}$ is a special featured RKBS. Therefore, 
$$ f^{\mathrm{sol}} (\cdot ):= \sum_{i=1}^{k} \beta_{i}k(\cdot,{x_i'})=\sum_{i=1}^{k} \beta_{i}  f_{\delta_{x_i'}}= f_{\sum_{i=1}^{k} \beta _{i} \delta_{x_i'} }.
$$
If $\ f^{\mathrm{sol}}(\cdot ):=\sum_{i=1}^{k}\beta _{i}k(\cdot
,x_{i}^{\prime })$ is a solution of (\ref{MNI-final}), then  $\sum_{i=1}^{k}\beta
_{i}k(x_{k},x_{i}^{\prime })=y_{k},$ as $\ f^{\mathrm{sol}%
}(x_{j})=y_{j},$ for every $j=1\cdots, m.$ Hence,   (\ref{eq7}) holds.
Moreover, by Theorem  \ref{mnip}, the functional $\sum_{i=1}^{k} \beta _{i} \delta_{x_i'} \in W^*$ is a Hahn-Banach extension the data map  $$\mu^{%
\mathrm{data}}:W^{\mathrm{data}}:=\mathrm{Lin}\{k(x_{1},\cdot ),\dots ,k(x_{m},\cdot
)\}\rightarrow \mathbb{R},$$ defined in (\ref{datam}). Consequently, there
exists $v_{0}:=\sum_{i=1}^{m}\alpha _{i}k(x_{i},\cdot)\in W^{\mathrm{data}}$, such that $\left\Vert v_{0} \right\Vert \leq 1$ and 
$$
\sum_{i=1}^{k} \beta _{i} \delta_{x_i'}(v_0) = \left\lvert \sum_{i=1}^{k} \beta _{i} \delta_{x_i'}(v_0) \right\rvert =\left\Vert \sum_{i=1}^{k} \beta _{i} \delta_{x_i'} \right\Vert = \left\Vert f^{\mathrm{sol}}\right\Vert.
$$
Since $\sum_{i=1}^{k} \beta _{i} \delta_{x_i'}(v_0)=\sum_{i=1}^{k} \beta _{i} \delta_{x_i'}(\sum_{i=1}^{m}\alpha _{i}k(x_{i},\cdot))= \sum_{i,j} \alpha _{i}  \beta _{j} k(x_{i},x_j'),$ we conclude that (\ref{eq6}) is also satisfied.

Conversely, by (\ref{eq7}) we have that $\sum_{i=1}^{k} \beta _{i} \delta_{x_i'}$ is an extension of $\mu^{\mathrm{data}}$. At the same time,  by 
(\ref{eq5}) and (\ref{eq6}) we deduce that  $%
v_{0}:=\sum_{i=1}^{m}\alpha _{i}k(x_{i},\cdot )\in W^{%
\mathrm{data}}$ is such that $\left\Vert v_{0} \right\Vert \leq 1$ and $ \sum_{i=1}^{k} \beta _{i} \delta_{x_i'}(v_0) = \left\Vert \sum_{i=1}^{k} \beta _{i} \delta_{x_i'} \right\Vert = \left\Vert f^{\mathrm{sol}}\right\Vert$, so that $\sum_{i=1}^{k} \beta _{i} \delta_{x_i'}$ is a
Hahn-Banach extension of $\mu^{\mathrm{data}}$. By Theorem  \ref{mnip}, we conclude that $f^{\mathrm{sol}}(\cdot )= f_{\sum_{i=1}^{k} \beta _{i} \delta_{x_i'}}$ is a solution of the MNI problem (\ref{MNI-final}), as desired. The rest is clear.
\end{proof}

\paragraph{Interpretation and implications.}
Theorem~\ref{finM} clarifies when Banach-space learning problems admit finite kernel representations, despite the lack of Hilbert structure. The existence of a suitable predual and compatibility with the dataset ensure that solutions of minimal-norm interpolation can be expressed as finite linear combinations of kernel evaluations. This result explains why special featured RKBSs support kernel-style learning algorithms and provides the structural basis for the vector-valued and neural-network constructions developed in subsequent sections.

These conditions are not artifacts of the proof technique, but reflect genuine geometric obstructions in general Banach spaces.

In order to find sufficient conditions for having a solution of the type $$
f^{\mathrm{sol}}(\cdot ):= \sum_{i=1}^{k}\beta _{i}k(\cdot
,x_{i}^{\prime }),$$ for some $x_{1}^{\prime },$ $\cdots $ $,x_{k}^{\prime
}\in X^{\prime },$ we introduce the next definition.

\begin{definition}
Let $B_{\Phi }\subseteq \mathcal{F}(X,\mathbb{R})$ be a featured RKBS and
let $k:X\times X^{\prime }\rightarrow \mathbb{R\,\ }$be the corresponding
kernel function. We say that a dataset $\mathcal{D}_{m%
}:=\{(x_{i},y_{i})\}_{i=1}^{m}\,\ $is\textbf{\ {nice}} if there
exist $x_{1}^{\prime },\cdots ,x_{m}^{\prime }\in X^{\prime }$
such that the matrix%
\begin{equation*}
(k(x_{i},x_{j}^{\prime }))_{ij}\in M_{m\times m}(%
\mathbb{R})
\end{equation*}%
is invertible.
\end{definition}

Note that nice datasets  are those $\mathcal{D}_{m%
}:=\{(x_{i},y_{i})\}_{i=1}^{m}$ for which there exists elements $x_{1}^{\prime
},\cdots ,x_{m}^{\prime }\in X^{\prime }$ such that the
functionals in $(W^{\mathrm{data}})^{\ast }=(\mathrm{Lin}\{k(x_{1},\cdot
),\cdots ,k(x_{m},\cdot )\})^{\ast }$ are uniquely determined by
its evaluations in $x_{1}^{\prime },\cdots ,x_{m}^{\prime }.$
More precisely, the interest of working with nice datasets is the following
one.

\begin{proposition}
\label{anterior}
Let $B_{\Phi }\subseteq \mathcal{F}(X,\mathbb{R})$ be a special featured RKBS with associated kernel function  $k:X\times X^{\prime }\rightarrow \mathbb{R}.$ Given a dataset $\mathcal{D}_{m}:=\{(x_{i},y_{i})\}_{i=1}^{m},\,$\, and arbitrary   $x_{1}^{\prime },\cdots ,x_{%
m}^{\prime }\in X^{\prime }$, we have that the matrix $$
(k(x_{i},x_{j}^{\prime }))_{ij}\in M_{m\times m}(
\mathbb{R})$$ is invertible if, and only if,  $\{\delta _{x_{1}^{\prime
}|W^\text{data}},\cdots ,\delta _{x_{m}^{\prime }|W^\text{data}}\}$ is a basis of 
\begin{equation*}
\ (W^{\mathrm{data}})^{\ast }=(\mathrm{Lin}\{k(x_{1},\cdot ),\cdots ,k(x_{%
m},\cdot )\})^{\ast }=(\mathrm{Lin}\{\Phi (x_{1} ),\cdots
,\Phi (x_{m})\})^{\ast }.
\end{equation*}
\end{proposition}

\begin{proof}
If $(k(x_{i},x_{j}^{\prime }))_{ij}\in M_{m\times m}(%
\mathbb{R})$ is invertible, then the vectors 
\begin{equation*}
\{(k(x_{1},x_{1}^{\prime }),\cdots ,k(x_{1},x_{m}^{\prime })),\cdots
,(k(x_{m},x_{1}^{\prime }),\cdots ,k(x_{m},x_{m}^{\prime }))\}
\end{equation*}%
are linearly independent. It follows that the functions $\{k(x_{1},\cdot
),\cdots ,k(x_{m},\cdot )\}$ also are linearly independent and consequently 
they determine a basis of $W^{\mathrm{data}}.$ Furthermore, if $\alpha
_{i}\in \mathbb{R}$, for $i= 1,\cdots ,m$, then we have that $\psi :=\alpha
_{1}\delta _{x_{1}^{\prime }|W^\text{data}}+\cdots +\alpha _{m}\delta _{x_{%
m}^{\prime }|W^\text{data}}=0$ if, and only if, $\psi (k(x_{k},\cdot ))=0,$ for $%
k=1,\cdots ,m.$ This means that $(k(x_{i},x_{j}^{\prime }))_{ij}%
\boldsymbol{\alpha }^T=0$, where $\boldsymbol{\alpha }=(\alpha _{1},\cdots,\alpha _{%
m})\in M_{1\times m}(\mathbb{R)}$. Since  $%
(k(x_{i},x_{j}^{\prime }))_{ij}$ is invertible it follows that  $\boldsymbol{\alpha }$ $=0$
and hence $\{\delta _{x_{1}^{\prime }|W^\text{data}},\cdots ,\delta _{x_{m%
}^{\prime }|W^\text{data}}\}$ are linear independent and therefore  determine a basis of $\
(W^{\mathrm{data}})^{\ast }.$

Conversely, if $\{\delta _{x_{1}^{\prime }|W^\text{data}},\cdots ,\delta _{x_{m%
}^{\prime }|W^\text{data}}\}$ is a basis of $(W^{\mathrm{data}})^{\ast }$ then we have that $\dim
((W^{data})^{\ast })=m$, and therefore, $\dim (W^{\mathrm{data}})=%
m$. Moreover, $$W^{\mathrm{data}}=\mathrm{Lin}%
\{k(x_{1},\cdot ),\cdots ,k(x_{m},\cdot )\}$$ so that 
$\{k(x_{1},\cdot ),\cdots ,k(x_{m},\cdot )\}$ is a basis. If $%
(k(x_{i},x_{j}^{\prime }))_{ij}$ is not invertible, then there exists a
non-zero $\boldsymbol{\alpha}\in M_{1\times m}(\mathbb{R})$ such that $(k(x_{i},x_{j}^{\prime }))\boldsymbol{\alpha }^{T}=0$. This means that $\sum_{i=1}^{m}\alpha _{i}\delta _{x_{i}^{\prime
}|W^\text{data}}(k(x_{j},\cdot ))=0,$ for every $j=1,\cdots ,m,$ so that $%
\sum_{i=1}^{m}\alpha _{i}\delta _{x_{i}^{\prime }|W^\text{data}}=0,$ a
contradiction.
\end{proof}

If the dataset $\mathcal{D}_{m}:=\{(x_{i},y_{i})\}_{i=1}^{i=
m}$ is nice, then the function  $\mu^{\mathrm{data}}\in (W^{\mathrm{data}
})^{\ast }$ given by $ \mu^{\mathrm{data}}(k(x_{k},\cdot ))=y_{k},$ for $k=1,\cdots ,
m,$ is such that $ \mu^{\mathrm{data}}:=\sum_{i=1}^{m
}\beta _{i}\delta _{x_{i}^{\prime }|W^\text{data}}$ for a unique $\boldsymbol{\beta }:=(\beta
_{1}$ $\cdots $ $\beta _{m})\in M_{1\times m}(\mathbb{R}).$ Because of this it makes sense to ask when
the extension of $\mu^{\mathrm{data}} $ given by 
\begin{equation*}
\widehat{\mu}^{\mathrm{data}}:=\sum_{i=1}^{m}\beta _{i}\delta _{x_{i}^{\prime }}:W:=%
\overline{\mathrm{Lin}\{k(x,\cdot ):x\in X\}}^{\left\Vert \cdot \right\Vert
}\rightarrow \mathbb{R}
\end{equation*}
is a Hahn-Banach extension of $\mu^{\mathrm{data}} $  and, consequently, whether $f_{\widehat{\mu}^{\mathrm{data}}} $ is a
solution of the MNI problem we are dealing with. Obviously, this is the case whenever
 $\left\Vert \mu^{\mathrm{data}}\right\Vert $ is the biggest possible norm of an extension $\widehat{\mu}^{\mathrm{data}}$,
as we show next.

\vspace{0.1cm}
\begin{corollary}
\label{nice1}Let $B_{\Phi }\subseteq \mathcal{F}(X,\mathbb{R})$ be a special
featured RKBS and \mbox{$k:X\times X^{\prime }\rightarrow \mathbb{R\,\ }$}be
the corresponding kernel function. Let $\mathcal{D}_{%
m}:=\{(x_{i},y_{i})\}_{i=1}^{m}\,$ be a nice data set and let $%
x_{1}^{\prime },\cdots ,x_{m}^{\prime }\in X^{\prime }$ be such
that the matrix \break $M:=(k(x_{i},x_{j}^{\prime }))_{ij}\in M_{m\times 
m}(\mathbb{R})$ is invertible. If $\boldsymbol{\beta }:=(\beta _{1} \dots\beta _{m}) \in  M_{1\times m}(\mathbb{R}) $ is such that $\boldsymbol{\beta }^{T}:=M^{-1}\boldsymbol{y}^{T},$
where $\boldsymbol{y=}(y_{1},\cdots ,y_{m}),$ then the set of
solutions of the problem%
\begin{equation}
    \arg \min \{\left\Vert f_{\mu }\right\Vert :f_{\mu }\in \mathcal{M}_{y}\}
    \label{otravez}
\end{equation}
is  $\mathcal{S}_{y}:=\{f_{\mu }\in B_{\Phi }: \mu \text{ is a
Hahn-Banach extension of }\mu^{\mathrm{data}}:=\sum_{i=1}^{%
m}\beta _{i}\delta_{x'_{i}|W^\text{data}}\}.$

Consequently, if $\left\Vert \mu^{\mathrm{data}}\right\Vert
=\sum_{i=1}^{m}\left\vert \beta _{i}\right\vert
\left\Vert \delta_{x'_{i}}\right\Vert $, then $f^{\mathrm{sol}%
}(\cdot ):=\sum_{i=1}^{m}\beta _{i}k(\cdot
,x_{i}^{\prime })$ is a solution of (\ref{otravez}).
\end{corollary}
\begin{proof}

Let $\mu^{\mathrm{data}}\in (W^{\mathrm{data}
})^{\ast }$ be the map given by $ \mu^{\mathrm{data}}(k(x_{k},\cdot ))=y_{k},$ for $k=1,\cdots ,m.$ From Proposition~\ref{anterior} it follows that $ \mu^{\mathrm{data}}:=\sum_{i=1}^{m}\beta _{i}\delta _{x_{i}^{\prime }|W^\text{data}}$ for a unique vector $\boldsymbol{\beta }:=(\beta
_{1}\dots\beta _{m})$ in $ M_{1\times m}(\mathbb{R})$. Obviously $\boldsymbol{\beta }$ is determined by the equality $\boldsymbol{\beta }^{T}:=M^{-1}\boldsymbol{y}^{T}.$

Since $k(\cdot ,x_{i}^{\prime })= f_{\delta_{x'_{i}}}$ by  Theorem \ref{special}, if $f^{\mathrm{sol}}(\cdot ):=\sum_{i=1}^{m}\beta
_{i}k(\cdot ,x_{i}^{\prime })$, then 
$$
\left\Vert f^{\mathrm{sol}} \right\Vert=\left\Vert f_{\sum_{i=1}^{m}\beta _{i} \delta_{x'_{i}}} \right\Vert =\left\Vert \sum_{i=1}^{m}\beta _{i} \delta_{x'_{i}}\right\Vert \leq\sum_{i=1}^{m}\left\vert \beta _{i}\right\vert
\left\Vert \delta_{x'_{i}}\right\Vert =\left\Vert \mu^{\mathrm{data}}\right\Vert,
$$ 
so that $\left\Vert f^{\mathrm{sol}} \right\Vert=\left\Vert \mu^{\mathrm{data}}\right\Vert$. This proves that $f^{\mathrm{sol}}$ is a solution of (\ref{otravez}) by Theorem \ref{mnip}.
\end{proof}

The above result is helpful for describing the algorithm that we will use in last
section of this paper. Indeed, this algorithm will be based in the next
corollary.

For any $\alpha\in \mathbb{R}$ define $\mathrm{sign}(\alpha):=+1$ if $\alpha >0,$ $\ 
\mathrm{sign}(\alpha ):=-1$ if $\alpha<0$ and $\mathrm{sign}(0):=0$. 

If $
\boldsymbol{\alpha }:=\boldsymbol{(}\alpha _{1},\cdots ,\alpha _{m})\in 
 M_{1\times m}(\mathbb{R})$ define 
\begin{equation*}
\mathrm{sign}(\boldsymbol{\alpha }):=\boldsymbol{(}\mathrm{sign}(\alpha _{1}),\cdots
,\mathrm{sign}(\alpha _{m}))\in \Delta^{m},
\end{equation*}
where $\Delta :=\{-1,0,1\}$ and $\Delta ^{m}:=\Delta \times 
\overset{m}{\cdots }\times \Delta .$

\vspace{0.1cm}
\begin{definition}
\label{admisible}
Let $B_{\Phi }\subseteq \mathcal{F}(X,\mathbb{R})$ be a featured
RKBS and let $\mathcal{D}_{%
m}:=\{(x_{i},y_{i})%
\}_{i=1}^{m}$ be a dataset. We say that sign vector $\boldsymbol{s}:=(s_{1},\cdots ,s_{%
m})\in \Delta^{m}$ is \textbf{admissible }if $$\left\Vert \sum_{i=1}^{{m%
}}s _{i}k(x_{i},\cdot )\right\Vert \leq 1.$$ 
\end{definition}

\begin{corollary}
\label{nice2}
Let $B_{\Phi}\subseteq \mathcal{F}(X,\mathbb{R})$ be a special
featured RKBS and \(\mbox{$k : X \times X' \to \mathbb{R}$}\) be
the corresponding kernel function. Given a \textit{nice} dataset $\mathcal{D}_{%
m}:=\{(x_{i},y_{i})\}_{i=1}^{m},$ let $%
x_{1}^{\prime },\cdots ,x_{m}^{\prime }\in X^{\prime }$ be such
that the matrix \break$M:=(k(x_{i},x_{j}^{\prime }))_{ij}\in M_{m%
\times m}(\mathbb{R})$ is diagonal, and satisfies that $$
\left\vert k(x_{i},x_{i}^{\prime })\right\vert =\left\Vert k(\cdot
,x_{i}^{\prime }))\right\Vert , \hspace{0.2cm}\text{for every} \hspace{0.2cm}i\in \{1,\dots ,m\}.$$ Let \(\mbox{$\boldsymbol{\beta } := (\beta_1, \dots, \beta_m)\in  M_{1\times m}(\mathbb{R})$}\)
be such that $%
\boldsymbol{\beta }^{T}:=M^{-1}\boldsymbol{y}^{T},$ where $\boldsymbol{y=}(y_{1},\cdots ,y_{%
m}).$ If $\mathrm{sign}(\boldsymbol{y })$ is admissible then  $\
f^{\mathrm{sol}}(\cdot ):=\sum_{i=1}^{m}\beta
_{i}k(\cdot ,x_{i}^{\prime })$ is a solution of $$\arg \min \{\left\Vert
f_{\mu }\right\Vert :f_{\mu }\in \mathcal{M}_{y}\}$$
and $\left\Vert f^{\mathrm{sol}}\right\Vert =\sum_{i=1}^{m%
}\left\vert \beta _{i}\right\vert \left\Vert k(\cdot ,x_{i}^{\prime
})\right\Vert .$
\end{corollary}

\begin{proof}
If  $\mathrm{sign}(\boldsymbol{y })= (s _{1},\cdots ,s _{%
m})\in \Delta^{m},$ then $\left\Vert
\sum_{i=1}^{m}s _{i}k(x_{i},\cdot )\right\Vert\leq
1,$ as $ \mathrm{sign}(\boldsymbol{y })$ is admissible. Moreover, $ f^{\mathrm{sol}}$ attains its norm in $\sum_{i=1}^{m}s _{i}k(x_{i},\cdot ),$ as 
\begin{equation*}
\boldsymbol{\alpha }M\boldsymbol{\beta }^{T}=\sum_{i=1}^{%
m} s_i\left\vert
k(x_{i},x_{i}^{\prime })\right\vert \beta _{i}  =\sum_{i=1}^{m%
}\left\vert \beta _{i}\right\vert \left\Vert k(\cdot ,x_{i}^{\prime
}))\right\Vert =\left\Vert f^{\mathrm{sol}}\right\Vert .
\end{equation*}%
Therefore, Theorem \ref{finM} applies with $\boldsymbol{\alpha }=\mathrm{sign}(\boldsymbol{y })= (s _{1},\cdots ,s _{%
m}),$ completing the proof.  
\end{proof}

\section{Vector Valued RKBS}
\label{sec:vector_valued}
\textcolor{blue}{
}Many modern machine learning problems involve vector-valued outputs, including multi-task learning, structured prediction, and multi-output neural networks. While vector-valued reproducing kernel Hilbert spaces provide a well-established framework for such problems, their Banach-space counterparts remain less explored. In this section, we extend the theory of featured RKBSs to vector-valued function spaces, thereby enabling the study of multi-output learning problems under non-Hilbertian norms.

We introduce vector-valued featured RKBSs and develop the associated kernel constructions, showing how they naturally decompose learning problems into coordinate-wise components while preserving a unified functional-analytic structure. This formulation provides the mathematical basis for interpreting multi-output models, including neural networks with multiple outputs, within a kernel-based Banach-space framework.


To this end, we consider Banach spaces $B\subseteq \mathcal{F}(X,\mathbb{K}%
^{m}),$ where 
\begin{equation*}
\mathcal{F}(X,\mathbb{K}^{m}):=\{f:X\rightarrow \mathbb{K}^{m}:f\text{ }\ 
\text{is a function}\}. 
\end{equation*}%
As usual, this means that the Banach space $B$ is a vector subspace of $\mathcal{F}(X,\mathbb{K}%
^{m})$ (so that $\mathbb{K}$ is the base field of $B$).

Denote by $\mathcal{L}(B,\mathbb{K}^{m})$ the mapping of all linear maps
from $B$ to $\mathbb{K}^{m}.$ According to \cite[Definition 2.1]{LIN2021101514}, we have the
following definition.

\vspace{0.1cm}
\begin{definition}
A \textbf{vector valued RKBS}$\,$\ is a Banach space $B\subseteq \mathcal{F}%
(X,\mathbb{K}^{m})$ such that
the mapping $\delta :X\to \mathcal{L}(B,\mathbb{K}^{m})$ given by $%
\delta _{x}(f)=f(x)$ is continuous for every $x\in X.$
\end{definition}

It is well known that the linear map  $\delta _{x}:B\to \mathbb{K}^{m}$ is
continuous if and only if there exists \thinspace $C_{x}\geq 0$ such that 
\begin{equation*}
\left\vert \pi _{i}\delta _{x}(f)\right\vert :=\left\vert \pi
_{i}f(x)\right\vert \leq C_{x}\left\Vert f\right\Vert ,\text{ for every }%
f\in B\text{ and every }i=1,\dots ,m.
\end{equation*}%
where $\pi _{i}:\mathbb{K}^{m}\to \mathbb{K}\,\ $\ is the canonical
proyection over the $i$-th coordinate. \ In other words, the continuty of $%
\delta _{x}$ is equivalent to the continuity of $\pi _{i}\delta _{x},$ for
every $i=1,\dots ,m.$ Consequently, we have the following result.

\vspace{0.1cm}
\begin{proposition}
Let $B\subseteq \mathcal{F}(X,\mathbb{K}^{m})\,\ $be a Banach space and
let 
\begin{equation*}
B_{i}:=\{\pi _{i}f:X\to \mathbb{K}:f\in B\}\subseteq \mathcal{F}(X,%
\mathbb{K}),\quad (i=1,\dots ,m).
\end{equation*}%
Then, $B$ is a vector valued RKBS if, and only if, $B_{i}$ is an RKBS,$\,\ $\
for every $i=1,\dots ,m$.
\end{proposition}

Let $B\subseteq \mathcal{F}(X,\mathbb{K}^{m})\,\ $be a vector valued RKBS. This means that $\pi _{i}\delta _{x}\in B^{\ast },$ for
every $i=1,\dots ,m.$ A possible reason why this occurs is that $\pi _{i}\delta _{x}\in J(w_x^i)$ for some $w_x^i\in B_\ast$,  for every $i=1,\dots,m.$ As we have shown already (see Proposition~\ref{propo}), $w_x^i$ has to be unique, which gives rise to a function $\Phi _{i}:X\rightarrow
B_{\ast }$ such that $J_{\Phi _{i}(x)}\equiv\pi _{i}\delta _{x}$. This motivates the following definition.

\vspace{0.1cm}
\begin{definition}
Let $B\subseteq \mathcal{F}(X,\mathbb{K}^{m})$ a Banach space with a predual 
$B_{\ast }\subseteq \mathcal{F}(X^{\prime },\mathbb{K}),$ such that $(B_{\ast })^{\ast }\equiv B$.  A \textbf{vector valued feature map on} $(B,B_{\ast })$ is a map 
$$\Phi :=(\Phi
_{1},\dots ,\Phi _{m}),$$
where  $\Phi _{i}:X\rightarrow B_{\ast }$, \ for
every $i=1,\dots ,m$, satisfying that 
\begin{equation*}
B_{\ast } =\overline{\mathrm{Lin}\{\Phi _{i}(x):x\in X,\text{ }i=1,\dots
,m\}}^{\|\cdot\|}.
\end{equation*}%
The \textbf{vector valued featured RKBS associated}  with $\Phi =(\Phi_1, \cdots \Phi_m)$ is
\begin{equation}
\label{vect}
B_{\Phi }:=\{f_{\mu }:\mu \in (B_{\ast })^{\ast }\}\subseteq \mathcal{F}(X,%
\mathbb{K}^{m}),
\end{equation}%
where, for every $x\in X$,%
\begin{equation*}
f_{\mu }(x):=\mu (\Phi (x))=(\mu (\Phi _{1}(x)),\cdot \cdot \cdot ,\mu (\Phi
_{m}(x)))\in \mathbb{K}^{m},
\end{equation*}%
is a Banach space provided with the norm $\left\Vert f_{\mu
}\right\Vert :=\left\Vert \mu \right\Vert .$ In fact, 
 $B_{\Phi }\equiv (B_{\ast })^{\ast }\equiv B.$ 
\end{definition}

\begin{remark}
According to Definition \ref{feature}, in the scalar-valued setting a
feature map is map $\Phi :X\rightarrow W\subseteq \mathcal{F}(X^{\prime },%
\mathbb{K}),$ where $X$ and $X^{\prime }$ are non-empty sets and $W$ is a
Banach space. Moreover, the RKBS associated with \textbf{\ }$\Phi $ is the space $%
B_{\Phi }$ described in Example \ref{pro}, and $B_{\ast }:=\overline{\mathrm{%
Lin}\{\Phi (x):x\in X\}}^{\Vert \cdot \Vert }$ is a predual of $B_{\Phi }.$
Therefore, in view of Definition \ref{vect}, more precisely we can say that $\Phi $ is a feature map on $(B_{\Phi },B_{\ast
}).$
\end{remark}

Every vector valued feature map has associated a kernel function, as we show
next.

\vspace{0.1cm}
\begin{definition}
\label{vvkf}
Let $B\subseteq \mathcal{F}(X,\mathbb{K}^{m})$ be a Banach space with a predual 
$B_{\ast }\subseteq \mathcal{F}(X^{\prime },\mathbb{K}).$ A \textbf{%
vector valued kernel function on } $(B,B_{\ast })$ is a map $$k :=(k_{1},\dots
,k_{m}):X\times X^{\prime }\rightarrow \mathbb{K}^{m}$$ such that 
\begin{equation*}
B_{\ast } =\overline{\mathrm{Lin}\{k_{i}(x,\cdot ):x\in X,\text{ }i=1,\dots
,m\}}^{\|\cdot\|}
\end{equation*}%
and $k_{i}(x,x')=\delta _{x^{\prime
}}(k_{i}(x,\cdot )),$ for every $x\in X,$  $x^{\prime }\in X^{\prime },\,\ $%
and $i=1,\dots ,m.$
\end{definition}

The proof of the next result is straightforward. 

\vspace{0.1cm}
\begin{proposition}
\label{phik}
Let $B\subseteq \mathcal{F}(X,\mathbb{K}^{m})$ be a Banach space with a predual 
\mbox{$B_{\ast }\subseteq \mathcal{F}(X^{\prime },\mathbb{K})$}. Every vector valued feature map $\Phi :=(\Phi _{1},\dots ,\Phi _{m})$ on $(B,B_{\ast })$
has associated a unique vector valued kernel function on $(B,B_{\ast }),$
say $k=(k_{1},\dots ,k_{m}):X\times X^{\prime }\rightarrow \mathbb{K}^{m},$
and \textit{vice versa}, satisfying that  $k_{i}(x,\cdot ):=\Phi _{i}(x),$ for every $%
i=1,\dots ,m$.
\end{proposition}

Let $k$ be a vector valued kernel function on $(B,B_{\ast }).$ The previous
result allows us to express the space $B_{\Phi }$ defined in (\ref{vect}) in
terms of the associated  kernel. That is, to consider the space
\begin{equation}
B_{\Phi }^{k}:=\{\widehat{f}_{\mu }:\mu \in (B_{\ast })^{\ast }\}\subseteq 
\mathcal{F}(X,\mathbb{K}^{m}),
\end{equation}%
where, for every $x\in X$,
\begin{equation*}
\widehat{f}_{\mu }(x):=\mu (k(x,\cdot ))=(\mu (k_{1}(x)),\cdot \cdot \cdot
,\mu (k_{m}(x,\cdot )))\in \mathbb{K}^{m}.
\end{equation*}
We say that $B_{\Phi }^{k}$ is the \ \textbf{RKBS associated with the vector valued
kernel function} $k.$ Clearly, $B_{\Phi }^{k}$ is a Banach space provided
with the norm $\left\Vert \widehat{f_{\mu }}\right\Vert :=\left\Vert \mu
\right\Vert .$ In fact, $B_{\Phi }=B_{\Phi }^{k}\equiv (B_{\ast })^{\ast
}\equiv B$, under the identification $$f_{\mu }=\widehat{f}_{\mu }\equiv \mu
\equiv \varphi (\mu ),$$ whenever $(B_{\ast })^{\ast }\overset{\varphi }{%
\equiv }B.$

The following result will be useful for characterizing vector-valued featured RKBS. The definition of such spaces will be introduced afterwards.

\begin{theorem}
\label{t224} Let $B\subseteq \mathcal{F}(X,\mathbb{K}^{m})\,\ $be a Banach
space and let $B_{\ast }\subseteq \mathcal{F}(X^{\prime },\mathbb{K}^{m})$
be a predual of $B,$ with $B\overset{{\varphi }}{\equiv }(B_{\ast })^{\ast }$.
Then, the following assertions are equivalent.

$\mathrm{(i)}$\ For every $x\in X$, \ and every  $i=1,\dots ,m,$ there exists 
$w_{x}^{i}\in B_{\ast }$ such that  $\pi _{i}\delta _{x}\overset{{\varphi }^*}\equiv J_{w_{x}^{i}}\in J(B_{\ast })$.

$\mathrm{(ii)}$\ For every $i=1,\dots ,m$, there exist a  map  $\Phi_{i}:X\rightarrow B_{\ast }$  such that $J_{\Phi _{i}(x)} \overset{{\varphi }^*}\equiv \pi
_{i}\delta _{x},$ for every $x\in X.$

$\mathrm{(iii)}$ There exists a vector valued feature map $\Phi :=(\Phi
_{1},\dots ,\Phi _{m})$ on $(B, B_{\ast })$, such that $B=B_{\Phi }.$

$\mathrm{(iv)}$\ There exists $k=(k_{1},...,k_{m}):X\times X^{\prime
}\rightarrow \mathbb{K}^{m}$, such that $k_{i}(x,\cdot )\in B_{\ast }$ and $\pi _{i}\delta _{x}  \overset{{\varphi }^*}\equiv J_{k_{i}(x,\cdot )},$ for every $x\in X$ and  $i=1,...,m$.

$\mathrm{(v)}$ There exists a \ vector valued kernel function on $(B,B_{\ast })$%
, say $k:X\times X^{\prime }\rightarrow \mathbb{K}^{m}$, such that $%
B=B_{\Phi }^{k}.$

If the above assertions are fulfilled then $\Phi $ and $k$ are unique.
\end{theorem}

\begin{proof}
    $\mathrm{(i)\iff (ii)} $ This assertion follows directly from Proposition~\ref{propo}.

    $\mathrm{(ii)\Longrightarrow (iii)}. $ Let $\mu\in (B_\ast)^\ast$ be such that $\mu (\Phi{_i}(x))=0$, for every $i=1,...,m,$ and each $x \in X$. Thus,  $J_{\Phi_{i}(x)}(\mu)=0$. Let  $f\in B$ be such that $\varphi(f)=\mu$ and let $x \in X$.  Since $J_{\Phi_i(x)} \overset{{\varphi }^*}\equiv \pi_{i} \delta_x$, we have $$J_{\Phi_i(x)}(\mu) = J_{\Phi_i(x)}(\varphi (f))= \pi_{i} \delta_x (f)=0, \quad \text{for all } i=1, \cdots m,$$ so that $f(x)=0$. It follows that  $f=0$ and hence $\mu=0$. Thus, by  \cite[Corollary 6.14]{conway1994course},  $$B_{\ast }:=\overline{\mathrm{Lin}\{\Phi _{i}(x):x\in X,\text{ }i=1,\dots
,m \}}^{\|\cdot\|}.$$ This proves that $\Phi$ is a vector valued feature map on $(B,B_\ast)$. That $B=B_\Phi$ follows from the fact that, \break$f(x)=f_{\varphi(f)}(x)$, for every $x \in X$ and every $f\in B$. Indeed, 
\begin{align*}
f_{\varphi(f)}(x) 
  &= ([\varphi(f)](\Phi_1(x)), \dots, [\varphi(f)](\Phi_m(x))) 
   = (J_{\Phi_1(x)}(\varphi(f)), \dots, J_{\Phi_m(x)}(\varphi(f))) \\
  &= (\pi_1 \delta_x(f), \dots, \pi_m \delta_x(f)) 
   = (f_1(x), \dots, f_m(x)) = f(x).
\end{align*}

 $\mathrm{(iii)\Longrightarrow (i)}.$ Since $B=B_\Phi$ as $f=f_{\varphi(f)}$, for every $f\in B$, we have that $$ \pi_i\delta_x(f)=\pi_i\delta_x(f_{\varphi(f)})=J_{\Phi_i(x)}(\varphi(f)).$$ Thus, $\pi_i\delta_x=J_{\Phi_i(x)}\varphi$, and hence $\pi_i\delta_x\equiv J_{\Phi_i(x)}$, for every $x\in X $ and $i=1,...,m$. 

 The equivalences  $\mathrm{(ii)\Longrightarrow (iv)}$ and  $\mathrm{(iii)\Longrightarrow (v)}$ follows from Proposition~\ref{phik}.

 The uniqueness of $\Phi$ and $k$ is clear using the same argument given in Theorem~\ref{th216}.
    
\end{proof}

\begin{definition}
Let $B\subseteq \mathcal{F}(X,\mathbb{K}^{m})\,\ $be a Banach space, with predual $B_\ast\subseteq\mathcal{F}(X',\mathbb{K})$. We say
that $B$ is a \textbf{vector valued featured RKBS} if there exists a vector valued feature map $\Phi :=(\Phi
_{1},\dots ,\Phi _{m})$ on $(B,B_{\ast })$, such that $B=B_{\Phi }$.
\end{definition}

Obviously, each one of the equivalent assertions stated in Theorem~\ref{t224} provides a characterization of the notion of the vector valued featured RKBS.

\vspace{0.1cm}
\begin{theorem}
\label{weak}
    Let $B\subseteq \mathcal{F}(X,\mathbb{K}^{m})$ be a vector valued featured RKBS, $B_{\ast }\subseteq \mathcal{F}(X^{\prime },\mathbb{K})$ a predual of $B$  and  \break$\Phi :=(\Phi_{1},\dots ,\Phi _{m})$ a vector valued feature map on $(B,B_{\ast })$ such that $B=B_{\Phi }$. 
    If $B_\ast$ is an RKBS, then $k(\cdot,x')\in B$, $\delta_x'\in (B_\ast)^\ast$ and $f_{\delta_{x'}}=k(\cdot,x')$. Moreover, 
    $$ \overline{\mathrm{Lin}\{k(\cdot,x'):x'\in X'\}}^{w^*}=B.$$
\end{theorem}

\begin{proof}
    If $B_\ast$ is an RKBS, then $\delta_{x'}\in (B_\ast)^\ast$, for every $x'\in X'$, and hence $f_{\delta_{x'}}\in B_\Phi=B$. If $x\in X$, then $$f_{\delta_{x'}}(x):=\delta_{x'}(k(x,\cdot))=k(\cdot,x')(x).$$
    So that $k(\cdot,x')=f_{\delta_{x'}}$, for every $x'\in X'$. Let $V_0:=\{\delta_{x'}:x'\in X'\}\subseteq(B_\ast)^\ast$. Then, 
    $$^\bot V_0=\{g\in B_\ast: \delta_{x'}(g)=0, \hspace{0.1cm}x'\in X'\}=\{g\in B_\ast: g(x')=0, \hspace{0.1cm}x'\in X'\}=\{0\}.$$
    From the bipolar theorem \cite[1.8. Bipolar Theorem]{conway1994course}, we obtain that
    $$(^\bot V_0)^\bot=  \overline{\mathrm{Lin}\{V_0\}}^{w^*}=(B_\ast)^\ast.$$
    Since $\delta_{x'}\equiv f_{\delta_{x'}}=k(\cdot,x')$, as $(B_\ast)^\ast\equiv B=B_\Phi$, it follows that  $$ \overline{\mathrm{Lin}\{k(\cdot,x'):x'\in X'\}}^{w^*}=B,$$
    as desired.
\end{proof}
\begin{definition}
Let $B=B_{\Phi }\subseteq \mathcal{F}(X,\mathbb{K}^{m})$ be a vector valued
featured RKBS, associated with the vector valued
feature map $\Phi :=(\Phi _{1},\dots ,\Phi _{m})$ on  $(B,B_{\ast })$, where the Banach space $B_\ast\subseteq\mathcal{F}(X',\mathbb{K})$ is a predual of $B$. We
say that $B$ is a \textbf{special vector valued RKBS}\ if $B_{\ast }$ is a
RKBS. Consequently, $$\overline{\mathrm{Lin}\{k(\cdot ,x^{\prime }):x^{\prime
}\in X^{\prime }\}}^{w^{\ast }}=B=B_{\Phi }.$$
\end{definition}

\begin{remark}
If $W_{1},\dots ,W_{m}$ are Banach spaces then $W_{1}\times \dots \times
W_{m}$ becomes a Banach space provided with the norm 
\begin{equation*}
\left\Vert (w_{1},\dots ,w_{m})\right\Vert _{\infty }=\max \{\left\Vert
w_{1}\right\Vert ,\dots ,\left\Vert w_{m}\right\Vert \},\text{ \ \ \ }%
(w_{1},\dots ,w_{m})\in W\times \dots \times W.
\end{equation*}%
We denote this Banach space by $W_{1}\times \overset{\left\Vert \cdot
\right\Vert _{\infty }}{\dots }\times W_{m}.$ The topological dual of this
space, namely $(W_{1}\times \overset{\left\Vert \cdot \right\Vert _{\infty }}{\dots 
}\times W_{m})^{\ast },$ \ is isometrically isomorphic to $W_{1}^{\ast
}\times \overset{\left\Vert \cdot \right\Vert _{1}}{\dots }\times
W_{m}^{\ast },$ that is the space $W_{1}^{\ast }\times \dots \times
W_{m}^{\ast }$ \ provided with the norm 
\begin{equation*}
\left\Vert (\mu _{1},\dots ,\mu _{m})\right\Vert
_{1}=\sum_{i=1}^{m}\left\Vert \mu _{i}\right\Vert ,\text{ for every }%
(\mu _{1},\dots ,\mu _{m})\in W_{1}^{\ast }\times \dots \times W_{m}^{\ast
}.
\end{equation*}%
In other words,  $(W_{1}\times \overset{\left\Vert \cdot \right\Vert _{\infty }}{%
\dots }\times,W_{m})^{\ast }{\equiv} $ $W_{1}^{\ast }\times \overset{%
\left\Vert \cdot \right\Vert _{1}}{\dots }\times W_{m}^{\ast },$ where the
isometric isomorphism \break$\varphi :W_{1}^{\ast }\times \overset{%
\left\Vert \cdot \right\Vert _{1}}{\dots }\times W_{m}^{\ast }\rightarrow
(W_{1}\times \overset{\left\Vert \cdot \right\Vert _{\infty }}{\dots }%
\times W_{m})^{\ast }$ is given by 
\begin{equation*}
[ \varphi (\mu _{1},\dots ,\mu _{m})](w_{1},\dots,w_{m}):=\sum_{i=1}^{m}\mu_{i}(w_{i})\in \mathbb{K}.
\end{equation*}%
Similarly, $(W_{1}\times \overset{\left\Vert \cdot \right\Vert _{\infty }}{%
\dots }\times W_{m})^{\ast \ast }\equiv $ $W_{1}^{\ast \ast }\times \overset%
{\left\Vert \cdot \right\Vert _{\infty }}{\dots }\times W_{m}^{\ast \ast }$
 \cite[Theorem 1.10.13]{megginson1998introduction}. In this case, the isometric isomorphism is $\psi
:W_{1}^{\ast \ast }\times \overset{\left\Vert \cdot \right\Vert _{\infty }}{%
\dots }\times W_{m}^{\ast \ast }\rightarrow (W_{1}\times \overset{%
\left\Vert \cdot \right\Vert _{\infty }}{\dots }\times W_{m})^{\ast \ast }$
is given by%
\begin{equation*}
\lbrack \psi (w_{1}^{\ast \ast },\dots ,w_{m}^{\ast \ast })](\mu
):=\sum_{i=1}^{m}w_{i}^{\ast \ast }(\mu _{i})\in \mathbb{K}\text{,}
\end{equation*}%
for every $\mu \in (W_{1}\times \overset{\left\Vert \cdot \right\Vert
_{\infty }}{\dots }\times W_{m})^{\ast }$ and every $(w_{1}^{\ast \ast
},\dots ,w_{m}^{\ast \ast })\in W_{1}^{\ast \ast }\times \overset{%
\left\Vert \cdot \right\Vert _{\infty }}{\dots }\times W_{m}^{\ast \ast }$,
whenever $\mu \equiv (\mu _{1},\dots ,\mu _{m}).$

Finally note that the canonical embedding $$J:W_{1}\times \overset{\left\Vert
\cdot \right\Vert _{\infty }}{\dots }\times W_{m}\rightarrow (W_{1}\times 
\overset{\left\Vert \cdot \right\Vert _{\infty }}{\dots }\times
W_{m})^{\ast \ast }\equiv W_{1}^{\ast \ast }\times \overset{\left\Vert
\cdot \right\Vert _{\infty }}{\dots }\times W_{m}^{\ast \ast }$$ is just $$J\equiv (J^{1},\dots ,J^{m}),$$
where $J^{i}:W_{i}\rightarrow W_{i}^{\ast \ast }$ is the canonical embedding
associated with  $W_{i}.$ This means that, if $\mu \equiv (\mu _{1},\dots ,\mu
_{m})\,\ $then, 
\begin{equation*}
J_{(w_{1},\dots ,w_{m})}(\mu ):=(J_{w_{1}}^{1},\dots ,J_{w_{m}}^{m})(\mu
_{1},\dots ,\mu _{m})=\sum_{i=1}^{m}J_{w_{i}}^{i}(\mu
_{i})=\sum_{i=1}^{m}\mu _{i}(w_{i})\in \mathbb{K},
\end{equation*}%
for every $(w_{1},\dots ,w_{m})\in W_{1}\times \overset{\left\Vert \cdot
\right\Vert _{\infty }}{\dots }\times W_{m}.$ Consequently, $$J_{(0,\dots
w_{i},\dots ,0)}(\mu )=J_{w_{i}}^{i}(\mu _{i})=\mu _{i}(w_{i})\in \mathbb{K}.
$$
\end{remark}

Obtaining examples of vector-valued featured RKBS is a rather straightforward task, as we show in the next example.

\vspace{0.1cm}
\begin{example}
\label{prototipo}(\textbf{Prototype of a vector valued featured RKBS}). Let $%
X$ be a nonempty set, $W_{i}\subseteq \mathcal{F}(X',\mathbb{K})\,$\ a Banach
space and $\Phi _{i}^0:X\rightarrow W_{i}$ a feature map, for every \ $%
i=1,\dots ,m.$ It is not restrictive to assume that%
\begin{equation*}
W_{i}:=\overline{\mathrm{Lin}\{\Phi _{i}^0(x):x\in X\}}^{\|\cdot\|}.
\end{equation*}%
We now define 
\begin{equation*}
B_{\Phi ^0}:=\{f_{\mu }(\cdot ):\mu \equiv (\mu _{1},\dots ,\mu _{m})\in
W_{1}^{\ast }\times \overset{\left\Vert \cdot \right\Vert _{1}}{\dots }%
\times W_{m}^{\ast }\}\subseteq \mathcal{F}(X,\mathbb{K}^{m}),
\end{equation*}%
where, for every $x\in X,$ 
\begin{equation*}
f_{\mu }(x) :=\mu (\Phi^0 (x))=(\mu _{1}(\Phi _{1}^0(x)),\dots ,\mu _{m}(\Phi
_{m}^0(x))),
\end{equation*}
provided with the norm%
\begin{equation*}
\|f_{\mu }\| :=\sum_{i=1}^{m}\left\Vert \mu
_{i}+W_{i}^{\bot }\right\Vert =\sum_{i=1}^{m}\| \mu
_{i|W_{i}}\| .
\end{equation*}%
Then, $B_{\Phi^0 }$ is a Banach space and $B_{\ast }:=W_{1}\times \overset{%
\left\Vert \cdot \right\Vert _{\infty }}{\dots }\times W_{m}$ is a predual
of $B_{\Phi^0 }.$ In fact, considering the featured RKBS 
\begin{equation*}
B_{\Phi_i^0}:=\{\pi _{i}f_{\mu }(\cdot ) :=\mu _{i}(\Phi _{i}^0(\cdot )):\mu
_{i}\in W_{i}^{\ast }\}\subseteq \mathcal{F}(X,\mathbb{K})
\end{equation*}%
endowed with the norm $\left\Vert \pi _{i}f_{\mu }\right\Vert =\left\Vert
\mu _{i}\right\Vert $ we obtain that 
\begin{equation*}
B_{\Phi ^0}\equiv B_{\Phi _{1}^0}\times \overset{\left\Vert \cdot \right\Vert
_{1}}{\dots }\times B_{\Phi _{m}^0}.
\end{equation*}%
Moreover $B_{\Phi^0 }$ is a vector valued featured RKBS. Indeed, the map $%
\Phi _{i}^0:X\rightarrow W_{i}$ can be identify with the map 
\begin{equation*}
{\Phi _{i}}:X\rightarrow B_{\ast }:=W_{1}\times \overset{\left\Vert
\cdot \right\Vert _{\infty }}{\dots }\times W_{m}
\end{equation*}
given by 
\begin{equation*}
{\Phi }_{i}(x):=(0,\dots ,\Phi^0 _{i}(x),\dots 0)\in B_{\ast
}:=W_{1}\times \overset{\left\Vert \cdot \right\Vert _{\infty }}{\dots }%
\times W_{m}.
\end{equation*}%
Thus, $$J{\Phi }_{i}(x)=(0,\dots ,J ^{i}\Phi^0 _{i}(x),\dots 0)\in
(B_{\ast })^{\ast \ast }\equiv W_{1}^{\ast \ast }\times \overset{\left\Vert
\cdot \right\Vert _{\infty }}{\dots }\times W_{m}^{\ast \ast }.$$
By the above remark, if $\mu =(\mu _{1},\dots ,\mu _{m})\in W_{1}^{\ast
}\times \overset{\left\Vert \cdot \right\Vert _{1}}{\dots }\times
W_{m}^{\ast }$ we have that 
\begin{equation*}
J_{{\Phi }_{i}(x)}(\mu )=J_{\Phi^0 _{i}(x)}(\mu _{i})=\mu _{i}(\Phi^0
_{i}(x))=\pi _{i}\delta _{x}(f_{\mu })\quad (f_{\mu }\in B_{\Phi^0 }),
\end{equation*}%
so that $J_{{\Phi }_{i}(x)}\equiv \pi _{i}\delta _{x},$ for every $%
x\in X$, and every $i=1,\dots ,m.$ Consequently, $B_{\Phi ^0}=B_\Phi$, which means that $B_{\Phi ^0}$ is a vector
valued featured RKBS. Note that 
\begin{equation*}
B_{\Phi }\equiv (W_{1}\times \overset{\left\Vert \cdot \right\Vert _{\infty }%
}{\dots }\times W_{m})^{\ast }\equiv W_{1}^{\ast }\times \overset{%
\left\Vert \cdot \right\Vert _{1}}{\dots }\times W_{m}^{\ast }.
\end{equation*}%
Therefore $B_{\Phi }^{\ast }\equiv $ $W_{1}^{\ast \ast }\times 
\overset{\left\Vert \cdot \right\Vert _{\infty }}{\dots }\times W_{m}^{\ast
\ast }$ and $B_{\ast }:=W_{1}\times \overset{\left\Vert \cdot \right\Vert
_{\infty }}{\dots }\times W_{m}$ is a predual of $B_{\Phi }.$

We conclude that, given $m$ Banach spaces $W_{i}\subseteq \mathcal{F}(X,%
\mathbb{K})\,$\ and associated features maps $\Phi _{i}^0:X\rightarrow W_{i},$ with $%
W_{i}:=\overline{\mathrm{Lin}\{\Phi^0 _{i}(x):x\in X\}},$ for every \ $%
i=1,\dots ,m,$ the space 
\begin{equation*}
B_{\Phi }\equiv B_{\Phi^0 _{1}}\times \overset{\left\Vert \cdot \right\Vert
_{1}}{\dots }\times B_{\Phi^0 _{m}}\subseteq \mathcal{F}(X,\mathbb{K}^{m})
\end{equation*}%
is a vector valued featured RKBS, which we refer to as the\textbf{\
prototype of a vector valued featured RKBS}.

Moreover, if the evaluation map $\delta _{x^{\prime }}:W_{i}\rightarrow 
\mathbb{K}$ is continuous for every $x^{\prime }\in X^{\prime }$, and $i=1 \dots ,m,$ then, by Theorem~\ref{weak}, the
function $k(\cdot ,x^{\prime })$ is such that$$k(\cdot,x')=(\delta_{x'}(\Phi_1^0(\cdot)),...,\delta_{x'}(\Phi_m^0(\cdot)) \in B_{\Phi^0}=B_\Phi,$$
and $$\overline{\mathrm{Lin}\{k(\cdot,x'):x'\in X'\}}^{w^*}=B_{\Phi^0}=B_\Phi.$$
\end{example}

\section{Neural Networks as Vector Valued RKBS}
\label{neuronas}

Neural networks are among the most expressive and widely used models in modern machine learning, yet their theoretical analysis often relies on either parameter-space arguments or infinite-width kernel limits. In this section, we adopt a complementary function-space perspective by showing that fixed-architecture neural networks can be naturally embedded into special vector-valued featured reproducing kernel Banach spaces.

From this viewpoint, training a neural network corresponds to solving a minimal-norm interpolation or regularization problem in an appropriate function space, where the norm is induced by the chosen parameter norm. This interpretation bridges kernel methods and neural networks without appealing to infinite-width limits, and it clarifies how architectural choices and parameter norms induce specific function-space geometries. We illustrate the framework with explicit examples, demonstrating how neural networks fit naturally into the RKBS formalism developed in the previous sections. Although our analysis focuses on fixed architectures, the framework accommodates arbitrary finite architectures and norms on the parameter space.

In what follows, we aim to apply the theoretical framework developed in the previous sections to the setting of deep neural networks. 
With this objective in mind, we introduce the notion of a deep neural network, following the presentation in \cite[Chapter 4]{librod}, 
but adapting the definitions to our notation for the sake of consistency and clarity.

Before presenting the definition of a deep neural network, we need to define its fundamental component: 
the neuron function.

\vspace{0.1cm}
\begin{definition}
A \textbf{neuron function} $f : \mathbb{R}^s \to \mathbb{R}$ is a mapping of the form
\begin{equation}
\label{neurona}
    f(\mathbf{x}) = \sigma(\omega \cdot \mathbf{x} + b), \hspace{1cm}  (\mathbf{x} \in \mathbb{R}^s)
\end{equation}
where $\sigma : \mathbb{R} \to \mathbb{R}$ is a continuous non-linear function called the $\textbf{activation function}$, 
$\omega \in \mathbb{R}^s$ is a $\textbf{weight}$ vector, and $b \in \mathbb{R}$ is a scalar,  called the $\textbf{bias}$. 
Here, $\omega \cdot \mathbf{x}$ denotes the inner product on $\mathbb{R}^s$.

\end{definition}

Since the output of the artificial neuron is a scalar, we can combine several neurons to create a vector-valued function called a \textbf{layer function}.

\vspace{0.1cm}
\begin{definition} 
\label{capan}
A \textbf{layer function} $g : \mathbb{R}^s \to \mathbb{R}^n$ is a mapping of the form
\[
    g(\mathbf{x}) = (f_1(\mathbf{x}), f_2(\mathbf{x}), \ldots, f_n(\mathbf{x})),
\]
where each $f_i : \mathbb{R}^s \to \mathbb{R}$ is an \textbf{artificial neuron} as in~\eqref{neurona}, 
with its own weight vector $w_i = (w_{i1}, \ldots, w_{is})$ and bias $b_i$, for each $i = 1, \ldots, n$.
\end{definition}

\begin{remark}
In general, each neuron $f_i$ may have its own activation function $\varphi_i$. 
However, in practice it is common to assume that all neurons within the same layer share the same activation function, 
which simplifies both the notation and the implementation.
\end{remark}

\vspace{0.1cm}
\begin{remark}
\label{l}
The first layer, the \textit{input} layer, is usually not considered a proper layer function. 
Thus, if a network is said to have $l$ layer functions, there are actually $l+1$ layers in total. 
In practice, the terms ``layer functions'' and ``layers'' are often used interchangeably, 
so when a network is said to have $k$ layers, the input layer is not counted, 
and the total number of layers is $k+1$, including both the hidden layers and the output layer.
\end{remark}

The Definition~\ref{capan} can be equivalently reformulated in matrix form as
\begin{equation}
    g(\mathbf{x}) = \bar{\sigma}(W \mathbf{x} + \boldsymbol{b}),
\end{equation}
where $W \in \mathbb{R}^{n \times s}$ is the \textbf{weight matrix}, 
$\boldsymbol{b} \in \mathbb{R}^n$ is the \textbf{bias vector}, and 
$\bar{\sigma} : \mathbb{R}^n \to \mathbb{R}^n$ is the \textbf{vector activation function} defined by
\begin{equation}
    \bar{\sigma}(z_1, \dots, z_n) = \big(\sigma (z_1), \dots, \sigma(z_n)\big),
\end{equation}
for a scalar activation function $\sigma$ as in Definition~\ref{neurona}. In this formulation, the $i$-th pre-activation value is given by
\begin{equation}
    z_i = w_{i1} x_1 + \dots + w_{is} x_s + b_i, \qquad i = 1, \ldots, n,
\end{equation}
so that the $i$-th component of $g(\mathbf{x})$ is $f_i(\mathbf{x}) = \sigma(z_i)$.

As stated in Remark~\ref{l}, the parameter $l$, which represents the number of layer functions, 
determines the depth of the network. 
If $l=1$ or $l=2$, the network has at most one hidden layer and is referred to as a shallow artificial neural network (ANN). 
In what follows, we focus mainly on the case $l \geq 3$, that is, networks with at least two hidden layers 
in addition to the input and output layers. Such architectures are commonly known as \textbf{deep neural networks} (DNN).

A deep neural network is fully characterized by its layer widths and by the collection of trainable parameters 
\[
   \{(W^{(k)}, b^{(k)}): k=1,\ldots,l\},
\]
where $W^{(k)} \in \mathbb{R}^{m_k \times m_{k-1}}$ and $\boldsymbol{b}^{(k)} \in \mathbb{R}^{m_k}$ denote, respectively, 
the weight matrix and bias vector of the $k$-th layer.

In this context, to represent all parameters in a unified way, we introduce the parameter space
\[
\Theta := \prod_{k=1}^l [\mathcal{M}_{m_k\times m_{k-1}}(\mathbb{R})\times \mathcal{M}_{m_k\times 1}(\mathbb{R})],
\]
so that each $\theta \in \Theta$ corresponds to a family $\{(W^{(k)}, \boldsymbol{b}^{(k)})\}_{k=1}^l$.
\vspace{0.1cm}
\begin{remark}
Note that $m_0 := s$ and $m_l := t$. 
\end{remark}

\vspace{0.1cm}
\begin{definition}
\label{anng}
    A \textbf{neural network}  is a function $f_\theta^\text{net}:\mathbb{R}^s \to \mathbb{R}^t$ of the form
    \begin{equation}
        f_\theta^\text{net}(\mathbf{x}) = \big(g^{(l)} \circ g^{(l-1)} \circ \dots \circ g^{(1)}\big)(\mathbf{x}), \quad l \geq 1,
    \end{equation}
    where each $g^{(i)}:\mathbb{R}^{m_{i-1}} \to \mathbb{R}^{m_i}$ is the layer function of the $i$-th layer (see~\eqref{capan}), 
    with its own weight matrix $W^{(i)}$ and bias vector $\boldsymbol{b}^{(i)}$. If $l\geq 3,$
then we say that $f_{\theta }^{\mathrm{net}}$ is a \textbf{deep neural
network.}
\label{redneu}
\end{definition}

Now, let us consider the class of all  neural networks with a fixed architecture, 
consisting of $l$ layers, $s$ neurons in the input layer, and $t$ neurons in the output layer. 
We denote this class by
\[
    \big\{ f_\theta^{\text{net}} : \mathbb{R}^s \to \mathbb{R}^t \;\big|\; \theta \in \Theta \big\} 
   \;\subseteq\; \mathcal{F}(\mathbb{R}^s,\mathbb{R}^t),
\]
where $\Theta$ is the corresponding parameter space.

Let $W_i \subseteq \mathcal{F}(\Theta,\mathbb{R})$ be a Banach space. 
For $\mathbf{x}\in\mathbb{R}^s$ and $\theta\in\Theta$, we define 
\begin{equation}
    \label{kernel}
    k:\mathbb{R}^s\times \Theta\to \mathbb{R}^t \hspace{0.2cm}\text{as}\hspace{0.2cm}
    k(\mathbf{x},\theta) := f^{\mathrm{net}}_\theta(\mathbf{x})\,\xi(\theta).
\end{equation}

Writing $k=(k_1,\ldots,k_t)$, the function $\xi:\Theta \to \mathbb{R}^+$ is chosen so that, for every $\mathbf{x}\in\mathbb{R}^s$ 
and each $i=1,\ldots,t$, the function $k_i(\mathbf{x},\cdot)$ belongs to $W_i$. For instance, if we take $W_i=C_0(\Theta)$, the space of continuous functions vanishing at infinity, we have to 
choose $\xi$ such that, for every $\mathbf{x}\in\mathbb{R}^s$,
\[
\theta \in \Theta \mapsto f^{\mathrm{net}}_\theta(\mathbf{x})\,\xi(\theta)\ \in \mathbb{R}^t  \text{ is continuous } \text{ and } 
\lim_{\|\theta\|\to\infty}\big\|f^{\mathrm{net}}_\theta(\mathbf{x})\,\xi(\theta)\big\|=0,
\]
or equivalently,
\[
\lim_{\|\theta\|\to\infty} \pi_if^{\mathrm{net}}_\theta(\mathbf{x})\,\xi(\theta)=0
\quad\text{for all } i=1,\ldots,t.
\]

Let $W:=W_1\times\overset{||\cdot||_\infty}{\dots}\times W_t$ be the Banach space endowed with the product norm $\|(w_1,\ldots,w_t)\|_{\infty}:=\max_{i=1 \cdots t}\|w_i\|$ and set
\[
    B_* \;:=\; \overline{\mathrm{Lin}\big\{\,(k_1(\mathbf{x},\cdot),\ldots,k_t(\mathbf{x},\cdot)):\ \mathbf{x}\in\mathbb{R}^s\,\big\}}^{\lVert \cdot \rVert_\infty}
    \ \subseteq\ W,\hspace{0.2cm}
    \text{and} \hspace{0.2cm}
    B \equiv (B_{*})^{\,*}.
\]

With these definitions, $k$ becomes a vector valued kernel function on $(B,B_{*})$ 
in the sense of Definition~\ref{vvkf}. Equivalently, for each $i=1,\ldots,t$, let 
\[
    W_i:=\overline{\mathrm{Lin}\{\,k_i(\mathbf{x},\cdot):\ \mathbf{x}\in\mathbb{R}^s\,\}}^{\lVert \cdot \rVert_\infty}
    \quad\text{and let}\quad 
    B_{\Phi_i}^k \equiv (W_i)^{*}.
\]
Then,  a vector valued RKBS associated with the kernel
function defined in (\ref{kernel}) is the space $B_{\Phi }^{k}$ given by the product 
\[
    B_\Phi^k \;:=\; B_{\Phi_1}^k \times \overset{\left\Vert \cdot \right\Vert
_{1}}{\dots }\times B_{\Phi_t}^k,
\]
with predual $W_1 \times \overset{\left\Vert \cdot \right\Vert
_{\infty }}{\dots }\times W_t$.

The evaluation function $\delta_\theta:W_i\to\mathbb{R}$ is continuous, for every $\theta \in \Theta$. Indeed,  for every $\mathbf{x}\in\mathbb{R}^s$,
\[
\left\vert \delta_\theta \left( \sum_{j=1}^{k} \alpha_j k_i(\mathbf{x},\cdot) \right) \right\vert = \left\vert  \sum_{j=1}^{k} \alpha_j k_i(\mathbf{x},\theta)  \right\vert \leq 
\left\Vert  \sum_{j=1}^{k} \alpha_j k_i(\mathbf{x},\cdot)  \right\Vert_{\infty}.
\] Then, by Theorem~\ref{weak}, see also Example~\ref{prototipo}, $$\overline{\mathrm{Lin}\{k(\cdot,\theta):\theta\in \Theta\}}^{w^*}= \overline{\mathrm{Lin}\{f_\theta^\text{net}:\theta\in \Theta\}}^{w^*}=B_{\Phi^0}=B_\Phi.$$

Recent investigations \cite{bohn2018representertheoremdeepkernel,wang2025hypothesisspacesdeeplearning,heeringa2025deepnetworksreproducingkernel} have demonstrated that the training of a deep neural 
network with fixed architecture can be rigorously formulated as a 
minimum-norm problem in a reproducing kernel Banach space (RKBS), in accordance with the approach
adopted in this work.
In this setting, the kernel 
\[
    k : X \times X' \to \mathbb{R},
\]
is determined by the input domain $X=\mathbb{R}^s$ together with the parameter space $X'=\Theta$. 
Consequently,  training the neural network with fixed architecture and parameters $\theta\in\Theta$, with a training dataset \[
\mathcal{D}_m := \{(\mathbf{x}_j,\boldsymbol{y}_j)\}_{j=1}^{m} \subset \mathbb{R}^s \times \mathbb{R}^t,
\] is equivalent to solving the 
minimum-norm interpolation problem in the RKBS  $B_\Phi,$ given by 
\begin{equation}
    \label{problem}
    \operatorname*{arg\,min}\{ \|f_\mu\| \ : \ f_\mu\in \mathcal{M}_\textbf{y}\},
\end{equation}
where $\mathcal{M}_{\boldsymbol{y}}:=\{f_{\mu }\in B_{\Phi }:f_{\mu }(\mathbf{x}%
_{j})=\boldsymbol{y}_{j},$ $\ $with $\ j=1,\cdots ,m\}.$

This approach makes sense, since a neural network can naturally be
interpreted as an element of a special  vector-valued  RKBS. In this setting, $B_{\Phi }$
represent the smallest Banach space that contains every neural network \ $%
f_{\theta }^{\mathrm{net}},$ for every $\theta \in \Theta ,$ and is closed
with respect to pointwise convergence on $\mathbf{x},$ as $\overline{\mathrm{Lin}%
\{f_{\theta }^{\mathrm{net}}:\theta \in \Theta \}}^{w^{\ast }}=B_{\Phi }.$

Since $B_{\Phi }=B_{\Phi }^{k}\;:=\;B_{\Phi _{1}}^{k} \times \overset{\left\Vert \cdot \right\Vert
_{1}}{\dots }\times 
B_{\Phi _{t}}^{k}$,  for every $f_{\mu }\in B_{\Phi }=B_{\Phi }^{k}$ we have that 
\begin{equation*}
\Vert f_{\mu }\Vert :=\sum\limits_{i=1}^{t}\Vert \pi _{i}f_{\mu }\Vert ,
\end{equation*}%
and it is clear that $\Vert f_{\mu }\Vert $ is minimal whenever $\Vert \pi
_{i}f_{\mu }\Vert $ is mininal for each $i=1,\cdots ,t.$ Therefore, the
training problem (\ref{problem}) decomposes into $t$ scalar minimum-norm
interpolation problems. More precisely, letting 
\begin{equation*}
\mathcal{M}_{\boldsymbol{y}}^{i}:=\{\,\pi _{i}f_{\mu }:\mathbb{R}^{s}\rightarrow 
\mathbb{R}\ :\ g(\mathbf{x}_{j})=\pi_i(y_{j}),\ \text{\ with }j=1,\ldots ,m\,\},
\end{equation*}%
the solution of (\ref{problem}) is given by 
\begin{equation*}
\operatorname*{arg\,min}\Big\{\Vert \pi _{i}f_{\mu }\Vert \ :\ \pi _{i}f_{\mu }\in 
\mathcal{M}_{\boldsymbol{y}}^{i}\Big\},\quad \text{for every}\hspace{0.2cm}%
i=1,\ldots ,t.
\end{equation*}

As proved in Theorem \ref{mnip}, the solutions of these problems are derived  by
the Hahn-Banach extensions $\widehat{\mu}^{\mathrm{data}}_{i}$ of the corresponding data functions
${\mu_i}^{\mathrm{data}}$. Actually such solutions are the functions  $f_{\widehat{\mu}^{\mathrm{data}}_{i}}$, for every $i= 1 \cdots t$.

In summary, $B_{\Phi }\;:=\;B_{\Phi _{1}}  \times \overset{\left\Vert \cdot \right\Vert
_{1}}{\dots }\times B_{\Phi _{t}}$%
, and the training process is performed component by component. Beyond
Theorem \ref{mnip}, obviously, the Representer Theorem (Theorem \ref{finM})
naturally comes into play, along with all the results established around it. More
precisely, we can make use of Corollary \ref{nice1} and corollary \ref{nice2}. Under these last considerations, we next provide a way to proceed.

As previously stated, restricting our analysis to the case $t=1$ entails no
loss of generality. Then, proceed as follows:

(i) Get the maximal nice  dataset contained in the given dataset $$\mathcal{D}_m%
:=\{(\mathbf{x}_{j},\boldsymbol{y}_{j})\}_{j=1}^{m}\subset \mathbb{R}^{s}\times 
\mathbb{R}.$$

(ii)\ Find $\theta _{1},\cdots ,\theta _{m}$ satisfying that $M=(k(\mathbf{x}%
_{i},\theta _{j}))_{ij}$ is diagonal with $\left| k(\mathbf{x}_i,\theta_i) \right|
= \left\| k(\cdot,\theta_i) \right\|.$

(iii)\ Determine the corresponding set of  admissible sign vector in $%
\mathbb{R}^{m}$.

(iv) If $\boldsymbol{\beta }=(\beta _{1},\cdots ,\beta _{m})$ and $\mathbf{y}=(y_1,\dots,y_m)$ are such that $M\boldsymbol{\beta }^{T}=\boldsymbol{y}^{T}$ and $\mathrm{sign}(\boldsymbol{\beta })$ is an
admissible sign vector then $f_{\mu }(\cdot )=\sum_{i=1}^{m}\beta
_{i}k(\cdot ,\theta _{i})$ is a solution of (\ref{problem}).

(v) If $\mathrm{sign}(\mathbf{\beta })$ is not admissible, then we can deal
with the regularizations of $$f_{\mu }(\cdot )=\sum_{i=1}^{m}\beta _{i}k(\cdot
,\theta _{i}).$$ To do this, consider all  vectors of the type $\boldsymbol{
\beta }_{reg}=(\widetilde{\beta }_{1},\cdots ,\widetilde{\beta }_{m})$ whose
sign vector $\mathrm{sign}(\mathbf{\beta }_{reg}),$ is admissible. If $%
\boldsymbol{\beta }_{reg}=(\widetilde{\beta }_{1},\cdots ,\widetilde{\beta }_{m})
$ is such a vector, then we say that $\widetilde{f}_{\mu }(\cdot
)=\sum_{i=1}^{m}\widetilde{\beta }_{i}k(\cdot ,\theta _{i})$ is a
regularization of $f_{\mu }.$ Note that $\widetilde{f}_{\mu }(\cdot )$ is a
solution of the problem%
\begin{equation*}
\text{arg min}\{\Vert f_{\mu }\Vert \ :\ f_{\mu }\in \mathcal{M}_{\widetilde{{
\mathbf{y}}}}\},
\end{equation*}%
with $M\boldsymbol{\beta }_{reg}^{T}=\widetilde{\mathbf{y}}^{T}$. Then, choose
the best option between $f_{\mu }(\cdot )=\sum_{i=1}^{m}\beta _{i}k(\cdot
,\theta _{i})$ and all the regularizations $\widetilde{f}_{\mu }(\cdot
)=\sum_{i=1}^{m}\widetilde{\beta }_{i}k(\cdot ,\theta _{i}),$ according to
the defined function $Q_{\mathbf{y}}$ and parameter $\lambda _{0}$
considered in the regularization problem 
\begin{equation*}
\inf \left\{ Q_{\mathbf{y}}(\mathcal{L}(f_{\mu }))+\lambda _{0}\,\varphi
\!\left( \left\Vert f_{\mu }\right\Vert \right) :f_{\mu }\in B_{\Phi
}\right\} .
\end{equation*}

Recall that for neural networks $f_{\theta }^{\mathrm{net}}:%
\mathbb{R}^{s}\rightarrow \mathbb{R}^{t}$, the problem decomposes into $t$
scalar minimum-norm interpolation problem (see Section~\ref{sec:rkbs}) by
considering $\pi _{i}f_{\theta }^{\mathrm{net}}:\mathbb{R}^{s}\rightarrow 
\mathbb{R}$, for $i =1 \cdots , t.$

In the next section, we provide an example to illustrate this procedure for
learning a function from a dataset. We will also verify that, in
order to work with nice datasets, obtaining a
diagonal matrix $M=(k(\mathbf{x}_{i},\theta _{j}))_{ij}$ with $$k(\mathbf{x}%
_{i},\theta _{i})=\left\Vert k(\mathbf{\cdot },\theta _{i})\right\Vert,$$ is rather straightforward, in many cases.

\section{Illustrative Example: Neural Networks as Featured RKBSs}
\label{sec:example}

In this section, we present a concrete illustrative example that demonstrates how the theoretical framework developed in the previous sections can be instantiated in a simple setting. The purpose of this example is not to model a realistic large-scale learning problem, but rather to make explicit how featured reproducing kernel Banach spaces arise from fixed neural network architectures and how the associated learning formulation and representer results apply in practice.

For clarity, we consider a small neural network architecture with two input nodes, one hidden layer consisting of two neurons, and an output layer with three nodes. This architecture is sufficiently simple to allow all constructions to be written explicitly, while still capturing the essential features of the theory, including the induced feature map, the associated Banach norm on the space of realized functions, and the resulting kernel representation. Through this example, we illustrate how minimal-norm interpolation in the induced featured RKBS leads to finite representations consistent with the representer theorems established earlier.

The structure of such a net is $%
f_\theta^\text{net}:\mathbb{R}^{2}\rightarrow \mathbb{R}^{3}$. For every $\mathbf{x}=(x_{1},x_{2})\in \mathbb{R}^{2}$, we have that $$f_\theta^\text{net}(%
\boldsymbol{x)=(}f_\theta^{\text{net},1}(\mathbf{x}),f_\theta^{\text{net},2}(\mathbf{x}%
),f_\theta^{\text{net},3}(\mathbf{x}))\in \mathbb{R}^{3},$$ where 
\begin{equation*}
\mathbf{x} = 
\begin{pmatrix} x_1 \\ x_2 \end{pmatrix}
\longmapsto 
f_\theta^{\text{net}}(\mathbf{x}) =
\begin{pmatrix}
f_\theta^{\text{net},1}(\mathbf{x}) \\
f_\theta^{\text{net},2}(\mathbf{x}) \\
f_\theta^{\text{net},3}(\mathbf{x})
\end{pmatrix},
\end{equation*}

\begin{align*}
f_\theta^{\text{net},1}(\mathbf{x}) &= 
w_5 \sigma(w_1 x_1 + w_2 x_2 + b_1) 
+ w_6 \sigma(w_3 x_1 + w_4 x_2 + b_2) + b_3, \\[0.3em]
f_\theta^{\text{net},2}(\mathbf{x}) &= 
\widetilde{w}_5 \sigma(\widetilde{w}_1 x_1 + \widetilde{w}_2 x_2 + \widetilde{b}_1) 
+ \widetilde{w}_6 \sigma(\widetilde{w}_3 x_1 + \widetilde{w}_4 x_2 + \widetilde{b}_2) + \widetilde{b}_3, \\[0.3em]
f_\theta^{\text{net},3}(\mathbf{x}) &= 
\widehat{w}_5 \sigma(\widehat{w}_1 x_1 + \widehat{w}_2 x_2 + \widehat{b}_1) 
+ \widehat{w}_6 \sigma(\widehat{w}_3 x_1 + \widehat{w}_4 x_2 + \widehat{b}_2) + \widehat{b}_3.
\end{align*}

Here, $\sigma(t) = \max\{t,0\}$ denotes the ReLU activation function for every $t \in \mathbb{R}$. 
We will consider the training dataset $$\mathcal{D}_3=\{(\mathbf{x}_{1},%
\boldsymbol{y}_{1}),(\mathbf{x}_{1},\boldsymbol{y}_{2}),(\mathbf{x}_{3},\boldsymbol{y}%
_{3})\}$$ where the input data $\mathbf{x}_{1}=(1,-1),$ $\mathbf{x}_{2}=(-1,0),$ $\mathbf{x}%
_{3}=(0,1)$ are in $\mathbb{R}^2,$ and the output data $\boldsymbol{y}_{1}=(a_{1},b_{1},c_{1}),$ $\boldsymbol{y}_{2}=(%
\widetilde{a},\widetilde{b},\widetilde{c}),$ and $\boldsymbol{y}_{3}=(\widehat{a}%
,\widehat{b},\widehat{c})$ are in $\mathbb{R}^3.$

We have shown in Section~\ref{neuronas}  that training the whole neural network,
$f_{\theta }^{\text{net}}$, is equivalent to training each of its output components separately; 
$f_{\theta }^{\text{net},1}(\mathbf{x}), \,
 f_{\theta }^{\text{net},2}(\mathbf{x}), \,$ and 
$f_{\theta }^{\text{net},3}(\mathbf{x})$. 
For this reason, we restrict our attention to the training of a network
\begin{equation*}
f_{\theta }^{\text{net}} : \mathbb{R}^{2} \to \mathbb{R},
\end{equation*}
given by $f_\theta^\text{net}(\mathbf{x}) = 
w_{5}\,\sigma(w_{1}x_{1}+w_{2}x_{2}+b_{1})
+ w_{6}\,\sigma(w_{3}x_{1}+w_{4}x_{2}+b_{2}) + b_{3},$
where $\sigma(t) = \max\{t,0\}$ denotes the ReLU activation function. Consequently, we will consider the following training dataset
\[
\mathcal{D}_3 := \{ (\mathbf{x}_{1}, y_{1}), (\mathbf{x}_{2}, y_{2}), (\mathbf{x}_{3}, y_{3}) \},
\]
with $\mathbf{x}_{1} = (1,-1),$ $\mathbf{x}_{2} = (-1,0),$ $\mathbf{x}_{3} = (
0,1),$ and  
$y_{1} = a,$ $y_{2} = b,$ $y_{3} = c.$

In this framework, 
\begin{equation*}
\left\Vert \theta \right\Vert :=\max \{\left\vert w_{1}\right\vert
+\left\vert w_{2}\right\vert +\left\vert b_{1}\right\vert ;\left\vert
w_{3}\right\vert +\left\vert w_{4}\right\vert +\left\vert b_{2}\right\vert
;\left\vert w_{5}\right\vert +\left\vert w_{6}\right\vert +\left\vert
b_{3}\right\vert\} \qquad (\theta \in \Theta ).
\end{equation*}
where 
\begin{equation*}
\theta =\left( \left( 
\begin{array}{cc}
w_{1} & w_{2} \\ 
w_{3} & w_{4}%
\end{array}%
\right) ,\left( 
\begin{array}{c}
b_{1} \\ 
b_{2}%
\end{array}%
\right) ,(w_{5},w_{6}),b_{3}\right)
=(w_{1},w_{2},w_{3},w_{4},b_{1},b_{2},w_{5},w_{6},b_{3})\in \Theta .
\end{equation*}%
Since $\Theta $ is finite-dimensional, all norms on it are equivalent. Nevertheless, we may choose any convenient one.

In order to determine the kernel function to be used, we define 
$\xi : \Theta \to \mathbb{R}^+$ as
\begin{equation}
\xi(\theta) := \frac{1}{\max\{1, \lVert \theta \rVert^3\}}, 
\quad \theta \in \Theta.
\label{xi}
\end{equation}

Now, we define the kernel function 
$k : \mathbb{R}^{2} \times \Theta \to \mathbb{R}$ as
\begin{equation*}
k(\mathbf{x}, \theta) := f_\theta^\text{net}(\mathbf{x}) \, \xi(\theta),
\ \text{  for every } \mathbf{x} = (x_{1},x_{2}) \in \mathbb{R}^{2} \text{ and } \theta \in \Theta.
\end{equation*}

Note that $k(\mathbf{x}, \cdot) \in \mathcal{B}(\Theta)$ for every 
$\mathbf{x} \in \mathbb{R}^{2}$, where $\mathcal{B}(\Theta)$ denotes the Banach space 
of all continuous bounded functions on $\Theta$, endowed with the supremum norm 
$\lVert \cdot \rVert_{\infty}$. 
In fact, 
\[
\lim_{\lVert \theta \rVert \to \infty} f_\theta^\text{net}(\mathbf{x}) \, \xi(\theta) = 0,
\]
which implies that $k(\mathbf{x}, \cdot) \in C_{0}(\Theta)$ for every 
$\mathbf{x} \in \mathbb{R}^{2}$.

Therefore, for $\theta = (w_{1}, w_{2}, w_{3}, w_{4}, b_{1}, b_{2}, w_{5}, w_{6}, b_{3}) \in \Theta$ 
and $\mathbf{x} = (x_{1}, x_{2}) \in \mathbb{R}^{2}$, we have that
\begin{equation*}
k(\mathbf{x}, \theta) 
=  \frac{w_{5}\,\sigma(w_{1}x_{1} + w_{2}x_{2} + b_{1})
      + w_{6}\,\sigma(w_{3}x_{1} + w_{4}x_{2} + b_{2}) + b_{3}}{\max\{1, \lVert \theta \rVert^3\}},
\end{equation*}
Thus,
\begin{equation}
\label{7.1}
\begin{aligned}
\lvert k(\mathbf{x}, \theta)  \rvert 
&\leq \frac{
\lvert w_{5} \rvert \big(\lvert w_{1}x_{1} \rvert + \lvert w_{2}x_{2} \rvert + \lvert b_{1} \rvert \big)
+ \lvert w_{6} \rvert \big(\lvert w_{3}x_{1} \rvert + \lvert w_{4}x_{2} \rvert + \lvert b_{2} \rvert \big)
+ \lvert b_{3} \rvert}{\max\{1, \lVert \theta \rVert^3\}} \\
&\leq 
\frac{\max \{ \lvert x_{1} \rvert, \lvert x_{2} \rvert, 1 \} 
\Big( \lVert \theta \rVert \,\lvert w_{5} \rvert 
   + \lVert \theta \rVert \,\lvert w_{6} \rvert 
   + \lvert b_{3} \rvert \Big)}{\max\{1, \lVert \theta \rVert^3\}} \\
&\leq \frac{\max \{ \lvert x_{1} \rvert, \lvert x_{2} \rvert, 1 \}
\| \theta \|^2}{\max\{1, \lVert \theta \rVert^3\}} \leq \max \{ \lvert x_{1} \rvert, \lvert x_{2} \rvert, 1 \}.
\end{aligned}
\end{equation}
On the other hand, if $\max\{|x_{1}|,|x_{2}|,1\}=|x_{1}|$, then, for 
$
\theta_{0} = (\pm1,0,0,0,0,0,1,0,0)
$
(we choose $+1$ or $-1$ depending on the sign of $x_{1}$) we have that
\[k(\mathbf{x},\theta_{0})
=\sigma(|x_{1}|)
= |x_{1}|.\]

Similarly, if $\max\{|x_{1}|,|x_{2}|,1\}=|x_{2}|$, 
then for 
$
\theta_{0}=(0,\pm 1,0,0,0,0,1,0,0)
$
we have $k(\mathbf{x},\theta_{0}) = |x_{2}|$. 
Moreover, if $\max\{|x_{1}|,|x_{2}|,1\}=1$, then for 
$
\theta_{0} = (0,0,0,0,1,0,1,0,0)
$
we have $k(\mathbf{x},\theta_{0}) = \sigma(1) =1$. Therefore,
$$
\|k(\mathbf{x},\cdot)\|_{\infty} 
= \max\{|x_{1}|,|x_{2}|,1\}.
$$

Recall that our dataset is $\{(\mathbf{x}_{1},y_{1}), (\mathbf{x}_{2},y_{2}), (\mathbf{x}_{3},y_{3})\}$, where 
$$
\mathbf{x}_{1}=(1,-1), \qquad
\mathbf{x}_{2}=(-1,0), \qquad
\mathbf{x}_{3}=(
0,1),
$$
and $\boldsymbol{y}=(y_{1},y_{2},y_{3})=(a,b,c).$
Consequently,
$$
\|k(\mathbf{x}_{i},\cdot)\|_{\infty}
= 1, \qquad \text{for every } i=1,2,3.
$$

Furthermore, 

(i) For $\theta _{1}:=(1,0,0,0,0,0,1,0,0)$ we have $k(\mathbf{x},\theta
_{1}):=\sigma (x_{1}),$ so that $k(\mathbf{x}_{1},\theta _{1})=1$ and $k(%
\mathbf{x}_{2},\theta _{1})=$ $k(\mathbf{x}_{3},\theta _{1})=0.$

(ii) For $\theta_{2} := (-1,0,0,0,0,0,1,0,0)$ we have
$k(\mathbf{x}, \theta_{2}) := \sigma(-x_{1})$, so that
$k(\mathbf{x}_{2}, \theta_{2}) = 1$, while both
$k(\mathbf{x}_{1}, \theta_{2})$ and $k(\mathbf{x}_{3}, \theta_{2})$
are equal to $0$.

(iii) For $\theta_{3} := (0,1,0,0,0,0,1,0,0)$ we have
$k(\mathbf{x}, \theta_{3}) := \sigma(x_{2})$, so that
$k(\mathbf{x}_{3}, \theta_{3}) = 1$, while both
$k(\mathbf{x}_{1}, \theta_{3})$ and $k(\mathbf{x}_{2}, \theta_{3})$
are equal to $0$.

Therefore,
\[
k(\mathbf{x}_{i},\theta_{j}) = 0,  \quad\text{if }\quad i \ne j,
 \text{ and } 
\,k(\mathbf{x}_{i},\theta_{i}) 
= \|k(\cdot,\theta_i)\|
= 1, 
\quad \text{for }\quad   i=1,2,3.
\]
In fact, if \ $\sum\limits_{i=1}^{m}\alpha _{i}k(\widetilde{\mathbf{x}}%
_{i},\cdot )\in \mathrm{Lin}\{k(\mathbf{x},\cdot ):\mathbf{x}\in \mathbb{R}%
^{2}\}$\ is such that $\left\Vert \sum\limits_{i=1}^{m}\alpha _{i}k(%
\widetilde{\mathbf{x}}_{i},\cdot )\right\Vert _{\infty }\leq 1$ then, \ $%
\left\vert \sum\limits_{i=1}^{m}\alpha _{i}k(\widetilde{\mathbf{x}}%
_{i},\theta )\right\vert \leq 1$ so that 
\begin{equation*}
\left\Vert k(\cdot ,\theta )\right\Vert =\sup \left\{ \left\vert
\sum\limits_{i=1}^{m}\alpha _{i}k(\widetilde{\mathbf{x}}_{i},\theta
)\right\vert :\left\Vert \sum\limits_{i=1}^{m}\alpha _{i}k(\widetilde{%
\mathbf{x}}_{i},\cdot )\right\Vert _{\infty }\leq 1\right \}\leq 1,
\end{equation*}%
and it follows that if $ k(\mathbf{x},\theta )
=1=\left\Vert k(\mathbf{x},\cdot )\right\Vert _{\infty }$ for some $\boldsymbol{x%
}\in \mathbb{R}^{2}$, then $\left\Vert k(\cdot ,\theta )\right\Vert
=1.$ Hence,
\begin{equation}
M:=
\begin{pmatrix}
k(\mathbf{x}_{1},\theta_{1}) & k(\mathbf{x}_{1},\theta_{2}) & k(\mathbf{x}_{1},\theta_{3}) \\
k(\mathbf{x}_{2},\theta_{1}) & k(\mathbf{x}_{2},\theta_{2}) & k(\mathbf{x}_{2},\theta_{3}) \\
k(\mathbf{x}_{3},\theta_{1}) & k(\mathbf{x}_{3},\theta_{2}) & k(\mathbf{x}_{3},\theta_{3})
\end{pmatrix}=
\left( 
\begin{array}{ccc}
1 & 0 & 0 \\ 
0 & 1 & 0 \\ 
0 & 0 & 1%
\end{array}%
\right)
\end{equation}

Let us now determine the set of admissible sign vectors.  
The sign space is
\[
\Delta := \{ \boldsymbol{s} = (s_1, s_2, s_3) : s_i \in \{-1, 0, 1\},\ i = 1,2,3 \}.
\]
According to Definition~\ref{admisible}, 
a sign vector $\boldsymbol{s} =(s_1, s_2, s_3)\in \Delta^3$ is admissible if

\begin{equation*}
|||\,\boldsymbol{s}\,|||\text{ }:=\left\Vert
\sum\limits_{i=1}^{3}s_{i}k(\mathbf{x}_{i},\cdot )\right\Vert _{\infty }=\sup_{\theta \in \Theta }\left\vert
\sum\limits_{i=1}^{3}s_{i}k(\mathbf{x}_{i},\theta )\right\vert \leq 1.
\end{equation*}%

To determine the set of admissible sign vectors note that \mbox{$%
\theta =(0,0,0,0,0,0,0,0,1)\in \Theta $} is such that $\left\Vert \theta
\right\Vert =1$ and 
\begin{equation*}
\left\vert \sum\limits_{i=1}^{3}s_{i}k(\mathbf{x}_{i},\theta )\right\vert
=\left\vert \sum\limits_{i=1}^{3}s_{i}\right\vert ,
\end{equation*}%
so that all sign vectors $\boldsymbol{s}=(s_{1},s_{2},s_{3})$ with $\left\vert
\sum\limits_{i=1}^{3}s_{i}\right\vert >1$ are not admissible. Therefore, the vectors $\pm(1,1,0),$ $ \pm(1,0,1), \pm(0,1,1), \pm(1,1,1)$ are not admissible. Moreover:

(i) For $\theta :=(-\frac{4}{15},\frac{2}{5},0,0,\frac{1}{3},0,1,0,0)$  we
have that $$\left\vert -k(\mathbf{x}_{1},\theta )+k(\mathbf{x}_{2},\theta )+k(
\mathbf{x}_{3},\theta )\right\vert >1,$$ and therefore  $\pm (-1,1,1)$ is not
admissible.

(ii) For $\theta :=(\frac{3}{10},\frac{1}{5},0,0,\frac{1}{2},0,1,0,0)$ we
have that $$\left\vert k(\mathbf{x}_{1},\theta )-k(\mathbf{x}_{2},\theta )+k(%
\mathbf{x}_{3},\theta )\right\vert >1 , $$ so that  $\pm (1,-1,1)$ is not
admissible.

(iii) For $\theta :=(\frac{1}{5},-\frac{3}{10},0,0,\frac{1}{2},0,1,0,0)$ we have that $$
\left\vert k(\mathbf{x}_{1},\theta )+k(\mathbf{x}_{2},\theta )-k(\mathbf{x}%
_{3},\theta )\right\vert >1, $$ so that $\pm (1,1,-1)$  are not admissible.

On the other hand, we claim that $\pm (1,-1,0),$ $\pm (1,0,-1),$ $\pm (0,1,-1)$ are
admissible. Indeed, for $\mathbf{x=(}x_{1},x_{2})$,  $%
\widetilde{\mathbf{x}}\mathbf{=(}\widetilde{x}_{1},\widetilde{x}_{2}),$ and $%
\theta \in \Theta ,$ consider
\begin{equation*}
    k(\mathbf{x},\theta)= \frac{w_5\sigma(a)+w_6\sigma(c)+ b_3}{\max\{1,||\theta||^3 \}} \hspace{0.2cm} \text{and} \hspace{0.2cm} k(\widetilde{\mathbf{x}},\theta)= \frac{w_5\sigma(b)+w_6\sigma(d)+b_3}{\max\{1,||\theta||^3 \}},
\end{equation*}
where the auxiliary variables $a,b,c,d$ are defined as
\[
\begin{aligned}
a &= w_1 x_1 + w_2 x_2 + b_1, &\qquad
b &= w_1 \widetilde{x}_1 + w_2 \widetilde{x}_2 + b_1, \\[3pt]
c &= w_3 x_1 + w_4 x_2 + b_2, &\qquad
d &= w_3 \widetilde{x}_1 + w_4 \widetilde{x}_2 + b_2.
\end{aligned}
\]
Then, 
\[
\begin{aligned}
|k(\mathbf{x},\theta) - k(\widetilde{\mathbf{x}},\theta)|
&= \frac{\big|\,w_5\big(\sigma(a) - \sigma(b)\big)
        + w_6\big(\sigma(c) - \sigma(d)\big)\,\big|}
       {\max\{1, \|\theta\|^3\}} \\[6pt]
&\leq 
\frac{|w_5|\,|\sigma(a) - \sigma(b)|
     + |w_6|\,|\sigma(c) - \sigma(d)|}
     { \max\{1, \|\theta\|^3\}}.
\end{aligned}
\]
Since, $ |\sigma(a) - \sigma(b)|
\leq \max\{\,|\sigma(a)|,\,|\sigma(b)|\,\}
$ and  $|\sigma(c) - \sigma(d)|
\leq \max\{\,|\sigma(c)|,\,|\sigma(d)|\,\},
$ we have that 
\[
\begin{aligned}
|k(\mathbf{x},\theta) - k(\mathbf{\tilde{x}},\theta)|
&\leq 
\frac{
\max\{|\sigma(a)|,| \sigma(b)|,| \sigma(c)|, |\sigma(d)|\} \, (|w_5| + |w_6|)
}{
\max\{1, \|\theta\|^3\}
} \\[6pt]
&\leq
\frac{\max\{|x_1|,| x_2|,| \tilde{x}_1|, |\tilde{x}_2|,1\} ||\theta||^2} {\max\{1, \|\theta\|^3\}} \leq
\max\{|x_1|, |x_2|, |\tilde{x}_1|, |\tilde{x}_2|, 1\}.
\end{aligned}
\]
It follows that $\pm (1,-1,0),$ $\pm (1,0,-1),$ $\pm (0,1,-1)$ are admissible. Also, the vectors $\pm(1,0,0),$ $\pm(0,1,0)$, and $\pm(0,0,1)$ are admissible, as $||k(\mathbf{x}_i,\cdot)||_\infty=1$, for $i=1,2,3$. Consequently, we conclude that the set of admissible sign vectors is
\begin{equation*}
{\Delta^3}_{\rm{adm}}=\{(0,0,0),\pm (1,0,0),\pm
(0,1,0),\pm (0,0,1),\pm (1,-1,0),\pm (1,0,-1),\pm (0,1,-1)\}.
\end{equation*}

In order to train our network, since $M = (k(x_i, \theta_j))_{ij} = I_3$,
obviously we have that
\( \boldsymbol{\beta} = (a, b, c) = \boldsymbol{y}. \)
If \( \operatorname{sign}(\boldsymbol{y}) \) is admissible, then
 
\begin{equation}
f^{\mathrm{sol}}(\cdot )=a\text{ }k(\cdot ,\theta _{1})+b\,k(\cdot ,\theta _{2})+c%
\text{ }k(\cdot ,\theta _{3})  \label{ante}
\end{equation}%
is the solution of 
\begin{equation}
\label{eq:mni}
\arg \min \{\left\Vert f_{\mu }\right\Vert :f_{\mu }\in \mathcal{M}_{\boldsymbol{%
y}}\}.
\end{equation}
In fact, 
\(\|f^{\mathrm{sol}}\| = |a| + |b| + |c|\)
and it is attained in 
\(s_1 k(x_1, \cdot) + s_2 k(x_2, \cdot) + s_3 k(x_3, \cdot)\).


For instance, if $a<0, \ b=0$ and $c>0$, then $f^{\mathrm{sol}}=a k(x_1, \cdot)+ ck(x_3, \cdot)$ is a solution of \eqref{eq:mni} for $\mathbf{y}=(a,0,c),$ as $(-1,0,1)$ is an admissible sign vector. 

Suppose now that, for example, 
\(\mathbf{y} = (a, b, c) = (2, -3, \tfrac{1}{2})\).
Then, \(\operatorname{sign}(\boldsymbol{y})\) is not admissible.
Because of this, we consider a regularization problem of the type
\begin{equation} \label{eq:reg-problem}
    \arg\min \bigl\{ Q_{\mathbf{y}}(\mathcal{L}(f_\mu))
    + \lambda_0 \, \varphi(\|f_\mu\|) 
    : f_\mu \in B_{\Phi} \bigr\}.
\end{equation}

As an illustration, set
\[
\mathcal{R}(f_\mu)
:= Q_{\mathbf{y}}(\mathcal{L}(f_\mu))
   + \lambda_0 \, \varphi(\|f_\mu\|)
   = \frac{(\tilde{y}_1 - a)^2 + (\tilde{y}_2 - b)^2 + (\tilde{y}_3 - c)^2}{3}
   + \lambda_0 \|f_\mu\|,
\]
where 
\(\mathcal{L}(f_\mu) = (f_\mu(x_1), f_\mu(x_2), f_\mu(x_3)) 
= (\tilde{y}_1, \tilde{y}_2, \tilde{y}_3)\)
and \(\lambda_0 = \tfrac{1}{10}\).

In this case, we have that the map $ \widehat{\mu}^{\mathrm{data}} : W \to \mathbb{R}$ given by 
$$\widehat{\mu}^{\mathrm{data}} = 2 \delta_{\theta_1} - 3 \delta_{\theta_2}  + \tfrac{1}{2} \delta_{\theta_3}$$ is an extension of the  of the linear map 
\( \mu^{\text{data}} \colon W^{\text{data}} \to \mathbb{R} \)
determined by the equalities
\[
\mu^{\text{data}}(k(x_1, \cdot)) = 2, 
\qquad 
\mu^{\text{data}}(k(x_2, \cdot)) = -3, 
\qquad 
\mu^{\text{data}}(k(x_3, \cdot)) = \tfrac{1}{2}.
\]
Moreover, 
$$
f_{\widehat{\mu}^{\mathrm{data}}}  = 2k(\cdot, \theta_1) - 3k(\cdot, \theta_2) + \tfrac{1}{2}k(\cdot, \theta_3).
$$
However, we do not know if $ \widehat{\mu}^{\mathrm{data}} $ is a Hahn--Banach extension of $\mu^{\text{data}}$.
Because of this, we also consider regularized solutions of the type
\[
f_\varepsilon(\cdot)
= (2 + \varepsilon_1)k(\cdot, \theta_1)
  + (-3 + \varepsilon_2)k(\cdot, \theta_2)
  + \bigl(\tfrac{1}{2} + \varepsilon_3\bigr)k(\cdot, \theta_3),
\]
where \(\boldsymbol{\beta}_\varepsilon = (2+\varepsilon_1,\, -3+\varepsilon_2,\, \tfrac{1}{2}+\varepsilon_3)\)
and \(\operatorname{sign}(\boldsymbol{\beta}_\varepsilon)\) is admissible.
These are the solutions of
\[
\arg\min \{ \|f_\mu\| : f_\mu \in \mathcal{M}_{\mathbf{y}_\varepsilon} \},
\]
where 
\(\mathbf{y}_\varepsilon^{T} = M \boldsymbol{\beta}_\varepsilon^{T} 
= (2+\varepsilon_1,\, -3+\varepsilon_2,\, \tfrac{1}{2}+\varepsilon_3)^{T}\).
For every admissible sign vector we choose the best 
\(\varepsilon = (\varepsilon_1, \varepsilon_2, \varepsilon_3)\)
which is the one that minimizes~\eqref{eq:reg-problem}.

\smallskip
First, note that 
$
f_0(\cdot) = 2k(\cdot, \theta_1) - 3k(\cdot, \theta_2) + \tfrac{1}{2}k(\cdot, \theta_3)$
is such that 
\(\mathcal{R}(f_0) = \lambda_0 \|f_0\|\),
where 
\(\|f_0\| \le 2 + 3 + \tfrac{1}{2} = \tfrac{11}{2}\). Moreover,
\(\|f_0\| \ge |2(k(x_1, \theta_1) +3 k(x_2, \theta_3))| = 5,\) since $|k(x_1,\cdot)-k(x_2,\cdot)\| \leq 1$. Therefore, 
\(0.5 \le \mathcal{R}(f_0) \le 0.55.\)
Now, we determine the best regularized solutions associated with each admissible sign vector.

(i) For $\mathbf{s}_1 := (1, -1, 0)$, we have that the regularizations of $f_0$ are given by 
\[
f_{\boldsymbol{s_1}}(\cdot)
= (2 + \varepsilon_1)k(\cdot, \theta_1)
  + (-3 + \varepsilon_2)k(\cdot, \theta_2),
\]
with $2 + \varepsilon_1 \ge 0$ and $-3 + \varepsilon_2 \le 0$.
Therefore,
\[
\mathcal{R}(f_{\boldsymbol{s_1}})
= \frac{\varepsilon_1^2 + \varepsilon_2^2 + \tfrac{1}{4}}{3}
  + \frac{1}{10} (|2 + \varepsilon_1| + |-3 + \varepsilon_2|)
= \frac{\varepsilon_1^2 + \varepsilon_2^2}{3}
  + \frac{1}{12}
  + \frac{1}{10}(\varepsilon_1 - \varepsilon_2)
  + \frac{1}{2}.
\]

The minimum of $\mathcal{R}(f_{\boldsymbol{s_1}})$
is attained at 
$(\varepsilon_1, \varepsilon_2) = \bigl(-\tfrac{3}{20}, \tfrac{3}{20}\bigr)$.
Thus, related with the sign vector 
$\mathbf{s}_1 = (1, -1, 0)$, our best choice for a regularized solution is
\[
f_{\tilde{\mathbf{y}}_1}(\cdot)
= \bigl(2 - \tfrac{3}{20}\bigr)k(\cdot, \theta_1)
  + \bigl(-3 + \tfrac{3}{20}\bigr)k(\cdot, \theta_2),
\]
which is a solution of 
\[
\arg\min \{ \|f_\mu\| : f_\mu \in \mathcal{M}_{\tilde{\mathbf{y}}_1} \},
\]
where 
$\tilde{\mathbf{y}}_1 = \bigl(2 - \tfrac{3}{20},\, -3 + \tfrac{3}{20},\, 0\bigr)$.
Furthermore,
\[
\|f_{\tilde{\mathbf{y}}_1}\| 
= |2 - \tfrac{3}{20}| + |-3 + \tfrac{3}{20}|
= \tfrac{94}{20},
\]
this norm is attained in $k(x_1, \cdot) - k(x_2, \cdot)$, and
\[
\mathcal{R}(f_{\tilde{\mathbf{y}}_1})
= \frac{\tfrac{9}{400} + \tfrac{9}{400} + \tfrac{1}{4}}{3}
  + \frac{1}{10} \|f_{\tilde{\mathbf{y}}_1}\|
= \frac{59}{600} + \frac{1}{10}\tfrac{94}{20}
= \frac{341}{600} \approx 0.5683.
\]

\medskip
\noindent
(ii) For $\mathbf{s}_2 := (-1, 1, 0)$, we have that
\[
f_{\boldsymbol{s_2}}(\cdot)
= (2 + \varepsilon_1)k(\cdot, \theta_1)
  + (-3 + \varepsilon_2)k(\cdot, \theta_2),
\]
with $2 + \varepsilon_1 \le 0$ and $-3 + \varepsilon_2 \ge 0$.
Thus, we minimize
\[
\mathcal{R}(f_{\boldsymbol{s_2}})
= \frac{\varepsilon_1^2 + \varepsilon_2^2 + \tfrac{1}{4}}{3}
  + \frac{1}{10} (|2 + \varepsilon_1| + |-3 + \varepsilon_2|),
\]
with $\varepsilon_1 \le -2$ and $\varepsilon_2 \ge 3$.
The minimum is attained in $(\varepsilon_1, \varepsilon_2) = (-2, 3)$.
Thus,
\[
f_{\tilde{\mathbf{y}}_2}(\cdot) = 0
\quad\text{with}\quad
\mathcal{R}(f_{\tilde{\mathbf{y}}_2})
= \frac{\varepsilon_1^2 + \varepsilon_2^2 + \tfrac{1}{4}}{3}
= \frac{4+ 9 + \tfrac{1}{4}}{3}= \frac{53}{12}.
\]

\medskip
\noindent
(iii) For $\mathbf{s}_3 := (1, 0, -1)$ and $\mathbf{s}_4 := (-1, 0, 1)$ we have 
\[
f_{\boldsymbol{s_j}}(\cdot)
= (2 + \varepsilon_1)k(\cdot, \theta_1)
  + \bigl(\tfrac{1}{2} + \varepsilon_3\bigr)k(\cdot, \theta_3),
\]
(with $j = 3, 4$) where $2 + \varepsilon_1 \ge 0$ and 
$\tfrac{1}{2} + \varepsilon_3 \le 0$ if $j=3$, and 
$2 + \varepsilon_1 \le 0$ and $\tfrac{1}{2} + \varepsilon_3 \ge 0$ if $j=4$.
The minimum of
\[
\mathcal{R}(f_{\boldsymbol{s_j}})
= \frac{\varepsilon_1^2 + 9 + \varepsilon_3^2}{3}
  + \frac{1}{10}\Bigl(|2 + \varepsilon_1| + \bigl|\tfrac{1}{2} + \varepsilon_3\bigr|\Bigr),
\]
for $j=3$ is attained in 
$f_{\widetilde{\boldsymbol{y}}_3}(\cdot) = (2 - \tfrac{3}{20})k(\cdot, \theta_1)$
with $\mathcal{R}(f_{\widetilde{\boldsymbol{y}}_3}) = \tfrac{3931}{1200}  \approx 3.275$,
meanwhile for $j=4$ it is attained in
$f_{\widetilde{\boldsymbol{y}}_4}(\cdot) = \bigl(\tfrac{1}{2} - \tfrac{3}{20}\bigr)k(\cdot, \theta_3)$
with $\mathcal{R}(f_{\widetilde{\boldsymbol{y}}_4}) = \tfrac{5251}{1200} \approx 4.375$.

\medskip
\noindent
(iv) For $\mathbf{s}_5 := (0, -1, 1)$ and $\mathbf{s}_6 := (0, 1, -1)$,
we have 
\[
f_{\boldsymbol{s_5}}(\cdot) 
= f_{\boldsymbol{s_6}}(\cdot)
= (-3 + \varepsilon_2)k(\cdot, \theta_2)
  + \bigl(\tfrac{1}{2} + \varepsilon_3\bigr)k(\cdot, \theta_3),
\]
where $-3 + \varepsilon_2 \le 0$ and $\tfrac{1}{2} + \varepsilon_3 \ge 0$ for 
$f_{\boldsymbol{s_5}}(\cdot)$, 
and $-3 + \varepsilon_2 \ge 0$ and 
$\tfrac{1}{2} + \varepsilon_3 \le 0$ for 
$f_{\boldsymbol{s_6}}(\cdot)$.
Note that, for $j = 5, 6$,
\[
\mathcal{R}(f_{\boldsymbol{s_j}})
= \frac{4 + \varepsilon_2^2 + \varepsilon_3^2}{3}
  + \frac{1}{10}\Bigl(|-3 + \varepsilon_2| + \bigl|\tfrac{1}{2} + \varepsilon_3\bigr|\Bigr).
\]
We obtain as minimum 
\[
f_{\widetilde{\boldsymbol{y}}_5}(\cdot)
= (-3 + \tfrac{3}{20})k(\cdot, \theta_2)
  + \bigl(\tfrac{1}{2} - \tfrac{3}{20}\bigr)k(\cdot, \theta_3)
\text{  with } 
\mathcal{R}(f_{\widetilde{\boldsymbol{y}}_5}) = \tfrac{1001}{600} \approx 1.66,
\]
and $
f_{\widetilde{\boldsymbol{y}}_6}(\cdot) = 0
 \text{ with } 
\mathcal{R}(f_{\boldsymbol{\varepsilon}_6}) = \tfrac{53}{12} \approx 4.411.
$

\medskip
\noindent
(v) For $\mathbf{s}_7 := (1, 0, 0)$ and $\mathbf{s}_8 := (-1, 0, 0)$ we have
\[
f_{\boldsymbol{s_j}}(\cdot)
= (2 + \varepsilon_1)k(\cdot, \theta_1),
\]
for $j = 7, 8$, with $2 + \varepsilon_1 \ge 0$ if $j = 7$, and 
$2 + \varepsilon_1 \le 0$ if $j = 8$.
This gives rise to
\[
\mathcal{R}(f_{\boldsymbol{s_j}})
= \frac{\varepsilon_1^2 + 9 + \tfrac{1}{4}}{3}
  + \frac{1}{10}|2 + \varepsilon_1|,
  \qquad (j = 7, 8).
\]
Thus,
\[
\mathcal{R}(f_{\boldsymbol{s_7}})
= \frac{\varepsilon_1^2 + 9 + \tfrac{1}{4}}{3}
  + \frac{1}{10}(2 + \varepsilon_1),
\]
with $2 + \varepsilon_1 \ge 0$, and 
\[
\mathcal{R}(f_{\boldsymbol{s_8}})
= \frac{\varepsilon_1^2 + 9 + \tfrac{1}{4}}{3}
  - \frac{1}{10}(2 + \varepsilon_1),
\]
with $2 + \varepsilon_1 \le 0$.
We obtain 
\[
f_{\tilde{\mathbf{y}}_7}(\cdot)
= (2 - \tfrac{3}{20})k(\cdot, \theta_1),
\quad
\mathcal{R}(f_{\tilde{\mathbf{y}}_7})
= \tfrac{3921}{1200} \approx 3.26,
\quad
f_{\tilde{\mathbf{y}}_8} = 0,
\quad
\mathcal{R}(f_{\tilde{\mathbf{y}}_8}) = \tfrac{53}{12} \approx 4.41.
\]

\medskip
\noindent
(vii) For $\mathbf{s}_9 := (0, 1, 0)$ and $\mathbf{s}_{10} := (0, -1, 0)$ we have 
\[
f_{\boldsymbol{s_j}}(\cdot)
= (-3 + \varepsilon_2)k(\cdot, \theta_1),
\quad
(j = 9, 10),
\]
with $-3 + \varepsilon_2 \ge 0$ if $j=9$, and $-3 + \varepsilon_2 \le 0$ if $j=10$.
Thus,
\[
\mathcal{R}(f_{\boldsymbol{s_j}})
= \frac{4 + \varepsilon_2^2 + \tfrac{1}{4}}{3}
  + \frac{1}{10}|-3 + \varepsilon_2|
= \frac{\varepsilon_2^2}{3} + \frac{17}{12} + \frac{1}{10}|-3 + \varepsilon_2|,
\quad (j = 9, 10).
\]
We obtain 
\[
f_{\tilde{\mathbf{y}}_9} = 0
\quad \text{with} \quad 
\mathcal{R}(f_{\tilde{\mathbf{y}}_9}) = \tfrac{53}{12}\approx 4.41,
\quad
f_{\tilde{\mathbf{y}}_{10}}(\cdot)
= (-3 + \tfrac{3}{20})k(\cdot, \theta_1),
\quad
\mathcal{R}(f_{\tilde{\mathbf{y}}_{10}}) = \tfrac{2051}{1200} \approx 1.7.
\]

\medskip
\noindent
(viii) Similar for 
$\mathbf{s}_{11} := (0, 0, 1)$ and $\mathbf{s}_{12} := (0, 0, -1)$
with 
\[
f_{\boldsymbol{s_j}}(\cdot)
= \bigl(\tfrac{1}{2} + \varepsilon_3\bigr)k(\cdot, \theta_3),
\]
where $\tfrac{1}{2} + \varepsilon_3 \ge 0$ if $j=11$, and 
$\tfrac{1}{2} + \varepsilon_3 \le 0$ if $j=12$.
Then,
\[
\mathcal{R}(f_{\boldsymbol{\varepsilon}_j})
= \frac{4 + 9 + \varepsilon_3^2}{3}
  + \frac{1}{10}\Bigl|\tfrac{1}{2} + \varepsilon_3\Bigr|,
\qquad (j = 11, 12).
\]
Thus,
\[
f_{\tilde{\mathbf{y}}_{11}}(\cdot)
= \bigl(\tfrac{1}{2} - \tfrac{3}{20}\bigr)k(\cdot, \theta_3),
\quad
\mathcal{R}(f_{\tilde{\mathbf{y}}_{11}}) = \tfrac{5251}{1200} \approx 4.37,
\quad
f_{\tilde{\mathbf{y}}_{12}}(\cdot) = 0,
\quad
\mathcal{R}(f_{\tilde{\mathbf{y}}_{12}}) = \tfrac{53}{12} \approx 4.41.
\]

\medskip
\noindent
(ix) For $\mathbf{s}_{13} := (0, 0, 0)$ we have 
$f_{\widetilde{\boldsymbol{y}}_{13}}(\cdot) = 0$
and 
$\mathcal{R}(f_{\widetilde{\boldsymbol{y}}_{13}}) = \tfrac{53}{12} \approx 4.41.$

Consequently, our best choice in this framework is 
\[
f_0(\cdot):= f_{\widehat{\mu}^{\mathrm{data}}}
= 2k(\cdot, \theta_1)
  - 3k(\cdot, \theta_2)
  + \tfrac{1}{2}k(\cdot, \theta_3).
\]

\vspace{0.1cm}
\begin{remark}
If $(a,b,c) = \bigl(1, -\tfrac{1}{2}, \tfrac{1}{2}\bigr)$ and 
$\lambda_0 = \tfrac{1}{3}$, then we have that
\[
f_0 :=  k(\cdot, \theta_1)
  - \tfrac{1}{2}k(\cdot, \theta_2)
  + \tfrac{1}{2}k(\cdot, \theta_3),
\]
with $1 + \tfrac{1}{2} \le \|f_0\| \le 2$, as
$f_0(k(\cdot, \theta_1) - \tfrac{1}{2}k(\cdot, \theta_2)) = \tfrac{3}{2}$.
It follows that
\[
\tfrac{1}{2} \le \mathcal{R}(f_{0}) \le \tfrac{2}{3}.
\]
However, for $\mathbf{s}_1 := (1, -1, 0)$, we have that the regularizations of $f_0$ are
\[
f_{\boldsymbol{s}_1}(\cdot)
= (1 + \varepsilon_1)k(\cdot, \theta_1)
  + \bigl(-\tfrac{1}{2} + \varepsilon_2\bigr)k(\cdot, \theta_2),
\]
with $1 + \varepsilon_1 \ge 0$ and $-\tfrac{1}{2} + \varepsilon_2 \le 0$.
Hence,
\[
\mathcal{R}(f_{\boldsymbol{s}_1})
= \frac{\varepsilon_1^2 + \varepsilon_2^2 + \tfrac{1}{4}}{3}
  + \frac{1}{3}\bigl(|1 + \varepsilon_1| + |-\tfrac{1}{2} + \varepsilon_2|\bigr)
= \frac{\varepsilon_1^2 + \varepsilon_2^2}{3}
  + \frac{1}{12}
  + \frac{1}{3}(\varepsilon_1 - \varepsilon_2)
  + \frac{1}{2}.
\]
The minimum is attained in 
$f_{\tilde{\boldsymbol{y}}_1}(\cdot)
= (1 - \tfrac{1}{2})k(\cdot, \theta_1)$
with
\[
\mathcal{R}(f_{\tilde{\boldsymbol{y}}_1})
= \tfrac{1}{6} + \tfrac{1}{12} - \tfrac{1}{3} + \tfrac{1}{2}
= \tfrac{5}{12}.
\]
Therefore, in this case,
\[
f_{\tilde{\boldsymbol{y}}_1}(\cdot)
= \tfrac{1}{2}k(\cdot, \theta_1)
\]
is a better choice than 
\[
f_0(\cdot)
= k(\cdot, \theta_1)
  - \tfrac{1}{2}k(\cdot, \theta_2)
  + \tfrac{1}{2}k(\cdot, \theta_3).
\]
\end{remark}

\vspace{0.1cm}
\begin{remark}
If we replace $\mathbf{x}_1 := (1, -1)$ with 
$\tilde{\mathbf{x}}_1 := (2, -2)$, in the above example, then, 
for the same parameter $\theta_1 := (1, 0, 0, 0, 0, 0, 1, 0, 0)$, 
we have $k(\mathbf{x}, \theta_1) := \sigma(x_1)$,
so that $k(\tilde{\mathbf{x}}_1, \theta_1) = 2$ and 
$k(\mathbf{x}_2, \theta_1) = k(\mathbf{x}_3, \theta_1) = 0$.
Thus, $k(\cdot, \theta_1)$ attains its norm in 
$\tfrac{1}{2}k(\tilde{\mathbf{x}}_1, \cdot)$, and we proceed as above with 
$\tfrac{1}{2}k(\tilde{\mathbf{x}}_1, \cdot)$,
obtaining a solution
\[
f^{\mathrm{sol}}(\cdot)
= \tfrac{a}{2}k(\cdot, \theta_1)
  + b\,k(\cdot, \theta_2)
  + c\,k(\cdot, \theta_3).
\]
\end{remark}

\section{Learning-Theoretic Insights and Implications}
\label{sec:LT-insights}

This section distills the main learning-theoretic insights arising from the featured reproducing kernel Banach space framework developed in this work. Rather than revisiting formal constructions, we focus on how the results reshape the interpretation of learning problems beyond the Hilbert setting and clarify the role of function-space geometry in modern learning models.

A first insight is that minimal-norm interpolation and regularization remain conceptually meaningful learning principles in Banach spaces, but only when additional structural conditions are imposed. Unlike in reproducing kernel Hilbert spaces, continuity of point-evaluation functionals alone does not ensure the existence of feature representations, kernel-based algorithms, or finite-dimensional solution structures. From a learning perspective, this explains why naïve extensions of RKHS-based methods to non-Hilbertian norms often fail and why feature-map–induced constructions are essential for tractable learning formulations.

A second insight concerns the nature of representer theorems. In Banach settings, representer results are inherently conditional and may fail even when interpolation problems are well posed. This reveals a sharp conceptual boundary between learning problems that admit kernel-style solutions and those that give rise to genuinely infinite-dimensional solution sets. Understanding this boundary is crucial for interpreting the behavior of learning algorithms under non-quadratic regularization and sparsity-inducing norms.

Finally, viewing fixed-architecture neural networks as elements of special vector-valued featured RKBSs provides a unifying function-space perspective that connects kernel methods, norm-based regularization, and neural architectures. In this view, architectural choices and parameter norms determine the induced function-space geometry and, consequently, the implicit inductive bias of the learning model. This perspective highlights that many structural aspects of neural network learning can be understood independently of infinite-width limits or purely Hilbertian assumptions.

\section{Conclusions}
\label{sec:conc}

This work has developed a coherent functional-analytic framework extending classical reproducing kernel Hilbert space theory to Banach settings in a way that is compatible with modern learning models. By introducing and systematically analyzing featured reproducing kernel Banach spaces, we identified the precise structural conditions under which kernel-based learning formulations and representer-type results can be recovered beyond the Hilbertian regime.

A central outcome of this analysis is the identification of feature-map–induced structure and suitable predual representations as essential ingredients for kernel-based learning in Banach spaces. While evaluation continuity suffices in Hilbert spaces, it is generally insufficient in Banach settings, where additional geometric and duality assumptions are required. The featured RKBS framework isolates these requirements and provides a clear classification of Banach spaces that support kernel representations versus those that do not.

Within this framework, we established existence results for minimal-norm interpolation and regularization problems and proved representer theorems under explicit conditions. These results formally characterize when solutions admit finite kernel expansions and when only weak-$^\ast$ descriptions are available, thereby clarifying the limits of classical kernel intuition outside the Hilbert setting.

We further extended the theory to vector-valued featured RKBSs, providing a principled foundation for multi-output learning problems. Building on this extension, we showed that fixed-architecture neural networks naturally induce special vector-valued featured RKBSs, allowing network training to be rigorously formulated as a variational problem in an induced Banach space of functions. This establishes a precise mathematical link between kernel methods, Banach-space regularization, and neural network architectures.

The scope of this work is intentionally foundational. We do not address optimization dynamics, statistical generalization guarantees, or computational efficiency, which depend on additional modeling assumptions. Instead, the framework developed here provides a structural basis for future investigations into implicit regularization, architectural inductive bias, and hybrid kernel–neural models from a unified functional-analytic perspective.

Natural directions for future work include studying optimization dynamics within featured RKBSs, with particular emphasis on how Banach-space geometry and feature-map structure influence implicit regularization and the behavior of gradient-based algorithms. Another important direction is the derivation of statistical generalization guarantees under Banach-space geometries, including stability and robustness analyses that account for non-Hilbertian norms and duality structures. Finally, developing computationally efficient algorithms and approximation schemes that exploit the feature-map–induced structure of featured RKBSs remains a key challenge, especially in the context of large-scale kernel methods and neural-network models where explicit function-space representations are infeasible.

\subsubsection*{Acknowledgements} 

This publication is part of the project “Ethical, Responsible, and General Purpose Artificial Intelligence: Applications In Risk Scenarios” (IAFER) Exp.:TSI-100927-2023-1 funded through the creation of university-industry research programs (ENIA Programs), aimed at the research and development of artificial intelligence, for its dissemination and education within the framework of the Recovery, Transformation and Resilience Plan of the European Union Next Generation EU through the Ministry of Digital Transformation and the Civil Service. This work was also partially supported by Knowledge Generation Projects, funded by the Spanish Ministry of Science, Innovation, and Universities of Spain under the project PID2023-150070NB-I00. 

The authors acknowledge Miguel Cabrera and Juan F. Mena for helpful discussions and insightful comments.

\newpage
\bibliography{bibliography}
\bibliographystyle{apacite}

\end{document}